\pdfoutput=1
\documentclass[12pt]{article}

\usepackage[utf8]{inputenc} 
\usepackage[OT1]{fontenc}    
\usepackage{url}            
\usepackage{booktabs}       
\usepackage{amsfonts}       
\usepackage{nicefrac}       
\usepackage{wrapfig}
\usepackage{mystyle}
\usepackage{acronym}
\usepackage{thmtools,thm-restate}
\usepackage{geometry}
\usepackage{babel}
\acrodef{marl}[MARL]{\underline{M}ulti-\underline{A}gent \underline{R}einforcement \underline{L}earning}
\acrodef{mdp}[MDP]{\underline{M}arkov \underline{D}ecision \underline{P}rocess}
\acrodef{cnn}[CNN]{\underline{C}onvolutional \underline{N}eural \underline{N}etwork}
\acrodef{mlp}[MLP]{\underline{M}ulti\underline{L}ayer \underline{P}erceptrons}
\acrodef{gcn}[GCN]{\underline{G}raph \underline{C}onvolutional \underline{N}etwork}
\acrodef{nnr}[NNR]{\underline{N}eural \underline{N}onlinear \underline{R}egulator}
\acrodef{mle}[MLE]{\underline{M}aximum \underline{L}ikelihood \underline{E}stimate}
\acrodef{nlp}[NLP]{\underline{N}atural \underline{L}anguage \underline{P}rocessing}
\acrodef{cv}[CV]{\underline{C}omputer \underline{V}ision}
\acrodef{rl}[RL]{\underline{R}einforcement \underline{L}earning}
\acrodef{mpe}[MPE]{\underline{M}ultiple \underline{P}article \underline{E}nvironment}

\newcommand{\Att}{{\rm Att}}

\newcommand{\relu}{{\rm ReLU}}
\newcommand{\SM}{{\rm SM}}
\newcommand{\rff}{{\rm rFF}}
\newcommand{\tf}{{\rm tf}}
\newcommand{\U}{{\rm U}}
\newcommand{\mha}{{\rm MHA}}
\newcommand{\tv}{{\rm TV}}

\title{\LARGE Relational Reasoning via Set Transformers:  \\ Provable Efficiency and Applications to MARL }

\author{\normalsize
Fengzhuo Zhang\thanks{National University of Singapore; \texttt{fzzhang@u.nus.edu}}
\quad Boyi Liu\thanks{Northwestern University; \texttt{boyiliu2018@u.northwestern.edu}}
\quad Kaixin Wang\thanks{National University of Singapore; \texttt{kaixin.wang@u.nus.edu}}\\
\normalsize
\quad Vincent Y. F. Tan\thanks{National University of Singapore; \texttt{vtan@nus.edu.sg}}
\quad Zhuoran Yang\thanks{Yale University; \texttt{zhuoranyang.work@gmail.com}}
\quad Zhaoran Wang\thanks{Northwestern University; \texttt{zhaoranwang@gmail.com}}
}
\date{}

\begin{document}

\maketitle

\begin{abstract}
    The cooperative \ac{marl} with permutation invariant agents framework has achieved tremendous empirical successes in   real-world applications. Unfortunately, the theoretical understanding of this \ac{marl} problem is lacking due to the curse of many agents and the limited exploration of the relational reasoning in existing works. In this paper, we verify that the transformer implements complex relational reasoning, and we propose and analyze model-free and model-based offline \ac{marl} algorithms with the transformer approximators. We prove that the suboptimality gaps of the model-free and model-based algorithms are independent of and logarithmic in the number of agents respectively, which mitigates the curse of many agents. These results are   consequences of a  novel generalization error bound of the transformer and a novel analysis of the \ac{mle} of the system dynamics with the transformer. Our model-based algorithm is the first provably efficient \ac{marl}  algorithm that explicitly exploits the permutation invariance of the agents. Our improved generalization bound may be of independent interest and is applicable  to other regression problems related to the transformer beyond \ac{marl}.
\end{abstract}
\section{Introduction}
Cooperative \ac{marl} algorithms have achieved tremendous successes across a wide range of real-world applications including robotics~\citep{wang2022equivariant,wang2022mathrm}, games~\citep{tang2021sensory,mnih2013playing}, and finance~\citep{xu2021relation}. In most of these works, the {\em permutation invariance} of the agents is embedded into the problem setup, and the successes of these works hinge on leveraging this property.
 However, the theoretical understanding of why  the permutation invariant \ac{marl} has been so successful is lacking due to the following two reasons. First, the size of the state-action space grows exponentially with  the number of agents; this is known as ``the curse of many agents''~\citep{wang2020breaking,menda2018deep}. 
 The exponentially large state-action space prohibits the learning of value functions and policies due to the curse of dimensionality. 
 Second, although the mean-field approximation is widely adopted to mitigate  the curse of many agents~\citep{wang2020breaking,chen2021pessimism}, this approximation fails to capture the complex interplay between the agents. In the mean-field approximation, the influence of all the other agents on a fixed agent is captured only through the empirical distribution of the local states and/or local actions~\citep{wang2020breaking,chen2021pessimism}. This induces a restricted class of function approximators, which nullifies the possibly complicated relational structure of the agents, and thus fails to incorporate the complex interaction between agents. Therefore, designing provably efficient \ac{marl} algorithms that incorporate the efficient relational reasoning and break the curse of many agents remains an interesting and meaningful question.

In this paper, we regard  transformer networks as the representation learning module to incorporate relational reasoning among the agents. In particular, we focus on the offline \ac{marl} problem with the transformer approximators in the cooperative setting. In this setting, all the agents learn policies {\em cooperatively} to maximize a common reward function. More specifically, in the offline setting, the learner only has access to a pre-collected dataset and cannot interact adaptively with the environment. Moreover, we assume that the underlying \ac{mdp} is \emph{homogeneous}, which means that the reward and the transition kernel are permutation invariant functions of the state-action pairs of the agents. Our goal is to learn an {\em optimal} policy that is also permutation invariant.

To design provably efficient offline \ac{marl} algorithms,  we need to overcome three key challenges. (i) To estimate the action-value function and the system dynamics, the approximator function needs to implement   efficient relational reasoning among the agents. However, the theoretically-grounded function structure that incorporates the complex relational reasoning needs to be carefully designed. (ii) To mitigate the curse of many agents, the generalization bound of the transformer should be independent of the number of agents. Existing results  in~\cite{edelman2021inductive} thus require  rethinking and improvements. (iii) In   offline \ac{rl}, the mismatch between the sampling  and  visitation distributions induced by the optimal policy (i.e., ``distribution shift'') greatly restricts the application of the offline \ac{rl} algorithm. Existing works adopt the \emph{``pessimism''} principle to mitigate such a challenge. However,  this requires the quantification of the uncertainty in the value function estimation and the estimation of the dynamics in the model-free and model-based methods respectively. The quantification of the estimation error with the transformer function class is a key open question. 

We organize our work by addressing the abovementioned three challenges.

First, we theoretically identify the function class that can implement complex relational reasoning. We demonstrate the relational reasoning ability of the attention mechanism by showing that approximating the self-attention structure with the permutation invariant fully-connected neural networks (i.e., deep sets \citep{zaheer2017deep}) requires an {\em exponentially large} number of hidden nodes in the input dimension of each channel (Theorem \ref{thm:approx}). This result necessitates the self-attention structure in the set transformer.
    
Second, we design  offline model-free and model-based \ac{rl} algorithms with the transformer approximators. In the former, the transformer is adopted to estimate the action-value function of the policy. The {\em pessimism} is encoded in that we learn the policy according to the {\em minimal} estimate of the action-value function in the set of functions with  bounded empirical Bellman error. In the model-based algorithm, we estimate the system dynamics with the transformer structure. The policy is learned pessimistically according to the estimate of the system dynamics  in the confidence region that induces the conservative value function. 

Finally, we analyze the suboptimality gaps of our proposed algorithms, which indicate that the proposed algorithms mitigate the curse of many agents. For the model-free algorithm, the suboptimality gap  in Theorem~\ref{thm:main_mfree} is independent of the number of agents, which is a consequence of the fact that the generalization bound  of the transformer (Theorem~\ref{thm:main_concen}) is independent of the number of channels. For the model-based algorithm, the bound on the suboptimality gap in Theorem~\ref{thm:main_mbased} is logarithmic in the number of agents;  this follows from the analysis of the \ac{mle} of the system dynamics in Proposition~\ref{prop:mlebound}. We emphasize that our model-based algorithm is the first provably efficient \ac{marl} algorithm that exploits the permutation equivariance when estimating the dynamics.
\vskip4pt
\noindent{\bf Technical Novelties.} In Theorem~\ref{thm:main_concen}, we leverage a PAC-Bayesian framework to derive a generalization error bound of the transformer. Compared to 
\citet[Theorem~4.6]{edelman2021inductive}, the result is a significant improvement in   the dependence on the number of channels $N$ and the depth   of neural network $L$. This result may be of independent interest for enhancing our theoretical understanding of the attention mechanism and is applicable to other regression problems related to the transformer. In Proposition~\ref{prop:mlebound}, we derive the first estimation uncertainty quantification of the system dynamics with the transformer approximators, which can be also be used to  analyze other \ac{rl} algorithms with such approximators.
\vskip4pt
\noindent{\bf More Related Work.} In this paper, we consider the offline \ac{rl} problem, and the insufficient coverage lies at the core of this problem. With the global coverage assumption, a number of works have been proposed from both the model-free~\citep{chen2019information,antos2008learning,nachum2019algaedice,duan2021risk,xie2020q} and model-based~\citep{,chen2019information,ross2012agnostic} perspectives. To weaken the global coverage assumption, we leverage the ``pessimism'' principle  in the algorithms: the model-free algorithms impose additional penalty terms on the estimate of the value function~\citep{jin2021pessimism,rajaraman2020toward} or regard the  function that attains the minimum in the confidence region as the estimate of the value function~\citep{xie2021bellman}; the model-based algorithms estimate the system dynamics by incorporating additional penalty terms~\citep{chang2021mitigating} or minimizing in the region around \ac{mle}~\citep{uehara2021pessimistic}. For the \ac{marl} setting, the offline \ac{marl} with the mean-field approximation has been studied in~\cite{chen2021pessimism,li2021permutation}.

The analysis of the \ac{marl} algorithm with the transformer approximators requires the generalization bound of the transformer. The transformer is an element of the group equi/invariant functions, whose benefit in terms of its generalization capabilities  has attracted extensive recent attention. Generalization bounds have been successively improved by analyzing the cardinality of the ``effective'' input field and Lipschitz constants of functions~\citep{sokolic2017generalization, sannai2021improved}. However, these methods result in loose generalization bounds when applied to deep neural networks~\citep{jakubovitz2019generalization}. \cite{zhu2021understanding} empirically demonstrated the benefits of the invariance in the model by refining the covering number of the function class, but a  unified theoretical understanding is still lacking. The covering number of the norm-bounded transformer was shown by~\cite{edelman2021inductive} to be at most  logarithmic in the number of channels. We show that this can be further improved using a PAC-Bayesian framework. In addition, we refer to the related concurrent work~\citep{Anon2022invariance} for a Rademacher complexity-based generalization bound of the transformer that is independent of the length of the sequence for the tasks such as computer vision.


\section{Preliminaries}\label{sec:prelim}

\textbf{Notation.} Let $[n]=\{1,\ldots,n\}$. The $i^{\rm{th}}$ entry of the vector $\bx$ is denoted as $x_{i}$ or $[\bx]_{i}$.  The $i^{\rm{th}}$ row and the $i^{\rm{th}}$ column of matrix $\bX$ are denoted as $\bX_{i,:}$ and $\bX_{:,i}$  respectively. The  {\em $\ell_{p}$-norm} of the vector $\bx$ is $\|\bx\|_{p}$. The {\em $\ell_{p,q}$-norm} of the matrix $\bX\in\bbR^{m\times n}$ is defined as $\|\bX\|_{p,q}= (\sum_{i=1}^{n}\|\bX_{:,i}\|_{p}^{q})^{1/q}$, and the {\em Frobenius norm} of $\bX$ is defined as $\|\bX\|_{\rmF}= \|\bX\|_{2,2}$. The {\em total variation distance} between two distributions $P$ and $Q$ on $\calA$ is defined as $\tv(P,Q)=\sup_{A\subseteq\calA}|P(A)-Q(A)|$. For a set $\calX$, we use $\Delta(\calX)$ to denote the set of distributions on $\calX$. For two conditional distributions $P,Q:\calX\rightarrow\Delta(\calY)$, the $d_{\infty}$ distance between them is defined as $d_{\infty}(P,Q)= 2\sup_{x\in \calX} \tv(P(\cdot\,|\,x),Q(\cdot\,|\,x)).$ Given a metric space $(\calX,\|\cdot\|)$, for a set $\calA\subseteq\calX$, an {\em $\varepsilon$-cover} of $\calA$ is a finite set $\calC\subseteq\calX$ such that for any $a\in\calA$, there exists $c\in\calC$ and $\|c-a\|\leq \varepsilon$. The {\em $\varepsilon$-covering number} of $\calA$ is the cardinality of the smallest $\varepsilon$-cover, which is denoted as $\calN(\calA,\varepsilon,\|\cdot\|)$.


\vskip4pt
\noindent{\bf Attention Mechanism and Transformers.}
The {\em attention mechanism} is a technique that mimics  cognitive attention to process  multi-channel inputs~\citep{bahdanau2014neural}. Compared with the \ac{cnn}, the transformer has been empirically shown to possess outstanding robustness against occlusions and preserve the global context due to its special relational structure~\citep{naseer2021intriguing}. Assume we have $N$ query vectors that are in $\bbR^{d_{Q}}$. These vectors are stacked to form the matrix $\bQ\in\bbR^{N\times d_{Q}}$. With $N_{V}$ {\em key vectors} in the matrix $\bK\in\bbR^{N_{V}\times d_{Q}}$ and $N_{V}$ {\em value vectors} in the matrix $\bV\in\bbR^{N_{V}\times d_{V}}$, the attention mechanism maps the queries $\bQ$ using the function ${\Att}(\bQ,\bK,\bV)=\SM(\bQ\bK^{\top})\bV$, where $\SM(\cdot)$ is the row-wise softmax operator that normalizes each row using the exponential function, i.e., for $\bx\in\bbR^{d}$, $[\SM(\bx)]_{i}=\exp(\bx_{i})/\sum_{j=1}^{d}\exp(\bx_{j})$ for $i\in[d]$. The product $\bQ\bK^{\top}$ measures the similarity between the queries and the keys, which is then passed through the activation function $\SM(\cdot)$. Thus, $\SM(\bQ\bK^{\top})\bV$ essentially outputs a weighted sum of $\bV$ where a value vector has greater weight if the corresponding query and key are more similar. The {\em self-attention mechanism} is defined as the attention that takes $\bQ=\bX\bW_{Q}$, $\bK=\bX\bW_{K}$ and $\bV=\bX\bW_{V}$ as inputs, where $\bX\in \bbR^{N\times d}$ is the input of the self-attention, and $\bW_{Q},\bW_{K}\in\bbR^{d\times d_{Q}}$ and $\bW_{V}\in\bbR^{d\times d_{V}}$ are the parameters. Intuitively, the self-attention mechanism weighs the inputs with the correlations among the $N$ different channels. This mechanism demonstrates a special pattern of \emph{relational reasoning} among the channels of $\bX$. 

In addition, the self-attention mechanism is {\em permutation invariant} in the channels in $\bX$. This implies that for any row-wise permutation function $\bpsi(\cdot)$, which swaps the rows of the input matrix according to a given permutation of $[N]$, we have ${\Att}(\bpsi(\bX)\bW_{Q},\bpsi(\bX)\bW_{K},\bpsi(\bX)\bW_{V})=\bpsi({\Att}(\bX\bW_{Q},\bX\bW_{K},\bX\bW_{V}))$. The permutation equivariance of the self-attention renders it   suitable for   inference tasks where the output is equivariant with respect to the ordering of inputs. For example, in  image segmentation, the result should be invariant  to the permutation of the objects in the input image~\citep{bronstein2021geometric}. The resultant transformer structure combines  the self-attention with  multi-layer perceptrons and composes them to form   deep neural networks. It remains permutation equi/invariant with respect to the order of the channels and has achieved excellent performance in many tasks \citep{dosovitskiy2020image, yuan2021tokens, lee2019set}.

\vskip4pt
\noindent{\bf Offline Cooperative \ac{marl}.} 
In this paper, we consider the {\em cooperative} \ac{marl} problem, where all agents aim to maximize a {\em common} reward function. The corresponding \ac{mdp} is characterized by the tuple $(\bar{\bS}_{0},\bar{\calS},\bar{\calA},P^{*},r,\gamma)$ and the number of agents is $N$. The state space $\bar{\calS}=\calS^{N}$ is the Cartesian product of the state spaces of each agent $\calS$, and $\bar{\bS}=[\bs_{1},\ldots,\bs_{N}]^{\top}$ is the state, where $\bs_{i}\in\bbR^{d_{\calS}}$ is the state of the $i^{\rm{th}}$ agent. The initial state is $\bar{\bS}_{0}$. The action space $\bar{\calA}=\calA^{N}$ is the Cartesian product of the action spaces $\calA$ of each agent, and $\bar{\bA}=[\ba_{1},\ldots,\ba_{N}]^{\top}$ is the action, where $\ba_{i}\in\bbR^{d_{\calA}}$ is the action of the $i^{\rm{th}}$ agent. The transition kernel is $P^{*}:\calS^{N}\times\calA^{N}\rightarrow\Delta(\calS^{N})$, and $\gamma\in(0,1)$ is the {\em discount factor}. Without loss of generality, we assume that the reward function $r$ is deterministic and bounded, i.e., $r:\calS^{N}\times\calA^{N}\rightarrow[-R_{\max},R_{\max}]$. We define the the {\em state-value function} $V_{P}^{\pi}:\calS^{N}\rightarrow[-V_{\max},V_{\max}]$, where $V_{\max}=R_{\max}/(1-\gamma)$, and the {\em action-value function} $Q_{P}^{\pi}:\calS^{N}\times\calA^{N}\rightarrow[-V_{\max},V_{\max}]$ of a policy $\pi$ and a transition kernel~$P$ as
\begin{align*}
    V_{P}^{\pi}(\bar{\bS})\!=\!\bbE^{\pi}\bigg[\sum_{t=0}^{\infty}\gamma^{t}r(\bar{\bS}_{t},\bar{\bA}_{t})\,\bigg|\,\bar{\bS}_{0}\!=\!\bar{\bS}\bigg],\;\text{  and  }\; Q_{P}^{\pi}(\bar{\bS},\bar{\bA})\!=\!\bbE^{\pi}\bigg[\sum_{t=0}^{\infty}\gamma^{t}r(\bar{\bS}_{t},\bar{\bA}_{t})\,\bigg|\,\bar{\bS}_{0}\!=\!\bar{\bS}, \bar{\bA}_{0}\!=\!\bar{\bA}\bigg],
\end{align*}
respectively. Here, the expectation is taken with respect to\ the Markov process induced by the policy $\bar{\bA}_{t}\sim \pi(\cdot\,|\, \bar{\bS}_{t})$ and the transition kernel $P$. The action-value function $Q_{P^{*}}^{\pi}$ is the unique fixed point of the operator 
$(\calT^{\pi}f)(\bar{\bS},\bar{\bA})= r(\bar{\bS},\bar{\bA})+\gamma\bbE_{\bar{\bS}^{\prime}\sim  P^{*}(\cdot\,|\,\bar{\bS},\bar{\bA})}[f(\bar{\bS}^{\prime},\pi)\,|\,\bar{\bS},\bar{\bA}]$, where the term in the expectation is defined as $f(\bar{\bS},\pi)= \bbE_{\bar{\bA}\sim \pi(\cdot\,|\,\bar{\bS})}[f(\bar{\bS},\bar{\bA})]$. We further define the {\em visitation measure} of the state and action pair induced the policy $\pi$ and transition kernel $P$ as $d^{\pi}_{P}(\bar{\bS},\bar{\bA})= (1-\gamma)\sum_{t=0}^{\infty} \gamma^{t}d^{\pi}_{P,t}$, where $d^{\pi}_{P,t}$ is the distribution of the state and the action at step $t$. 

In   offline \ac{rl}, the learner only has access to a pre-collected dataset and cannot interact with the environment. The dataset $\calD=\{(\bar{\bS}_{i},\bar{\bA}_{i},r_{i},\bar{\bS}_{i}^{\prime})\}_{i=1}^{n}$  is collected in an i.i.d.\ manner, i.e., $(\bar{\bS}_{i},\bar{\bA}_{i})$ is independently sampled from $\nu\in\Delta(\bar{\calS}\times\bar{\calA})$, and $\bar{\bS}^{\prime}_{i}\sim P^{*}(\cdot\, | \,\bar{\bS}_{i},\bar{\bA}_{i})$.  This i.i.d.\ assumption is made to simplify our theoretical results; see Appendix~\ref{app:extnoniid} for extensions to the non i.i.d.\ case.  Given a policy class $\Pi$, our goal is to find an optimal policy that maximizes the state-value function $\pi^{*}=\argmax_{\pi\in\Pi} V_{P^{*}}^{\pi}(\bar{\bS}_{0})$. For any $\pi\in\Pi$, the {\em suboptimality gap} of $\pi$ is defined as $V_{P^{*}}^{\pi^{*}}(\bar{\bS}_{0})-V_{P^{*}}^{\pi}(\bar{\bS}_{0})$.

\begin{figure}[t]
	\centering
	\subfigure[$\rho_{\relu}(\sum_{i=1}^{N} \bphi_{\relu}(\bx_{i}))$ with $\rho_{\relu}$ and $\psi_{\relu}$ as single-hidden layer neural networks.]{
	\begin{minipage}[t]{0.44\linewidth}
	\centering
	\includegraphics[width=2.5in]{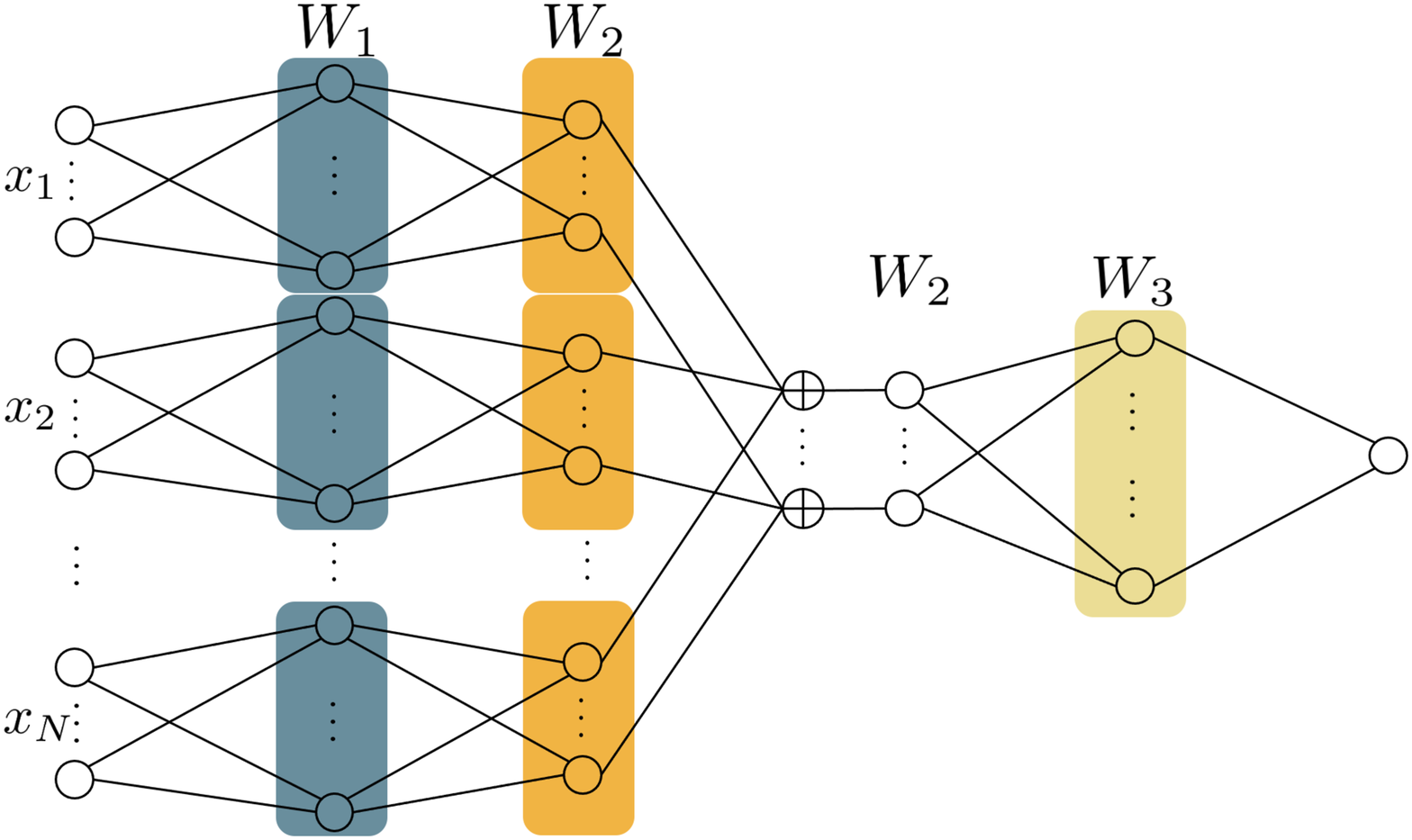}
	\end{minipage}%
	\label{fig:relunet}
	}%
	\hspace{1cm}
	\centering
	\subfigure[Self-attention mechanism $\bbI_{N}^{\top}\rm{Att}(\bX,\bX,\bX)\bw$.]{
	\begin{minipage}[t]{0.44\linewidth}
	\centering
	\includegraphics[width=2.5in]{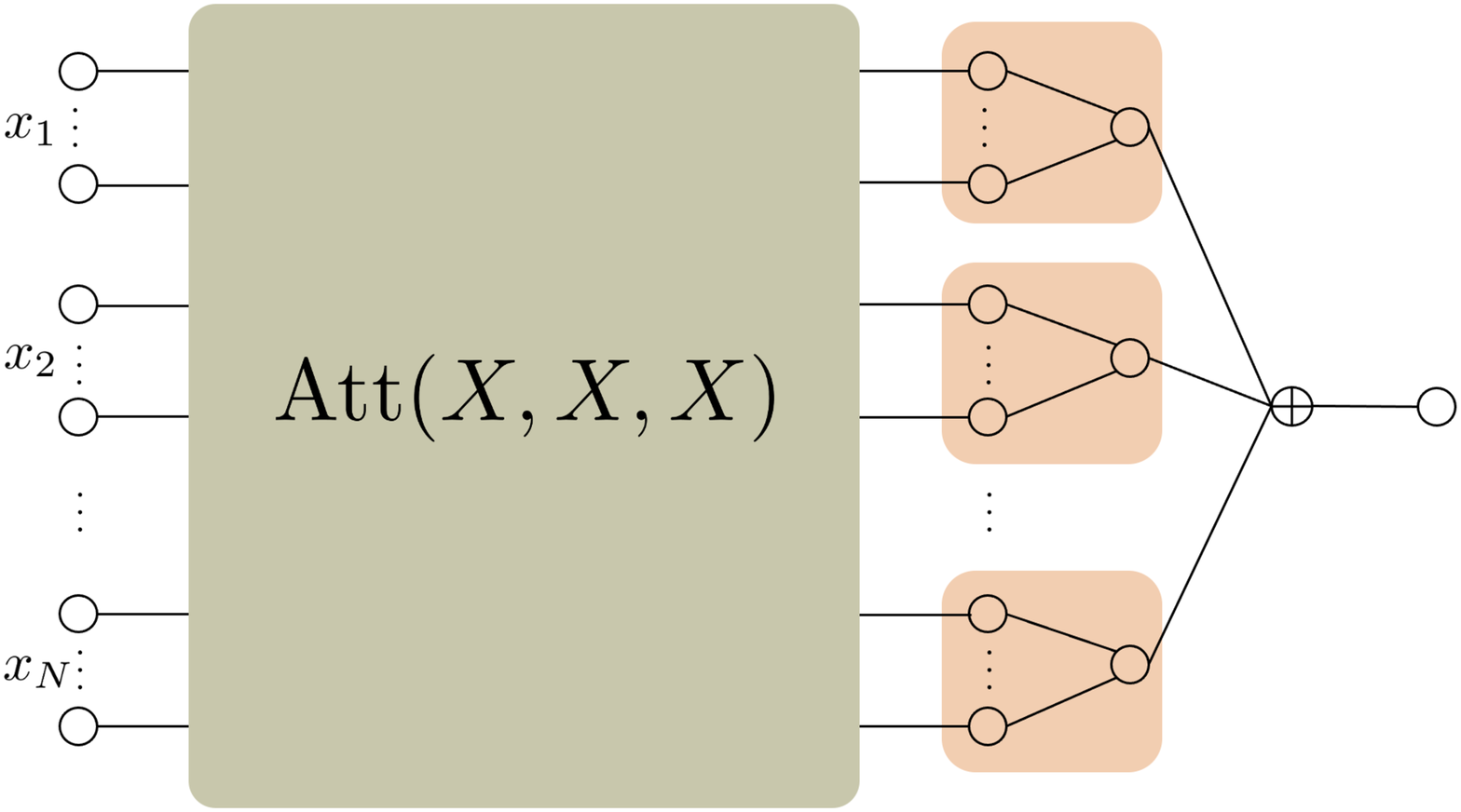}
	\end{minipage}%
	\label{fig:selfatt}
	}%
	\vspace{-0.3cm}
	\centering
	\caption{The blocks with the same color share the same parameters. The left figure shows that $\rho_{\relu}(\sum_{i=1}^{N} \bphi_{\relu}(\bx_{i}))$ first sums the outputs of $\bphi_{\relu}(\bx_{i})$, and it implements the relational reasoning only  through the single-hidden layer network $\rho_{\relu}$. In contrast, the self-attention block in the right figure captures the relationship among channels and then sums the outputs of each channel.}
	\label{fig:netgifse}
	\vspace{-0.4cm}
\end{figure}

\section{Provable Efficiency of Transformer on Relational Reasoning}\label{sec:funcapprox}

In this section, we provide the theoretical understanding of the outstanding relational reasoning ability of transformer. These theoretical results serves as a firm base for adopting set transformer to estimate the value function and system dynamics in \ac{rl} algorithms in the following sections. 

\subsection{Relational Reasoning Superiority of Transformer Over MLP}

The transformer neural network combines the self-attention mechanism and the fully-connected neural network, which includes the \ac{mlp} function class as a subset. On the inverse direction, we show that permutation invariant \ac{mlp} can not approximate transformer unless its width is exponential in the input dimension due to the poor relational reasoning ability of \ac{mlp}.

\citet[Theorem~2]{zaheer2017deep}  showed that all  permutation invariant functions take the form $\rho(\sum_{i=1}^{N} \bphi(\bx_{i}))$ with $\bX=[\bx_{1},\ldots,\bx_{N}]^{\top}\in\bbR^{N\times d}$ as the input. Since the single-hidden layer $\relu$ neural network is an universal approximator for continuous functions~\citep{sonoda2017neural}, we set $\bphi:\bbR^{N\times d}\rightarrow \bbR^{W_{2}}$ and $\rho:\bbR^{W_{2}}\rightarrow\bbR$ to be  single-hidden layer neural networks with $\relu$ activation functions as shown in Figure~\ref{fig:relunet}, where $W_{2}$ is the dimension of the intermediate outputs. The widths of the hidden layers in $\bphi_{\relu}$ and $\rho_{\relu}$ are $W_{1}$ and $W_{3}$ respectively. For the formal definition of $\bphi_{\relu}$ and $\rho_{\relu}$, please refer to Appendix~\ref{app:fcdef}. Then the function class with $\rho_{\relu}$ and $\bphi_{\relu}$ as width-constrained $\relu$ networks is defined as
\begin{align*}
    \calN(W)= \bigg\{f:\bbR^{N\times d}\rightarrow \bbR\ \bigg|\ f(\bX)=\rho_{\relu}\bigg(\sum_{i=1}^{N} \bphi_{\relu}(\bx_{i})\bigg) \text{ with } \max_{i \in [3]} W_{i} \leq W\bigg\}.
\end{align*}
We would like to use   functions in $\calN(W)$ to approximate the self-attention function class
\begin{align}
    \calF=\big\{f:\bbR^{N\times d}\rightarrow\bbR\ \big|\ f(\bX)=\bbI_{N}^{\top}\Att(\bX,\bX,\bX)\bw \text{ for some }\bw\in[0,1]^{d} \big\}.\nonumber
\end{align}
Figure~\ref{fig:relunet} shows that $\rho_{\relu}(\sum_{i=1}^{N} \bphi_{\relu}(\bx_{i}))$  first processes each channel with $\bphi_{\relu}$, and the relationship between channels is only reasoned with   $\rho_{\relu}$. The captured relationship in $\rho_{\relu}(\sum_{i=1}^{N} \bphi_{\relu}(\bx_{i}))$ cannot be too complex due to the simple structure of $\rho_{\relu}$. In contrast, the self-attention structure shown in Figure~\ref{fig:selfatt} first captures the relationship between channels with the self-attention structure and then weighs the results to derive the final output. Consequently, it is difficult to approximate the self-attention structure with $\rho_{\relu}(\sum_{i=1}^{N} \bphi_{\relu}(\bx_{i}))$ due to its poor relational reasoning ability. This observation is formally quantified in the following theorem.
\begin{restatable}{theorem}{propapprox}\label{thm:approx}
Let $W^*(\xi,d,\calF)$ be the smallest width of the neural network such that 
    \begin{align}
        \forall \,  f\in\calF, \ \exists\,  g\in \calN(W)\quad  \text{s.t. }\quad \sup_{\bX\in[0,1]^{N\times d}} \bigl|f(\bX)-g(\bX)\bigr|\leq \xi.\nonumber
    \end{align}
    With sufficient number of channels $N$, it holds that  $W^*(\xi,d,\calF)=\Omega(\exp{(cd)} \xi^{-1/4})$ for some $c>0$.
\end{restatable}
Theorem \ref{thm:approx} shows that the fully-connected neural network cannot approximate the relational reasoning process in the self-attention mechanism unless the width is exponential in the input dimension. This exponential lower bound of the width of the fully-connected neural network implies that the  relational reasoning process embedded within the self-attention structure is complicated, and it further motivates us to explicitly incorporate the self-attention structure in the neural networks in order to reason the complex relationship among the channels. 

\subsection{Channel Number-independent Generalization Error Bound}\label{sec:gene_err}
\begin{wrapfigure}{r}{6.3cm}
	\centering
	\vspace{-.15in}
	\includegraphics[width=2.5in]{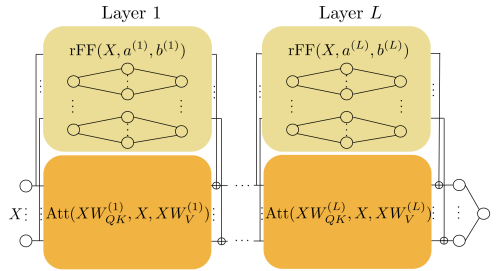}
	\label{fig:tf}
	\vspace{-0.5cm}
	\caption{Structure of the transformer function class, where the row-wise feedforward function is specified as  fully-connected networks.}
	\vspace{-0.6cm}
\end{wrapfigure}
In this section, we derive the generalization error bound of transformer. We take $\bX\in \bbR^{N\times d}$ as the input of the neural network. In the $i^{\rm th}$ layer, as shown in Figure~\ref{fig:tf}, we combine the self-attention mechanism $\Att(\bX\bW_{QK}^{(i)},\bX,\bX\bW_{V}^{(i)})$ with the \underline{r}ow-wise \underline{F}eed\underline{F}orward (rFF) single-hidden layer neural network $\rff(\bX,\ba^{(i)},\bb^{(i)})$ with width $m$. We combine $\bW_{Q}^{(i)}$ and $\bW_{K}^{(i)}$ to $\bW_{QK}^{(i)}$ for ease of calculation, and $\bb^{(i)}$ and $\ba^{(i)}$ are the parameters of the first and second layer of rFF. The output of each layer is normalized by the row-wise normalization function $\Pi_{\rm{norm}}(\cdot)$, which projects each row of the input into the unit $\ell_{p}$-ball  (for some $p\geq1$). For the last layer, we derive the scalar estimate of the action-value function by averaging the outputs of all the channels, and the ``clipping'' function $\Pi_{V}(x)$ is applied to normalize the output to $[-V,V]$. We note that such structures are also known as {\em set transformers} in \citet{lee2019set}. For the formal definition of the transformer, please refer to Appendix~\ref{app:tfdef}.

We consider a transformer with  bounded parameters. For a pair of conjugate numbers $p,q\in\bbR$, i.e., $1/p+1/q=1$ and $p,q\geq 1$, the transformer function class with bounded parameters is defined as
\begin{align}
    &\calF_{\tf}(\bB)=\Big\{g_{\tf}(\bX;\bW_{QK}^{1:L},\bW_{V}^{1:L},\ba^{1:L},\bb^{1:L},\bw)\, \Big| \, \big|a_{kj}^{(i)}\big|< B_{a}, \big\|\bb_{kj}^{(i)}\big\|_{q}< B_{b},\nonumber\\*
    & \big\|\bW_{QK}^{(i)\top}\big\|_{p,q}< B_{QK},\big\|\bW_{V}^{(i)\top}\big\|_{p,q}< B_{V},\|\bw\|_{q}< B_{w} \text{ for } i\in[L],j\in[m],k\in[d]\Big\},\nonumber
\end{align}
where $\bB= [B_{a},B_{b},B_{QK},B_{V},B_{w}]$ are the parameters of the function class, and $\bW_{QK}^{1:L},\bW_{V}^{1:L},\ba^{1:L}$ and $\bb^{1:L}$ are the stacked parameters in each layer. We only consider the non-trivial case where $B_{a},B_{b},B_{QK},B_{V},B_{w}$ are larger than one, otherwise the norms of the outputs decrease exponentially with growing depth. For ease of notation, we denote $\calF_{\tf}(\bB)$ as $\calF_{\tf}$ when the parameters are clear.

Consider the regression problem where we aim to predict the value of the response variable $y\in\bbR$ from the observation matrix $\bX\in\bbR^{N\times d}$, where $(\bX,y)\sim\nu$, and $|y|\leq V$. We derive our estimate $f:\bbR^{N\times d}\rightarrow\bbR$ from i.i.d.\ observations $\calD_{\text{reg}}=\{(\bX_{i},y_{i})\}_{i=1}^{n}$ generated from $\nu$. The {\em risk} of using $f\in \calF_{\tf}(\bB)$ as a regressor on sample $(\bX,y)$ is defined as $(f(\bX)-y)^{2}$. Then the {\em excess risk} of functions in the transformer function class $\calF_{\tf}$ can be bounded as in the following proposition.
\begin{proposition}\label{prop:regress_risk}
    Let $\bar{B}=B_{V}B_{QK}B_{a}B_{b}B_{\bw}$.  For all $f\in\calF_{\tf}$, with probability at least $1-\delta$, we have 
    \begin{align}
        &\Big|\bbE_{\nu}\Big[\big(f(\bX)-y\big)^{2}\Big]-\frac{1}{n}\sum_{i=1}^{n}\big(f(\bX_{i})-y_{i}\big)^{2}\Big|\nonumber\\
        &\leq \!\frac{1}{2}\bbE_{\nu}\!\Big[\big(f(\bX)-y\big)^{2}\Big]\!+\!O\bigg(\frac{V^{2}}{n}\biggl[mL^{2}d^{2}\log\frac{mdL\bar{B} n}{V} +\log\frac{1}{\delta}\biggr]\bigg).\nonumber
    \end{align}
\end{proposition}
Proposition~\ref{prop:regress_risk} is a corollary of Theorem~\ref{thm:main_concen}. We state it here since the generalization error bound of transformer may be interesting for other regression problems. We compare our generalization error bound in Proposition~\ref{prop:regress_risk} with   \citet[Theorem~4.6]{edelman2021inductive}. For the dependence on the number of agents $N$,  the result in \citet[Theorem~4.6]{edelman2021inductive} shows that the logarithm of the covering number of the transformer function class is logarithmic in $N$. Combined with the use of the Dudley integral \citep{mohri2018foundations}, \citet[Theorem~4.6]{edelman2021inductive} implies that the generalization error bound is logarithmic in $N$. In contrast, our result is independent of $N$. This superiority is attributed to our use of the PAC-Bayesian framework, in which  we measure the distance between functions using the KL divergence of the distributions on the function parameter space. For the transformer structure, the size of the parameter space is independent of the number of agents $N$, which helps us to remove the dependence on $N$. 

Concerning the dependence on the depth $L$ of the neural network, \citet[Theorem~4.6]{edelman2021inductive}  shows that the logarithm of the covering number of the transformer function class scales exponentially in $L$. In contrast, Proposition~\ref{prop:regress_risk} shows that the generalization bound is {\em polynomial} in $L$. We note that  Proposition~\ref{prop:regress_risk} does not contradict  the  exponential dependence shown in \citet{bartlett2017spectrally,neyshabur2017pac}, since we implement the layer normalization to restrict the range of the output. As a byproduct, Proposition~\ref{prop:regress_risk} shows that the invariant of the layer normalization adopted in our paper can greatly reduce the dependence of the generalization error on the depth of the neural network $L$. We note that our results can be generalized to the multi-head attention structure, and the extensions are provided in Appendix~\ref{app:ext}.

\section{Offline MARL with Set Transformers}\label{sec:apptomarl}
In this section, we apply the results in Section~\ref{sec:funcapprox} to \ac{marl}. We implement efficient relational reasoning via the set transformer to obtain  improved suboptimality bounds of the \ac{marl} problem.
In particular, we consider the \emph{homogeneous} \ac{mdp}, where the transition kernel and the reward function are  invariant to permutations of the agents, i.e., for any row-wise permutation function $\bpsi(\cdot)$, we have
\begin{align}
    P^{*}(\bar{\bS}^{\prime}\, |\, \bar{\bS},\bar{\bA})=P^{*}\big(\bpsi(\bar{\bS}^{\prime})\, \big|\, \bpsi(\bar{\bS}),\bpsi(\bar{\bA})\big)\quad\text{ and }\quad r(\bar{\bS},\bar{\bA})=r\big(\bpsi(\bar{\bS}),\bpsi(\bar{\bA})\big) \nonumber
\end{align}
for all $\bar{\bS},\bar{\bS}^{\prime}\in \calS^{N}$ and $\bar{\bA}\in\calA^{N}$. A key property of the homogeneous \ac{mdp} is that there exists a permutation invariant optimal policy, and the corresponding state-value function and the action-value function are also permutation invariant~\cite{li2021permutation}. 
\begin{restatable}{proposition}{piproperty}\label{prop:piproperty}
    For the cooperative homogeneous \ac{mdp}, there exists an optimal policy that is permutation invariant. Also, for any permutation invariant policy $\pi$, the corresponding value function $V_{P^{*}}^{\pi}$ and action-value function $Q_{P^{*}}^{\pi}$ are permutation invariant.
\end{restatable}
Thus, we restrict our attention to   the class of permutation invariant policies $\Pi$, where $\pi(\bar{\bA}\,|\, \bar{\bS})=\pi(\bpsi(\bar{\bA})\,|\,\bpsi(\bar{\bS}))$ for all $\bar{\bA}\in\bar{\calA}$, $\bar{\bS}\in\bar{\calS}$, $\pi\in\Pi$ and all permutations $\bpsi$. For example, if $\pi(\bar{\bA}\, |\, \bar{\bS})=\prod_{i=1}^{N}\mu(\ba_{i}\, |\,\bs_{i})$ for some $\mu$, then $\pi$ is permutation invariant. An optimal policy is any  $\pi^{*}\in \argmax_{\pi\in\Pi} V_{P^{*}}^{\pi}(\bar{\bS}_{0})$. 

\subsection{Pessimistic Model-Free Offline Reinforcement Learning}\label{sec:mfreeRL}
In this subsection, we present a model-free algorithm, in which we adopt the transformer to estimate the action-value function. We also learn a policy based on such an estimate.

\subsubsection{Algorithm}
We modify the single-agent offline \ac{rl} algorithm in \citet{xie2021bellman} to be applicable to the multi-agent case with the transformer approximators, but the analysis is rather different from that in~\citet{xie2021bellman}. Given the dataset $\calD=\{(\bar{\bS}_{i},\bar{\bA}_{i},r_{i},\bar{\bS}_{i}^{\prime})\}_{i=1}^{n}$, we define the mismatch between two functions $f$ and $\tilde{f}$ on $\calD$ for a fixed policy $\pi$ as $\calL(f,\tilde{f},\pi;\calD)= \frac{1}{n}\sum_{(\bar{\bS},\bar{\bA},\barr,\bar{\bS}^{\prime})\in\calD}(f(\bar{\bS},\bar{\bA})-\barr-\gamma \tilde{f}(\bar{\bS}^{\prime},\pi))^{2}$. We adopt the transformer function class $\calF_{\tf}(\bB)$ in Section~\ref{sec:gene_err} to estimate the action-value function and regard $\bX=[\bar{\bS},\bar{\bA}]\in \bbR^{N\times d}$ as the input of the neural network. The dimension  $d=d_{\calS}+d_{\calA}$ and each agent corresponds to a channel in $\bX$. The {\em Bellman error} of a function $f$ with respect to  the policy $\pi$ is defined as $\calE(f,\pi;\calD)=\calL(f,f,\pi;\calD)-\inf_{\tilde{f}\in\calF_{\tf}}\calL(\tilde{f},f,\pi;\calD)$. 

For a fixed policy $\pi$, we construct the confidence region of the action-value function of $\pi$ by selecting the functions in $\calF_{\tf}$ with the $\varepsilon$-controlled Bellman error. We regard the   function attaining the minimum in the confidence region as the estimate of the action-value function of the policy; this reflects the terminology ``pessimism''. Then the optimal policy is learned by maximizing the action-value function estimate. The algorithm can be written formally as
\begin{align}
    \hpi=\argmax_{\pi\in\Pi}\min_{f\in \calF(\pi,\varepsilon)} f(\bar{\bS}_{0},\pi),\quad \text{ where }\quad \calF(\pi,\varepsilon)= \big\{f\in\calF_{\tf}(B) \, \big|\, \calE(f,\pi;\calD)\leq \varepsilon \big\}.\label{algo:mfree}
\end{align}
The motivation for the pessimism originates from the \emph{distribution shift}, where the induced distribution of the learned policy is different from the sampling distribution $\nu$. Such  an issue is severe when there is no guarantee that the sampling distribution $\nu$ supports the visitation distribution $d_{P^{*}}^{\pi^{*}}$ induced by the optimal policy $\pi^{*}$. In fact, the algorithm in Eqn.~\eqref{algo:mfree} does not require the \emph{global} coverage of the sampling distribution $\nu$, where the global coverage means that $d_{P^{*}}^{\pi}(\bar{\bS},\bar{\bA})/\nu(\bar{\bS},\bar{\bA})$ is upper bounded by some constant for all $(\bar{\bS},\bar{\bA})\in\bar{\calS}\times\bar{\calA}$ and all $\pi\in\Pi$. Instead, it only requires {\em partial} coverage, and the mismatch between the distribution induced by the optimal policy $d_{P^{*}}^{\pi^{*}}$ and the sampling distribution $\nu$ is captured by
\begin{align}
    C_{\calF_{\tf}}=\max_{f\in\calF_{\tf}}\bbE_{d^{\pi^{*}}_{P^{*}}}\big[\big(f(\bar{\bS},\bar{\bA})-\calT^{\pi^{*}}f(\bar{\bS},\bar{\bA})\big)^{2}\big]\big/\bbE_{\nu}\big[\big(f(\bar{\bS},\bar{\bA})-\calT^{\pi^{*}}f(\bar{\bS},\bar{\bA})\big)^{2}\big].\label{eqn:CFtf}
\end{align}
We note that $C_{\calF_{\tf}}\leq \max_{(\bar{\bS},\bar{\bA})\in\bar{\calS}\times\bar{\calA}} d_{P^{*}}^{\pi^{*}}(\bar{\bS},\bar{\bA})/\nu(\bar{\bS},\bar{\bA})$, so the suboptimality bound involving $C_{\calF_{\tf}}$ in Theorem~\ref{thm:main_mfree} is tighter than the bound requiring   global convergence \citep{uehara2020minimax}. Similar coefficients also appear in many existing works such as~\cite{xie2021bellman} and \cite{yin2022near}.

\subsubsection{Bound on the Suboptimality Gap}

Before stating the suboptimality bound, We require two assumptions on $\calF_{\tf}$ and the sampling distribution $\nu$. We first state the standard regularity assumption of the transformer function class.
\begin{assumption}\label{assump:mfree1}
    For any $\pi\in\Pi$, we have $\inf_{f\in\calF_{\tf}}\sup_{\mu\in d_{\Pi}} \bbE_{\mu}[(f(\bar{\bS},\bar{\bA})-\calT^{\pi}f(\bar{\bS},\bar{\bA}))^{2}]\leq \varepsilon_{\calF}$ and $\sup_{f\in\calF_{\tf}}\inf_{\tilde{f}\in\calF_{\tf}} \bbE_{\nu}[(\tilde{f}(\bar{\bS},\bar{\bA})-\calT^{\pi}f(\bar{\bS},\bar{\bA}))^{2}]\leq \varepsilon_{\calF,\calF}$, where $d_{\Pi}=\{\mu\ | \ \exists\,  \pi\in \Pi \text{ s.t. }\mu=d^{\pi}_{P^{*}}\}$ is the set of distributions of the state and the action pair induced by any policy $\pi\in\Pi$.
\end{assumption}
This assumption, including the {\em realizability} and the {\em completeness}, states that for any policy $\pi\in\Pi$ there is a function in the transformer function class $\calF_{\tf}$ such that the Bellman error is controlled by $\varepsilon_{\calF}$, and the transformer function class is approximately closed under the Bellman operator $\calT^{\pi}$ for any $\pi\in\Pi$. In addition, we require that the mismatch between the sampling distribution and the visitation distribution of the optimal policy is bounded.
\begin{assumption}\label{assump:mfree2}
    For the sampling distribution $\nu$, the coefficient $C_{\calF_{\tf}}$ defined in Eqn.~\eqref{eqn:CFtf} is finite. 
\end{assumption}
We note that similar assumptions also appear in many existing works~\citep{xie2021bellman,yin2022near}. 

In the analysis of the algorithm in Eqn.~\eqref{algo:mfree}, we first derive a generalization error bound of the estimate of the Bellman error using the PAC-Bayesian framework~\cite{mcallester1999some, mcallester2003simplified}.
\begin{restatable}{theorem}{mainconcen}\label{thm:main_concen}
    Let $\bar{B}=B_{V}B_{QK}B_{a}B_{b}B_{\bw}$.  For all $f,\tilde{f}\in\calF_{\tf}(\bB)$ and all policies $\pi\in\Pi$, with probability at least $1-\delta$, we have 
    \begin{align}
        &\Big|\bbE_{\nu}\Big[\big(f(\bar{\bS},\bar{\bA})-\calT^{\pi}\tilde{f}(\bar{\bS},\bar{\bA})\big)^{2}\Big]-\calL(f,\tilde{f},\pi;\calD)+\calL(\calT^{\pi}\tilde{f},\tilde{f},\pi;\calD)\Big|\nonumber\\
        &\leq \!\frac{1}{2}\bbE_{\nu}\!\Big[\big(f(\bar{\bS},\bar{\bA})\!-\!\calT^{\pi}\tilde{f}(\bar{\bS},\bar{\bA})\big)^{2}\Big]\!+\!O\bigg(\frac{V_{\max}^{2}}{n}\biggl[mL^{2}d^{2}\log\frac{mdL\bar{B} n}{V_{\max}} +\log\frac{\calN(\Pi,1/n,d_{\infty})}{\delta}\biggr]\bigg).\nonumber
    \end{align}
\end{restatable}
For ease of notation, we define $e(\calF_{\tf},\Pi,\delta,n)$ to be $n$ times the second term of the generalization error bound. We note that the generalization error bound in Theorem~\ref{thm:main_concen} is independent of the number of agents, which will help us to remove the dependence on the number of agents in the suboptimality of the learned policy. The suboptimality gap of the learned policy $\hpi$ can be upper bounded as the following.
\begin{restatable}{theorem}{mfreethm}\label{thm:main_mfree}
If Assumptions~\ref{assump:mfree1} and~\ref{assump:mfree2} hold, and we take $\varepsilon=3\varepsilon_{\calF}/2+2e(\calF_{\tf},\Pi,\delta,n)/n$, then with probability at least $1-\delta$, the suboptimality gap of the policy derived in the algorithm shown in Eqn.~\eqref{algo:mfree} is upper bounded as 
\begin{align}
    \!V_{P^{*}}^{\pi^{*}}(\bar{\bS}_{0})\!-\!V_{P^{*}}^{\hpi}(\bar{\bS}_{0})\! \leq \!  O\Bigg(\frac{\sqrt{C_{\calF_{\tf}}\tilde{\varepsilon}  }}{1-\gamma}\!+\!\frac{V_{\max}\sqrt{ C_{\calF_{\tf}}  }}{(1-\gamma)\sqrt{n} }\sqrt{mL^{2}d^{2}\log\frac{mdL\bar{B}n}{V_{\max}}\!+\!\log \frac{2\calN(\Pi,1/n,d_{\infty})}{\delta} }\Bigg),\nonumber
\end{align}
    where $d=d_{\calS}+d_{\calA}$, $\tilde{\varepsilon} = \varepsilon_\calF + \varepsilon_{\calF,\calF}$, and $\bar{B}$ is defined in Proposition \ref{thm:main_concen}.
\end{restatable}
Theorem~\ref{thm:main_mfree} shows that the upper bound of the suboptimality gap does not scale with the number of agents $N$, which demonstrates that the proposed model-free algorithm breaks the curse of many agents. We note that the model-free offline/batch \ac{marl} with homogeneous agents has been studied in \cite{chen2021pessimism} and \cite{li2021permutation}, and the suboptimality upper bounds in \citet[Theorem 1]{chen2021pessimism} and \citet[Theorem~4.1]{li2021permutation} are also independent of $N$. However, these works adopt the mean-field approximation of the original \ac{mdp}, in which the influence of all the other agents on a specific agent is only coarsely considered through the distribution of the state. The approximation error between the action-value function of the mean-field \ac{mdp} and that of the original \ac{mdp} is not analyzed therein. Thus, the independence of $N$ in their works comes with the cost of the poor relational reasoning ability and the unspecified approximation error. In contrast, we analyze the suboptimality gap of the learned policy in the original \ac{mdp}, and the interaction among agents is captured by the transformer network.

\subsection{Pessimistic Model-based Offline Reinforcement Learning}\label{sec:mbasedRL}
In this subsection, we present the model-based algorithm, where we adopt the transformer to estimate the system dynamics and learn the policy based on such an estimate.
\subsubsection{Neural Nonlinear Regulator}
In this section, we consider the \ac{nnr}, in which we use the transformer to estimate the system dynamics. The ground truth transition $P^{*}(\bar{\bS}^{\prime}\, |\,\bar{\bS},\bar{\bA})$ is defined as $\bar{\bS}^{\prime}=\bF^{*}(\bar{\bS},\bar{\bA})+\bar{\bvarepsilon}$, where $\bF^{*}$ is a nonlinear function, $\bar{\bvarepsilon}=[\bvarepsilon_{1},\ldots,\bvarepsilon_{N}]^{\top}$is the noise, and $\bvarepsilon_{i}\sim \calN(0,\sigma^{2}\bI_{d\times d})$ for $i\in[N]$ are independent random vectors. We note that the function $\bF^{*}$ and the transition kernel $P^{*}$ are equivalent, and we denote the transition kernel corresponding to the function $\bF$ as $P_{\bF}$. Since the transition kernel $P^{*}(\bar{\bS}^{\prime}\, |\,\bar{\bS},\bar{\bA})$ is permutation invariant, $\bF^{*}$ should be  permutation equivariant, i.e., $\bF^{*}(\bpsi(\bar{\bS}),\bpsi(\bar{\bA}))=\bpsi(\bF^{*}(\bar{\bS},\bar{\bA}))$ for all row-wise permutation functions $\bpsi(\cdot)$. 

We take $\bX=[\bar{\bS},\bar{\bA}]\in \bbR^{N\times d}$ as the input of the network and adopt a similar network structure as the transformer specified in Section~\ref{sec:gene_err}. However, to predict the next state instead of the action-value function with the transformer, we remove the average aggregation module in the final layer of the structure in Section~\ref{sec:gene_err}. Please refer to Appendix~\ref{app:tfdef} for the formal definition. The permutation equivariance of the proposed transformer structure can be easily proved with the permutation equivariance of the self-attention mechanism. We consider the transformer function class with bounded parameters, which is defined as
\begin{align}
    \calM_{\tf}(\bB^{\prime})=\Big\{&\bF_{\tf}(\bX;\bW_{QK}^{1:L},\bW_{V}^{1:L},\ba^{1:L},\bb^{1:L})\, \Big| \, \big|a_{kj}^{(i)}\big|< B_{a}, \big\|\bb_{kj}^{(i)}\big\|_{2}< B_{b}, \nonumber\\
    &\big\|\bW_{QK}^{(i)\top}\big\|_{\rmF}< B_{QK},\big\|\bW_{V}^{(i)\top}\big\|_{\rmF}< B_{V} \text{ for } i\in[L],j\in[m],k\in[d]\Big\},\nonumber
\end{align}
where $\bB^{\prime}=[B_{a},B_{b},B_{QK},B_{V}]$ is  the vector of parameters of the function class. We denote $\calM_{\tf}(\bB^{\prime})$ as $\calM_{\tf}$ when the parameters are clear from the context.

\subsubsection{Algorithm}
Given the offline dataset $\calD=\{(\bar{\bS}_{i},\bar{\bA}_{i},r_{i},\bar{\bS}_{i}^{\prime})\}_{i=1}^{n}$, we first derive the \ac{mle} of the system dynamics. Next, we learn the optimal policy according to the confidence region of the dynamics that are constructed around the \ac{mle}. The term  ``pessimism'' is reflected in the procedure that we choose the system dynamics that induce the {\em smallest} value function, i.e.,  
\begin{align}
    \hat{\bF}_{\rm MLE}=\argmin_{\bF\in\calM_{\tf}} \frac{1}{n}\sum_{i=1}^{n}\big\|\bar{\bS}_{i}^{\prime}-\bF(\bar{\bS}_{i},\bar{\bA}_{i})\big\|_{\rmF}^{2} \quad\mbox{and}\quad \hpi=\argmax_{\pi\in\Pi} \min_{\bF\in\calM_{\rm MLE}(\zeta)} V_{P_{\bF}}^{\pi}(\bar{\bS}_{0}),\label{algo:mpolicy}
\end{align}
where $\calM_{\rm MLE}(\zeta)=\{\bF\in \calM_{\tf}(\bB^{\prime})\, |\, 1/n \cdot \sum_{i=1}^{n}\tv(P_{\bF}(\cdot\, |\, \bar{\bS}_{i},\bar{\bA}_{i}),\hat{P}_{\rm MLE}(\cdot\, |\, \bar{\bS}_{i},\bar{\bA}_{i}))^{2}\leq \zeta\}$ is the confidence region, which has a closed-form expression in terms of the difference between $\bF$ and $\hat{\bF}_{\rm MLE}$ as stated in Appendix~\ref{app:equialgo}. The transition kernel induced by $\hat{\bF}_{\rm MLE}$ is denoted as $\hat{P}_{\rm MLE}$. The parameter $\zeta$ is used to measure the tolerance of estimation error of the system dynamics, and it is set to according to the parameters of the function class such that $\bF^{*}$ belongs to $\calM_{\rm MLE}(\zeta)$ with high probability. 

Similar to the model-free algorithm, the model-based algorithm specified in Eqn.~\eqref{algo:mpolicy} does not require   global coverage. Instead, the mismatch between the distribution induced by the optimal policy $d_{P^{*}}^{\pi^{*}}$ and the sampling distribution $\nu$ is captured by the constant
\begin{align}
    \! C_{\calM_{\tf}}\!=\! \max_{F\in\calM_{\tf}}\! \bbE_{d_{P^{*}}^{\pi^{*}}}\big[\tv\big(P_{F}(\cdot\,|\,\bar{\bS},\bar{\bA}), P^{*}(\cdot\,|\,\bar{\bS},\bar{\bA})\big)^{2}\big]\Big/\bbE_{\nu}\big[\tv\big(P_{F}(\cdot\,|\,\bar{\bS},\bar{\bA}), P^{*}(\cdot\,|\,\bar{\bS},\bar{\bA})\big)^{2}\big]. \!\label{eqn:CMtf}
\end{align}
We note that $C_{\calM_{\tf}}\leq \max_{(\bar{\bS},\bar{\bA})\in\bar{\calS}\times\bar{\calA}} d_{P^{*}}^{\pi^{*}}(\bar{\bS},\bar{\bA})/\nu(\bar{\bS},\bar{\bA})$, so the suboptimality bound involving $C_{\calP_{\calF_{\tf}}}$ in Theorem~\ref{thm:main_mbased} is tighter than the bound requiring  global convergence. Similar coefficients also appear in many existing works such as \cite{sun2019model} and \cite{chang2021mitigating}.

\subsubsection{Analysis of the Maximum Likelihood Estimate}
Every $\bF \in \calM_{\rm MLE}(\zeta)$ is near to the \ac{mle} in the total variation sense and thus well approximates the ground truth system dynamics. Therefore, to derive an upper bound of the suboptimality gap of the learned policy, we first analyze the convergence rate of the \ac{mle} $\hat{P}_{\rm MLE}$ to $P^*$.
\begin{restatable}{proposition}{mlebound}\label{prop:mlebound}
    Let $\tilde{B}=B_{V}B_{QK}B_{a}B_{b}$. For the maximum likelihood estimate $\hat{P}_{\rm MLE}$ in Eqn.~\eqref{algo:mpolicy}, the following inequality holds with probability at least $1-\delta$,
    \begin{align}
        \bbE_{\nu}\Big[\tv\big(P^{*}(\cdot\,|\,\bar{\bS},\bar{\bA}), \hat{P}_{\rm MLE}(\cdot\,|\,\bar{\bS},\bar{\bA})\big)^{2}\Big]\leq O\bigg(\frac{1}{n}mL^{2}d^{2}\log\big(NLmd\tilde{B} n\big)+\frac{1}{n}\log\frac{1}{\delta}\bigg).\nonumber
    \end{align}
\end{restatable}
We define $e^{\prime}(\calM_{\tf},n)$ to be $n$ times the total variation bound. Proposition~\ref{prop:mlebound} shows that the total variation estimation error is polynomial in the depth of the neural network $L$. However, different from the model-free \ac{rl} results in Section~\ref{sec:mfreeRL}, the estimation error of \ac{mle} $\hat{P}_{\rm MLE}$ is logarithmic in the number of agents $N$. We note that this  logarithm dependency on $N$ comes from the fact that $\tv(P^{*}(\cdot\,|\,\bar{\bS},\bar{\bA}), \hat{P}_{\rm MLE}(\cdot\,|\,\bar{\bS},\bar{\bA}))$ measures the distance between two transition kernels that involves the states of $N$ agents, different from the scalar estimate of the value function in Section~\ref{sec:mfreeRL}. To prove the result, we adopt a  PAC-Bayesian framework to analyze the convergence rate of \ac{mle}, which is inspired by the analysis of  density estimation \citep{zhang2006}; more details are presented in Appendix~\ref{app:propmlebound}.

\subsubsection{Bound on the Suboptimality Gap}
To analyze the error of the learned model, we make the following realizability assumption.
\begin{assumption}\label{assump:mbased1}
    The nominal system dynamics belongs to the function class $\calM_{\tf}$, i.e., $\bF^{*}\in \calM_{\tf}(\bB^{\prime})$.
\end{assumption}
In addition, we require that the mismatch between the sampling distribution and the visitation distribution of the optimal policy is bounded.
\begin{assumption}\label{assump:mbased2}
    For the sampling distribution $\nu$, the coefficient $C_{\calM_{\tf}}$ defined in \eqref{eqn:CMtf} is finite. 
\end{assumption}
We note that these two assumptions are  also made  in many existing works, e.g., \cite{chang2021mitigating,uehara2021pessimistic}. 
\begin{restatable}{theorem}{mbasedmain}\label{thm:main_mbased}
    If Assumptions~\ref{assump:mbased1} and~\ref{assump:mbased2} hold, and we take $\zeta=c_{1}e^{\prime}(\calM_{\tf},n)/n$ for some constant $c_{1}>0$, then with probability at least $1-\delta$, the suboptimality gap of the policy learned in the algorithm in Eqn.~\eqref{algo:mpolicy} is upper bounded as
    \begin{align}
        V_{P^{*}}^{\pi^{*}}(\bar{\bS}_{0})-V_{P^{*}}^{\hpi}(\bar{\bS}_{0}) 
        \leq O\Bigg(\frac{V_{\max}}{(1-\gamma)^{2}} \sqrt{C_{\calM_{\tf}}\bigg(\frac{1}{n}mL^{2}d^{2}\log\big(NLmd\tilde{B} n\big)+\frac{1}{n}\log\frac{1}{\delta}\bigg)}\Bigg),\nonumber
    \end{align}
    where $d=d_{\calS}+d_{\calA}$, and $\tilde{B}$ is defined in Proposition \ref{prop:mlebound}.
\end{restatable}
Theorem~\ref{thm:main_mbased} presents an upper bound on the  suboptimality gap of the offline model-based \ac{rl} with the transformer approximators. The suboptimality gap depends on the number of agents only as $O(\sqrt{\log N})$, which shows that the proposed model-based \ac{marl} algorithm mitigates the curse of many agents. This weak dependence on $N$ originates from   measuring the distance between two system dynamics of $N$ agents in the learning of the dynamics. To the best of our knowledge,  there is no prior work on  analyzing the model-based algorithm for the homogeneous \ac{marl}, even from the mean-field approximation perspective. The proof of Theorem~\ref{thm:main_mbased} leverages novel analysis of the \ac{mle} in Proposition~\ref{prop:mlebound}. For more details, please refer to Appendix~\ref{app:thmmain_mbased}. 
 
\section{Experimental Results}
We evaluate the performance of the algorithms on the \ac{mpe}~\citep{lowe2017multi,mordatch2018emergence}. We focus on the \emph{cooperative navigation} task, where $N$ agents move cooperatively to cover $L$ landmarks in an environment. Given the positions of the $N$  agents $\bx_{i}\in\bbR^{2}$ (for $i\in[N]$) and the positions of the  $L$ landmarks  $\by_{j}\in\bbR^{2}$ (for $j\in[L]$), the agents receive  reward $r=-\sum_{j=1}^{L}\min_{i\in[N]}\|\by_{j}-\bx_{i}\|_{2}.$ This reward encourages the agents to move closer to the landmarks. We set the number of agents as $N=3,6,15,30$ and the number of landmarks as $L=N$. Here, we only present the result for $N=3, 30$. Please refer to Appendix~\ref{app:exp} for more numerical results. To collect an offline dataset, we learn a policy in the online setting. Then  the offline dataset is collected from the induced stationary distribution of such a policy. 

\begin{figure}[t]
    \centering
    \begin{tabular}{cc}
        \hspace{-.1in}  \includegraphics[width=2.9in]{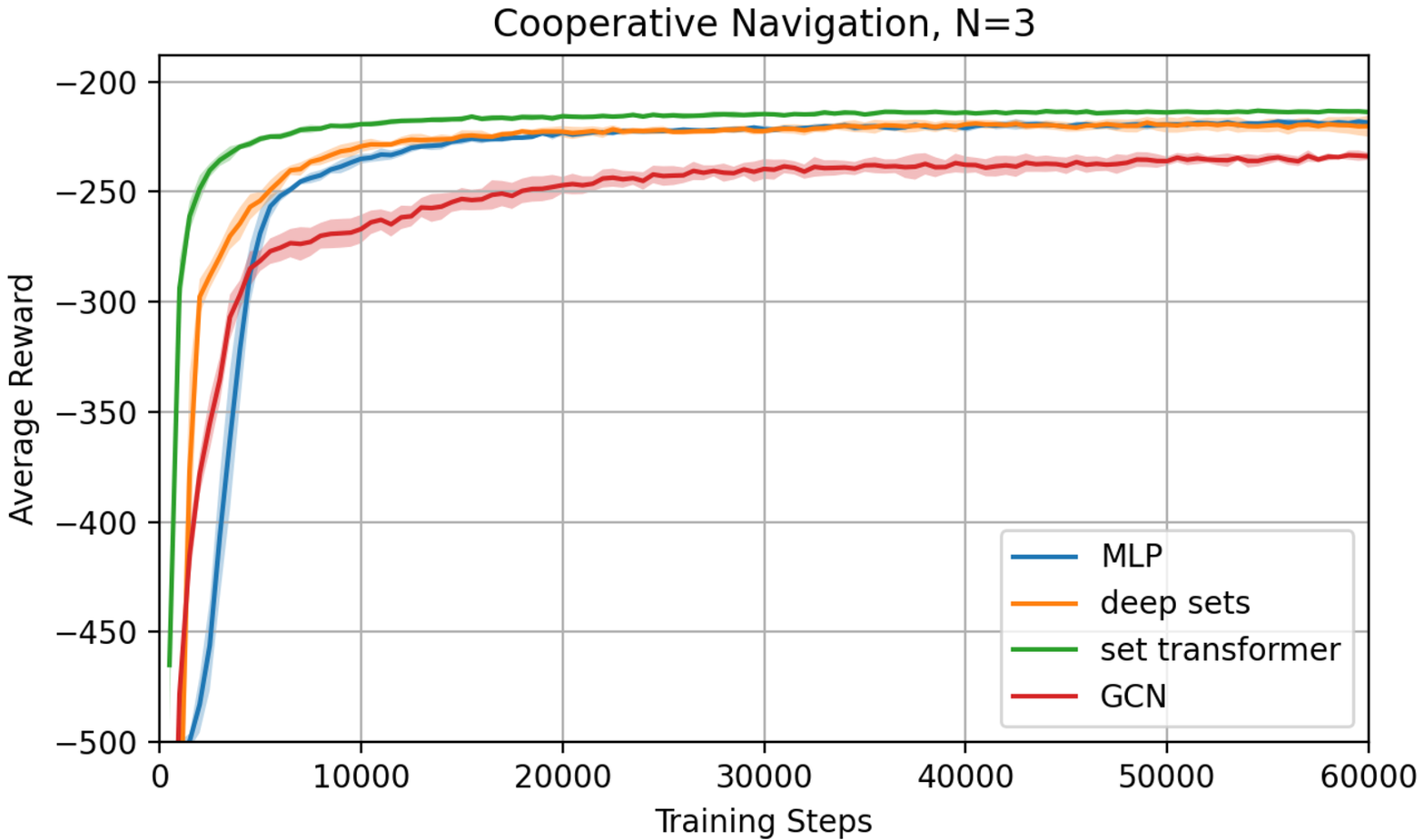} \hspace{-.3in} &   
          \includegraphics[width=2.9in]{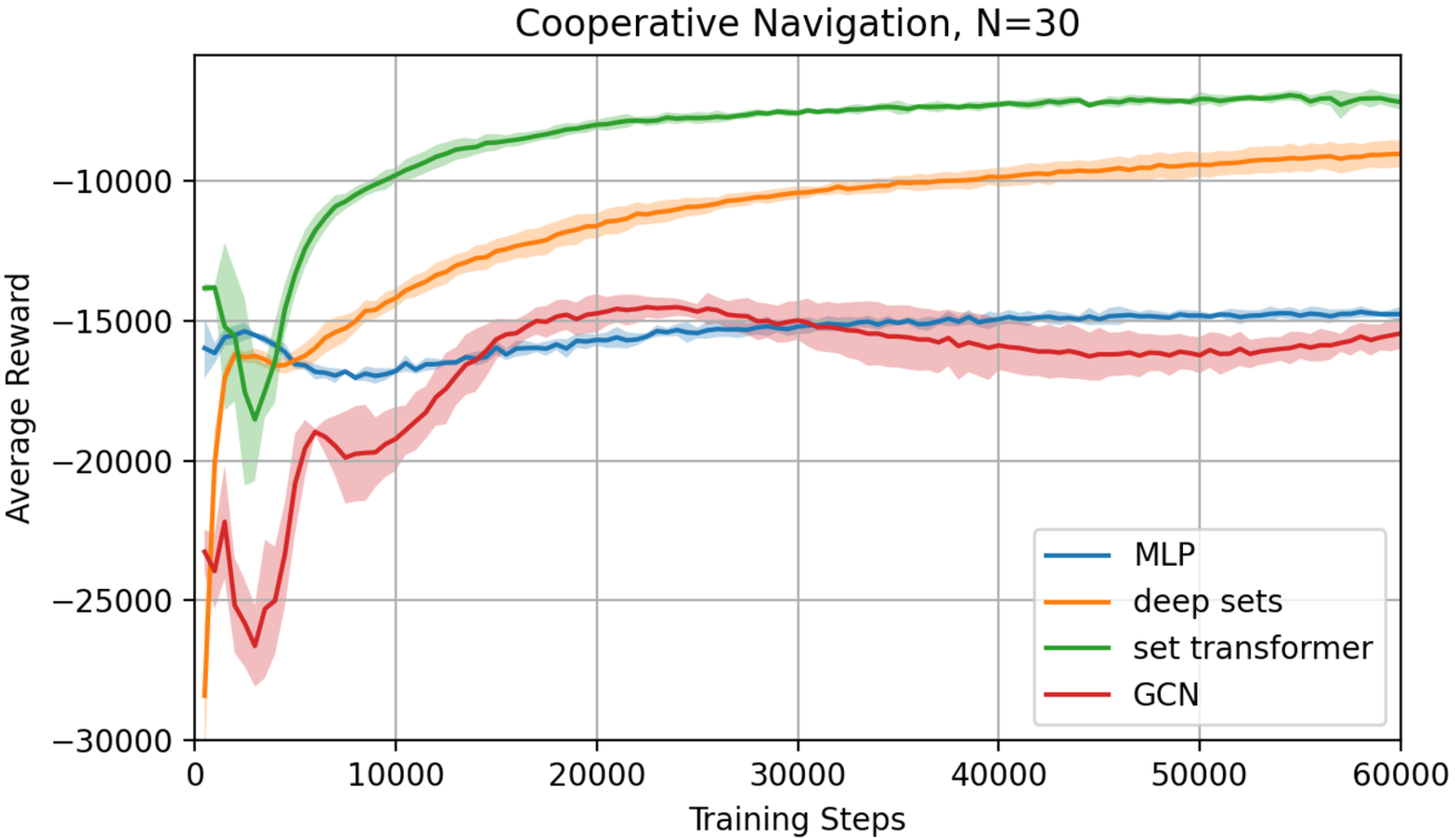}
    \end{tabular}
    \caption{Average rewards of   model-free \ac{rl} algorithms with their standard deviations for $N=3,30$.}
    \label{fig:free}
\end{figure}

    
    
We use   \ac{mlp}, deep sets, \ac{gcn}~\citep{liu2020pic}, and set transformer to estimate the value function. We note that the deep sets, \ac{gcn}, and set transformer are   permutation invariant functions. We use the code   in~\cite{zaheer2017deep} for the  implementation of the deep sets and set transformer. For other   implementation details, please refer to Appendix~\ref{app:exp}.

Figure~\ref{fig:free} shows that the performances of the \ac{mlp} and deep sets are worse than that of the set transformer. This is due to  the poor relational reasoning abilities of \ac{mlp} and deep sets, which corroborates   Theorem~\ref{thm:approx}. Figure~\ref{fig:free} indicates that  when the number of agents $N$ increases, the superiority of the algorithm with set transformer becomes more pronounced, which is strongly aligned with  our theoretical result in Theorem~\ref{thm:main_mfree}.

\section{Concluding Remarks}\label{app:conclude}
In view of  the tremendous empirical successes of  cooperative \ac{marl} with permutation invariant agents, it is imperative to develop a firm theoretical understanding of this \ac{marl} problem   because it will inspire the design of even more efficient algorithms. In this work, we   design and analyze   algorithms that break the curse of many agents and, at the same time, implement efficient relational reasoning. Our algorithms and analyses serve as a first step towards developing provably efficient \ac{marl} algorithms with permutation invariant approximators. We leave the extension of our results of the transformer to  {\em general permutation invariant approximators} as future works.


\paragraph{Acknowledgements}
Fengzhuo Zhang and Vincent Tan  acknowledge funding from a Singapore National Research Foundation (NRF) Fellowship (A-0005077-01-00)  and Singapore Ministry of Education (MOE) AcRF Tier 1 Grants (A-0009042-01-00 and A-8000189-01-00). Zhaoran Wang acknowledges the National Science Foundation (Awards 2048075, 2008827, 2015568, 1934931), Simons Institute (Theory of Reinforcement Learning), Amazon, J.~P.~Morgan, and Two Sigma for their support.

\newpage

\bibliographystyle{abbrvnat}
\bibliography{TGroup}

\clearpage
\appendix

\begin{center}
{\Large {\bf Supplementary Materials for\\ ``Relational Reasoning via Set Transformers:\\ Provable Efficiency and Applications to MARL''}}
\end{center}

\setcounter{equation}{0}
\counterwithin*{equation}{section}

\renewcommand{\theequation}{\thesection.\arabic{equation}}

\section{Formal Definition of the Fully-Connected Networks in Section~\ref{sec:funcapprox}}\label{app:fcdef}
For a multi-channel input, the output is the sum of the output of each channel, i.e., $$[\bphi_{\relu}]_{i}(\bX)=\sum_{k=1}^{N}[\bphi_{\relu}]_{i}(\bX_{k,:})\quad \mbox{for}\quad  i\in[W_{2}],$$  where $\bX_{k,:}$ is the $k^{\rm th}$ row of $\bX$. The fully-connected neural network for each channel is defined as
\begin{align}
 [\bphi_{\relu}]_{i}(\bx)&=\sum_{j=1}^{W_{1}} c_{ij} \relu\big(\ba_{j}^{\top}\bx+b_{j}\big)+d_{i}\quad  \text{ for }\quad  i\in[W_{2}],\nonumber
\end{align}
where $\ba_{j}\in \bbR^{d}$ and $b_{j},c_{ij},d_{i}\in\bbR$ for $i\in[W_{2}],j\in[W_{1}]$ are the parameters of $\bphi_{\relu}$. The network $\rho_{\relu}$ is defined as
\begin{align}
    \rho_{\relu}(\by)&=\sum_{k=1}^{W_{3}} g_{k} \relu\big(\be_{k}^{\top}\by+f_{k}\big)+h,\nonumber
\end{align}
where $\be_{i}\in\bbR^{W_{2}}$ and $f_{k},g_{k},h\in\bbR$ for $k\in[W_{3}]$ are the parameters of $\rho$.

\section{Formal Definition of the Transformer Structures in Sections~\ref{sec:mfreeRL} and~\ref{sec:mbasedRL}}\label{app:tfdef}
\textbf{The transformer structure in Section~\ref{sec:mfreeRL}.} In each layer, we combine the self-attention mechanism with the \underline{R}ow-wise \underline{F}eed\underline{F}orward (rFF) single-hidden layer neural network. rFF takes $\bX\in\bbR^{N\times d}$ as the input and outputs a matrix in the same dimension. It applies a single-hidden layer network in a row-wise manner. For the entry in the $i^{\rm th}$ row and the $k^{\rm th}$ column of the output, we have  
\begin{align}
    \big[\rff(\bX,\ba,\bb)\big]_{i,k}= \big[\rff(\bX_{i,:},\ba,\bb)\big]_{k} \text{ for }k\in[d],i\in[N],\nonumber
\end{align}
where $\bX_{i,:}\in\bbR^{d}$ is the $i^{\rm th}$ row of $\bX$. For a $d$-dimensional vector input, the single-hidden layer outputs a vector in the same dimension as
\begin{align}
    \big[\rff(\bx,\ba,\bb)\big]_{k}=\sum_{j=1}^{m}a_{kj}\relu(\bb_{kj}^{\top}\bx) \text{ for }k\in[d],\nonumber
\end{align}
where $\bx\in\bbR^{d}$ is the input, $m$ is the width of the network, and $\ba=[a_{11},a_{12},\ldots,a_{dm}]\in\bbR^{dm}$ and $\bb=[\bb_{11},\bb_{12},\ldots,\bb_{dm}]\in\bbR^{d\times dm}$ are the parameters of rFF. 

Then for any layer $i\in[L-1]$, the layer output is
\begin{align}
    \bG_{\tf}^{(i+1)}= \Pi_{\rm{norm}} \Big[\Att\big(\bG_{\tf}^{(i)}\bW_{QK}^{(i+1)},\bG_{\tf}^{(i)},\bG_{\tf}^{(i)}\bW_{V}^{(i+1)}\big)+\rff\big(\bG_{\tf}^{(i)},\ba^{(i+1)},\bb^{(i+1)}\big)\Big],\label{eq:net1}
\end{align}
where
\begin{align}
    \ba^{1:i}&= [\ba^{(1)},\ldots,\ba^{(i)}],\nonumber\\
    \bb^{1:i}&= [\bb^{(1)},\ldots,\bb^{(i)}],\nonumber\\
    \bW_{QK}^{1:i}&= [\bW_{QK}^{(1)},\ldots,\bW_{QK}^{(i)}],\nonumber\\*
    \bW_{V}^{1:i}&= [\bW_{V}^{(1)},\ldots,\bW_{V}^{(i)}],\nonumber
\end{align}
are the stacked parameters of the first $i$ layers of the network, and $\bG_{\tf}^{(i)}$ is a shorthand for $\bG_{\tf}^{(i)}(\bX;\bW_{QK}^{1:i},\bW_{V}^{1:i},\ba^{1:i},\bb^{1:i})$. $\Pi_{\rm{norm}}(\bX)$ is the row-wise normalization function, which projects each row of $\bX$ into the $\ell_p$-ball ( where $p\geq1$). We take $\bG_{\tf}^{(0)}(\bX)= \Pi_{\rm{norm}}\big(\bX\big)$ as the input of the first layer. For the last layer $L$, we derive the scalar estimate of the action-value function with the average aggregation among all the channels, i.e.,
\begin{align}
    g_{\tf}(\bX;\bW_{QK}^{1:L},\bW_{V}^{1:L},\ba^{1:L},\bb^{1:L},\bw)&=\Pi_{V_{\max}}\bigg(\frac{1}{N}\bbI_{N}\bG_{\tf}^{(L)}(\bX;\bW_{QK}^{1:L},\bW_{V}^{1:L},\ba^{1:L},\bb^{1:L})\bw\bigg),\nonumber
\end{align}
 where $\Pi_{V_{\max}}(x)$ is the ``clipping'' function, which is defined as $\Pi_{V_{\max}}(x)= x$ if $|x|\leq V_{\max}$ and $\Pi_{V_{\max}}(x)= V_{\max}\mathrm{sign}(x)$ otherwise.
 
\textbf{The transformer structure in Section~\ref{sec:mbasedRL}.}
 For the layer $i\in [L-2]$, we adopt the same neural network structure in Eqn.~\eqref{eq:net1}. For the final layer, we implement the structure that
\begin{align}
    &\bF_{\tf}(\bX;\bW_{QK}^{1:L},\bW_{V}^{1:L},\ba^{1:L},\bb^{1:L})\nonumber\\
    &\quad= \SM\big(\bG_{\tf}^{(L-1)}\bW_{QK}^{(L)}\bG_{\tf}^{(L-1)\top}\big)\bG_{\tf}^{(L-1)}\bW_{V}^{(L)}+\rff\big(\bG_{\tf}^{(L-1)},\ba^{(L)},\bb^{(L)}\big).\nonumber
\end{align}

\section{Equivalent Expression for the Model-based RL algorithm in Section~\ref{sec:mbasedRL}}\label{app:equialgo}
The algorithm in Eqn.~\eqref{algo:mpolicy} can be equivalently expressed in two forms.

\textbf{Transition Function.} The algorithm in Eqn.~\eqref{algo:mpolicy} can be expressed with the transition function $F$ as
\begin{align*}
    \hat{\bF}_{\rm MLE}=\argmin_{\bF\in\calM_{\tf}} \frac{1}{n}\sum_{i=1}^{n}\big\|\bar{\bS}_{i}^{\prime}-\bF(\bar{\bS}_{i},\bar{\bA}_{i})\big\|_{\rmF}^{2} \quad \mbox{and}\quad\hpi=\argmax_{\pi\in\Pi} \min_{\bF\in\calM_{\rm MLE}(\zeta)} V_{P_{\bF}}^{\pi}(\bar{\bS}_{0}),
\end{align*}
where the ``confidence region'' $\calM_{\rm MLE}(\zeta)$ is the set of all $\bF\in \calM_{\tf}$ such that  \citep{devroye2018total}
$$
\frac{1}{n}\sum_{i=1}^n \bigg(2\Phi\bigg(\sqrt{\frac{\bigl\|\bF(\bar{\bS}_{i},\bar{\bA}_{i})-\hat{\bF}_{\rm MLE}(\bar{\bS}_{i},\bar{\bA}_{i})\bigr\|_{\rmF}^{2}}{2\sigma^{2}}}\bigg)-1\bigg)^2\le\zeta,
$$ 
and $\Phi(\cdot)$ is the cumulative distribution function of the standard normal distribution.

\textbf{Transition Probability.} The algorithm in Eqn.~\eqref{algo:mpolicy} can also be expressed with the transition probability $P$. Since the function $\bF$ is equivalent to the transition kernel $P_{\bF}$, the transition kernel class can be correspondingly defined as
\begin{align}
    \calP_{\tf}(\bB^{\prime})=\big\{P \, \big|\, \exists\,  \bF\in \calM_{\tf}(\bB^{\prime})\text{ s.t. }P=P_{\bF}\big\}.\nonumber
\end{align}
Then the algorithm can be expressed as
\begin{align*}
    \hat{\bP}_{\rm MLE}=\argmax_{P\in\calP_{\tf}} \sum_{i=1}^{n}\log P(\bar{\bS}_{i}^{\prime}\,|\,\bar{\bS}_{i},\bar{\bA}_{i}) \quad \mbox{and}\quad\hpi=\argmax_{\pi\in\Pi} \min_{P\in\calP_{\rm MLE}(\zeta)} V_{P}^{\pi}(\bar{\bS}_{0}),
\end{align*}
where the confidence region $\calP_{\rm MLE}(\zeta)$ is defined as
\begin{align}
    \calP_{\rm MLE}(\zeta)=\bigg\{P\in \calP_{\tf}\, \bigg|\, \frac{1}{n}\sum_{i=1}^{n}\tv\bigl(P(\cdot\, |\, \bar{\bS}_{i},\bar{\bA}_{i}),\hat{P}_{\rm MLE}(\cdot\, |\, \bar{\bS}_{i},\bar{\bA}_{i})\bigr)^{2}\leq \zeta\bigg\}.\nonumber
\end{align}

\section{Proof of Propositions~\ref{prop:piproperty}}
\begin{proof}[Proof of Proposition~\ref{prop:piproperty}]
    We denote any optimal policy as $\pi^{*}=\argmax_{\pi} V_{P^{*}}^{\pi}(\bar{\bS}_{0})$. Note that the optimal policy may be not unique, and any policy that achieves the maximal value function is called an optimal policy. The corresponding action-value function is denoted as $Q_{P^{*}}^{*}$, which is defined as
    \begin{align}
        Q_{P^{*}}^{*}(\bar{\bS},\bar{\bA})=\bbE_{\bar{\bS}^{\prime}\sim P^{*}(\cdot\,|\,\bar{\bS},\bar{\bA})}\big[r(\bar{\bS},\bar{\bA})+\max_{\bar{\bA}^{\prime}}Q_{P^{*}}^{*}(\bar{\bS}^{\prime},\bar{\bA}^{\prime})\big].\label{eq:2}
    \end{align}
    For any row-wise permutation function $\bpsi(\cdot)$, we have
    \begin{align}
        Q_{P^{*}}^{*}\big(\bpsi(\bar{\bS}),\bpsi(\bar{\bA})\big)&= \bbE_{\bpsi(\bar{\bS}^{\prime})\sim P^{*}(\cdot\,|\,\bpsi(\bar{\bS}),\bpsi(\bar{\bA}))}\Big[r\big(\bpsi(\bar{\bS}),\bpsi(\bar{\bA})\big)+\max_{\bar{\bA}^{\prime}}Q_{P^{*}}^{*}\big(\bpsi(\bar{\bS}^{\prime}),\bpsi(\bar{\bA}^{\prime})\big)\Big]\nonumber\\*
        &=\bbE_{\bar{\bS}^{\prime}\sim P^{*}(\cdot\,|\,\bar{\bS},\bar{\bA})}\Big[r(\bar{\bS},\bar{\bA})+\max_{\bar{\bA}^{\prime}}Q_{P^{*}}^{*}\big(\bpsi(\bar{\bS}^{\prime}),\bpsi(\bar{\bA}^{\prime})\big)\Big]\label{eq:3},
    \end{align}
    where Eqn.~\eqref{eq:3} follows from the homogeneity of the \ac{mdp}. Since $Q_{P^{*}}^{*}$ is the unique solution of Eqn.~\eqref{eq:2}, we have  $Q_{P^{*}}^{*}(\bar{\bS},\bar{\bA})=Q_{P^{*}}^{*}(\bpsi(\bar{\bS}),\bpsi(\bar{\bA}))$ for all $\bpsi(\cdot)$. Thus, the permutation invariant policy $\pi(\bar{\bS})=\argmax_{\bar{\bA}} Q_{P^{*}}^{*}(\bar{\bS},\bar{\bA})$ is the optimal policy.
    
    When the policy $\pi$ is permutation invariant, we can show that the corresponding action-value function and the value function are permutation invariant following the similar argument as above. Therefore, we conclude the proof of Proposition \ref{prop:piproperty}.
\end{proof}

\section{Proof of Proposition~\ref{prop:regress_risk}}
\begin{proof}[Proof of Proposition~\ref{prop:regress_risk}]
    We note that Proposition~\ref{prop:regress_risk} is a corollary of Theorem~\ref{thm:main_concen}. Take $\tilde{f}=0$ in Theorem~\ref{thm:main_concen}, then we recover the result of Proposition~\ref{prop:regress_risk}. Thus, we only provide the proof of Theorem~\ref{thm:main_concen} in Appendix~\ref{app:propmain_concern}.
\end{proof}

\section{Proof of Theorem~\ref{thm:approx}}\label{app:thmapprox}
\begin{proof}[Proof of Theorem~\ref{thm:approx}]
    The functions in $\calN(W)$ are the fully-connected networks with the $\relu$ activation, so they are piece-wise linear functions on $[0,1]^{N\times d}$, where the number of the linear pieces are polynomial in the width of the network. In contrast, the self-attention function is convex on some subset of $[0,1]^{N\times d}$. In the following proof procedures, we specify a line in $[0,1]^{N\times d}$ where the second derivative of the self-attention function is high enough such that $\calN(W)$ should be exponentially wide to approximate the self-attention function on the longest linear piece of that line.
    
    To specify a line in $[0,1]^{N\times d}$, we set the inputs of all but the first channels to be $\bx$, and set the input of the first channel to be a scaled version of $\bx$. Fix any $\bx\in [0,1]^{d}$ and $k\in\bbR$, we set $\bx_{1}=k\bx$ and $\bx_{i}=\bx$ for all $i\in\{2,\ldots,N\}$. For $\bX=[\bx_{1},\ldots,\bx_{N}]^{\top}$, $\bw\in [0,1]^{d}$ and $a\in\bbR$, we define 
    \begin{align}
        f(a,\bX,\bw)&= \bbI_{N}^{\top}\Att(a\bX,a\bX,a\bX)\bw\nonumber\\
        &=ak\bx^{\top}\bw \bigg[\frac{e^{a^{2}k^{2}\bx^{\top}\bx}}{e^{a^{2}k^{2}\bx^{\top}\bx}+(N-1)e^{a^{2}k\bx^{\top}\bx}}+\frac{(N-1)e^{a^{2}k\bx^{\top}\bx}}{e^{a^{2}k\bx^{\top}\bx}+(N-1)e^{a^{2}\bx^{\top}\bx}}\bigg]\nonumber\\*
        &\qquad+ak\bx^{\top}\bw(N-1)\bigg[\frac{e^{a^{2}k\bx^{\top}\bx}}{e^{a^{2}k^{2}\bx^{\top}\bx}+(N-1)e^{a^{2}k\bx^{\top}\bx}}+\frac{(N-1)e^{a^{2}\bx^{\top}\bx}}{e^{a^{2}k\bx^{\top}\bx}+(N-1)e^{a^{2}\bx^{\top}\bx}}\bigg],\nonumber
    \end{align}
    where $\bbI_{N}\in\bbR^{N}$ is the vector with all entries  being equal to  1. The partial derivatives of $f(a,\bX,\bw)$ with respect to $a$ can be derived as
    \begin{align}
        \frac{\partial f(a,\bX,\bw)}{\partial a}&=\bigl[2a^{2}(k-1)\bx^{\top}\bx+1\bigr]k\bx^{\top}\bw e^{a^{2}(k-1)\bx^{\top}\bx}+N\bx^{\top}\bw+O\Big(\frac{1}{N}\Big), \label{eq:firstderi}\\
        \frac{\partial^{2} f(a,\bX,\bw)}{\partial a^{2}}&=2\bx^{\top}\bw(k-1)^{2}\bx^{\top}\bx a e^{a^{2}(k-1)\bx^{\top}\bx}\bigl[2a^{2}(k-1)\bx^{\top}\bx+3\bigr]+O\Big(\frac{1}{N}\Big). \label{eq:secderi}
    \end{align}
    We set $\bx=2/3 \cdot\bbI_{d}$, $k=1.1$, $\bw=\bx$, and define the function $g(a)= f(1,\bX+a\bX/3,\bx)$. Then Eqn.~\eqref{eq:firstderi} and \eqref{eq:secderi} show that $g(a)$ is a increasing convex function on $[-1,1]$.
    
     We can rearrange the weights in the first layer of $\bphi_{\relu}$ such that the input of the resultant network is a scalar $a\in[-1,1]$; the width of the resultant network is same as the width of $\rho_{\relu}(\sum_{i=1}^{N} \bphi_{\relu}(\bx_{i}))$; the resultant network represents the same function as 
     \begin{align}
         h(a)=\rho_{\relu}\bigg(\sum_{i=1}^{N}\bphi_{\relu}\biggl(\bx_{i}+\frac{a}{3}\bx_{i}\biggr)\bigg).\nonumber
     \end{align}
     Since $\rho_{\relu}(\sum_{i=1}^{N} \bphi_{\relu}(\bx_{i}))$ can approximate $\bbI_{N}^{\top}\Att(\bX,\bX,\bX)\bw$, the modified network can approximate $g(a)$ in terms of the sup-norm on $[-1,1]$. 
     
     Since $\relu$ is a 2-piece-wise linear function, $h(a)$ is also a piece-wise linear function, 
    whose number of pieces is denoted as $M$. Lemma 2.1 of \cite{telgarsky2015representation} shows that $M\leq 2(2W)^{2}=8W^{2}$, where $(2W)^{2}$ follows from two $\relu$ layers, and the additional factor of 2 follows from that $\bx_{1}$ and $\bx_{i}$ for $i\in\{2,\ldots,N\}$ take different values. 
    
    The pigeonhole principle implies that there is a piece-wise linear segment $[u,v]\subseteq[-1,1]$ whose length is at 
    least $2/M$. On this linear segment, the linear function $h(a)$  approximates $g(a)$ with error at most $\xi$. Eqn.~\eqref{eq:secderi}  then implies that
    \begin{align}
        \inf_{a\in[-1,1]} h^{\prime\prime}(a)\geq c_{1}>0, \nonumber
    \end{align}
    where $c_{1}=\Omega(d^{2}e^{cd})$ for some $c>0$. Denote the linear function on a linear piece $[u,v]$ and the approximation error as $\hath:\bbR\rightarrow\bbR$ and $e=h-\hath$, respectively. Since $\hath$ is a linear function, we have
    \begin{align}
        \max\big\{e(u),e(v)\big\} \geq e\Big(\frac{u+v}{2}\Big)+\frac{c_{1}}{2}\Big(\frac{v-u}{2}\Big)^{2}\label{ieq:78} 
        \end{align}
        and
        \begin{align}
        \xi \geq \frac{1}{2}\bigg(\max\bigl\{e(u),e(v)\bigr\}-e\biggl(\frac{u+v}{2}\biggr)\bigg)\label{ieq:79}.
    \end{align}
    Combining inequalities \eqref{ieq:78} and \eqref{ieq:79}, we have
    \begin{align}
        W\geq \biggl(\frac{c_{1}}{256\xi}\biggr)^{\frac{1}{4}}.\nonumber
    \end{align}
    Thus, we have $W=\Omega(\exp(cd) \xi^{-{1}/{4}})$ for some constant $c>0$, and this concludes the proof of Theorem \ref{thm:approx}.
\end{proof}

\section{Proof of Theorem~\ref{thm:main_mfree}}\label{app:thmmain_mfree}
\begin{proof}[Proof of Theorem~\ref{thm:main_mfree}]
    Recall the definition below Theorem~\ref{thm:main_concen}
    \begin{align}
        e(\calF_{\tf},\Pi,\delta,n)&=32V_{\max}^{2}\bigg[2+\gamma+2(m+1)L^{2}d^{2}\log\biggl(\frac{16mdLB_{V}B_{QK}B_{a}B_{b}n}{V_{\max}}\biggr)\nonumber\\
        &\qquad+2(m+1)Ld^{2}\log B_{w}+\log\biggl(\frac{2\calN(\Pi,1/n,d_{\infty})}{\delta}\biggr)\bigg].\nonumber
    \end{align}
    To simplify the proof, we define 
    \begin{align}
        f_{\pi^{*}}^{*}&=\arginf_{f\in\calF_{\tf}}\sup_{\mu\in d_{\Pi}} \bbE_{\mu}\Bigl[(f(\bar{\bS},\bar{\bA})-\calT^{\pi^{*}}f\bigl(\bar{\bS},\bar{\bA})\bigr)^{2}\Bigr],\nonumber\\
        \varepsilon&=\frac{3}{2}\varepsilon_{\calF}+\frac{2}{n}e(\calF_{\tf},\Pi,\delta,n).\nonumber
    \end{align}
    Our proof can be decomposed into three main procedures.
    \begin{itemize}
        \item Since $f_{\pi^{*}}^{*}$ is the best approximation of action-value function of the optimal policy $\pi^{*}$, we expect that it should belong to the confidence region of the action-value functions $\calF(\pi^{*},\varepsilon)$ with high probability.
        \item For any $\pi\in\Pi$ and any $f\in\calF(\pi,\varepsilon)$, since the empirical Bellman error is bounded $\calE(f,\pi;\calD)\leq \varepsilon$, we expect that the population Bellman error $\bbE_{\nu}[(f(\bar{\bS},\bar{\bA})-\calT^{\pi}f(\bar{\bS},\bar{\bA}))^{2}]$ can be controlled with high probability, which implies that $f$ is a reliable estimate of the action-value function of $\pi$. 
        \item The suboptimality gap of the learned policy according to the reliable action-value function estimate can be bounded using the estimation error bound.
    \end{itemize}
    
    We lay out the proof by the three steps as stated in the proof sketch.

    \textbf{Step 1: Show that $f_{\pi^{*}}^{*}\in\calF(\pi^{*},\varepsilon)$ with high probability.}
    
    From the definition of $f_{\pi^{*}}^{*}$ and Assumption~\ref{assump:mfree1}, we note that the population Bellman error of $f_{\pi^{*}}^{*}$ with respect to $\pi^{*}$ is bounded by $\varepsilon_{\calF}$. To bound the empirical Bellman error $\calE(f_{\pi^{*}}^{*},\pi^{*};\calD)$ of $f_{\pi^{*}}^{*}$, we need the generalization error bound of the action-value function with the transformer function class.
    \mainconcen*
    \begin{proof}
        See Appendix~\ref{app:propmain_concern} for a detailed proof.
    \end{proof}
    We can decompose the empirical Bellman error $\calE(f_{\pi^{*}}^{*},\pi^{*};\calD)$ as the sum of the population Bellman error and the generalization error, where the population Bellman error can be controlled with $\varepsilon_{\calF}$ according to Assumption~\ref{assump:mfree1}, and the generalization error can be controlled with Theorem~\ref{thm:main_concen}. Thus, we have the following lemma.
    \begin{lemma}\label{lem:errcontrol1}
        For any $\pi\in\Pi$, let $f_{\pi}^{*}=\arginf_{f\in\calF_{\tf}}\sup_{\mu\in d_{\Pi}} \bbE_{\mu}[(f(\bar{\bS},\bar{\bA})-\calT^{\pi}f(\bar{\bS},\bar{\bA}))^{2}]$. If Assumption~\ref{assump:mfree1} holds, the following inequality holds with probability at least $1-\delta$,
        \begin{align}
            \calE(f_{\pi}^{*},\pi;\calD)\leq\frac{3}{2}\varepsilon_{\calF}+\frac{2e(\calF_{\tf},\Pi,\delta,n)}{n}.\nonumber
        \end{align}
    \end{lemma}
    \begin{proof}
        See Appendix~\ref{app:lemerrcontrol1} for a detailed proof.
    \end{proof}

    \textbf{Step 2: For any policy $\pi\in\Pi$ and $f\in\calF(\pi,\varepsilon)$, show $\bbE_{\nu}[(f(\bar{\bS},\bar{\bA})-\calT^{\pi}f(\bar{\bS},\bar{\bA}))^{2}]\leq 2\varepsilon+3\varepsilon_{\calF,\calF}+4e(\calF_{\tf},\Pi,\delta,n)/n$ holds with high probability.}
    
    To prove the desired result, we relate the population Bellman error $\bbE_{\nu}[(f(\bar{\bS},\bar{\bA})-\calT^{\pi}f(\bar{\bS},\bar{\bA}))^{2}]$ with $\calE(f,\pi;\calD)$ through Theorem~\ref{thm:main_concen}, where we bound the population Bellman error as the difference between the empirical Bellman error and the generalization error. Thus, we have the following lemma.
    \begin{lemma}\label{lem:errcontrol2}
        For any $\pi\in\Pi$ and $f\in\calF_{\tf}$, if $\calE(f,\pi;\calD)\leq \varepsilon$ for some $\varepsilon>0$, and Assumption~\ref{assump:mfree1} holds, the following inequality holds with probability at least $1-\delta$,
        \begin{align}
            \bbE_{\nu}\Big[\big(f(\bar{\bS},\bar{\bA})-\calT^{\pi}f(\bar{\bS},\bar{\bA})\big)^{2}\Big]\leq 2\varepsilon+3\varepsilon_{\calF,\calF}+\frac{4e(\calF_{\tf},\Pi,\delta,n)}{n}.\nonumber
        \end{align}
    \end{lemma}
    \begin{proof}
        See Appendix~\ref{app:lemerrcontrol2} for a detailed proof.
    \end{proof}

    \textbf{Step 3: Bound the suboptimality gap of the learned policy with the population Bellman error bound in Step 2.}
    
    We define 
    \begin{align}
        \hat{f}_{\pi^{*}} =\argmax_{f\in\calF(\pi^{*},\varepsilon)}f(\bar{\bS}_{0},\pi^{*})  \quad \mbox{and}\quad 
        \breve{f}_{\pi^{*}} =\argmin_{f\in\calF(\pi^{*},\varepsilon)}f(\bar{\bS}_{0},\pi^{*}),\nonumber
    \end{align}
    where $\hat{f}_{\pi^{*}}$ and $\breve{f}_{\pi^{*}}$ are the maximal and minimal value functions in $\calF(\pi^{*},\varepsilon)$, respectively. Intuitively, since $f_{\pi^{*}}^{*}\in\calF(\pi^{*},\varepsilon)$ and that we learn the policy according to the pessimistic estimation of the action-value function in $\calF(\hpi,\varepsilon)$, we can upper bound the suboptimality gap by the difference between $\hat{f}_{\pi^{*}}$ and $\breve{f}_{\pi^{*}}$.
    
    Step 1 shows that with probability at least $1-\delta$, $f_{\pi^{*}}^{*}\in\calF(\pi^{*},\varepsilon)$. Then we have
    \begin{align}
        \max_{f\in\calF(\pi^{*},\varepsilon)}f(\bar{\bS}_{0},\pi)&\geq f_{\pi^{*}}^{*}f(\bar{\bS}_{0},\pi)=V_{P^{*}}^{\pi^{*}}(\bar{\bS}_{0})+\frac{\bbE_{d^{\pi^{*}}_{P^{*}}}\big[f_{\pi^{*}}^{*}-\calT^{\pi^{*}}f(\bar{\bS},\bar{\bA})\big]}{1-\gamma}\geq V_{P^{*}}^{\pi^{*}}(\bar{\bS}_{0})-\frac{\sqrt{\varepsilon_{\calF}}}{1-\gamma},\label{ieq:103}
    \end{align}
    where the equality follows from Lemma~\ref{lem:regretdecomp}, and the last inequality follows from Assumption~\ref{assump:mfree1}. Similarly, we can prove that 
    \begin{align}
        \min_{f\in\calF(\hpi,\varepsilon)}f(\bar{\bS}_{0},\hpi)\leq V_{P^{*}}^{\hpi}(\bar{\bS}_{0})+\frac{\sqrt{\varepsilon_{\calF}}}{1-\gamma}.\label{ieq:104}
    \end{align}
    Combining inequalities \eqref{ieq:103} and \eqref{ieq:104}, we have 
    \begin{align}
        &V_{P^{*}}^{\pi^{*}}(\bar{\bS}_{0})-V_{P^{*}}^{\hpi}(\bar{\bS}_{0})\nonumber\\
        &\quad\leq \max_{f\in\calF(\pi^{*},\varepsilon)}f(\bar{\bS}_{0},\pi)-\min_{f\in\calF(\hpi,\varepsilon)}f(\bar{\bS}_{0},\hpi)+\frac{2\sqrt{\varepsilon_{\calF}}}{1-\gamma}\nonumber\\ 
        &\quad\leq \hat{f}_{\pi^{*}}(\bar{\bS}_{0},\pi^{*})-\breve{f}_{\pi^{*}}(\bar{\bS}_{0},\pi^{*})+\frac{2\sqrt{\varepsilon_{\calF}}}{1-\gamma}\nonumber \\
        &\quad=\hat{f}_{\pi^{*}}(\bar{\bS}_{0},\pi^{*})-V_{P^{*}}^{\pi^{*}}(\bar{\bS}_{0})+V_{P^{*}}^{\pi^{*}}(\bar{\bS}_{0})-\breve{f}_{\pi^{*}}(\bar{\bS}_{0},\pi^{*})+\frac{2\sqrt{\varepsilon_{\calF}}}{1-\gamma},\label{ieq:127}
    \end{align}
    where the first inequality follows from inequalities \eqref{ieq:103} and \eqref{ieq:104}, the second inequality follows from Eqn.~\eqref{algo:mfree}. Applying the suboptimality gap decomposition in Lemma~\ref{lem:regretdecomp} to inequality \eqref{ieq:127}, we have 
    \begin{align}
        &V_{P^{*}}^{\pi^{*}}(\bar{\bS}_{0})-V_{P^{*}}^{\hpi}(\bar{\bS}_{0})\nonumber\\
        &\quad\leq\frac{1}{1-\gamma}\Big\{\bbE_{d^{\pi^{*}}_{P^{*}}}\big[\hat{f}_{\pi^{*}}(\bar{\bS},\bar{\bA})-\calT^{\pi^{*}}\hat{f}_{\pi^{*}}(\bar{\bS},\bar{\bA})\big]\nonumber\\
        &\quad\qquad-\bbE_{d^{\pi^{*}}_{P^{*}}}\big[\breve{f}_{\pi^{*}}(\bar{\bS},\bar{\bA})-\calT^{\pi^{*}}\breve{f}_{\pi^{*}}(\bar{\bS},\bar{\bA})\big]\Big\}+\frac{2\sqrt{\varepsilon_{\calF}}}{1-\gamma}\nonumber\\
        &\quad\leq \frac{1}{1-\gamma}\bigg\{\sqrt{C_{\calF_{\tf}}\bbE_{\nu}\Big[\big(\hat{f}_{\pi^{*}}(\bar{\bS},\bar{\bA})-\calT^{\pi^{*}}\hat{f}_{\pi^{*}}(\bar{\bS},\bar{\bA})\big)^{2}\Big]}\nonumber\\
        &\quad\qquad+\sqrt{C_{\calF_{\tf}}\bbE_{\nu}\Big[\big(\breve{f}_{\pi^{*}}(\bar{\bS},\bar{\bA})-\calT^{\pi^{*}}\breve{f}_{\pi^{*}}(\bar{\bS},\bar{\bA})\big)^{2}\Big]}\bigg\}+\frac{2\sqrt{\varepsilon_{\calF}}}{1-\gamma}\nonumber,
    \end{align}
    where the first inequality follows from Lemma~\ref{lem:regretdecomp}, and the second inequality follows from Jensen's inequality and the definition of $C_{\calF_{\tf}}$. Combined with the result in step 2, we have 
    \begin{align}
        &V_{P^{*}}^{\pi^{*}}(\bar{\bS}_{0})-V_{P^{*}}^{\hpi}(\bar{\bS}_{0})\nonumber\\
        &\quad\leq \frac{2\sqrt{C_{\calF_{\tf}}}}{1-\gamma}\sqrt{2\varepsilon+3\varepsilon_{\calF,\calF}+\frac{4e(\calF_{\tf},\Pi,\delta,n)}{n}}+\frac{2\sqrt{\varepsilon_{\calF}}}{1-\gamma}\nonumber\\ 
        &\quad\leq O\biggl(\frac{\sqrt{C_{\calF_{\tf}}(\varepsilon_{\calF}+\varepsilon_{\calF,\calF})}}{1-\gamma}+\frac{\sqrt{C_{\calF_{\tf}}}}{1-\gamma}\sqrt{\frac{e(\calF_{\tf},\Pi,\delta,n)}{n}}\biggr).\nonumber
    \end{align}
    Therefore, we conclude the proof of Theorem \ref{thm:main_mfree}.
\end{proof}

\section{Proof of Theorem~\ref{thm:main_mbased}}\label{app:thmmain_mbased}
For ease of notation, we denote the parameters of the neural network as 
\begin{align}
    \btheta= [\bW_{QK}^{1:L},\bW_{V}^{1:L},\ba^{1:L},\bb^{1:L}]\nonumber.
\end{align}
The parameter space is 
\begin{align}
    \Theta(B_{a},B_{b},B_{QK},B_{V})&=\Big\{\btheta \, \Big| \, \big|a_{kj}^{(i)}\big|< B_{a}, \big\|\bb_{kj}^{(i)}\big\|_{2}< B_{b}, \big\|\bW_{QK}^{(i)\top}\big\|_{\rmF}< B_{QK},\nonumber\\
    &\qquad \big\|\bW_{V}^{(i)\top}\big\|_{\rmF}< B_{V} \text{ for } i\in[L],j\in[m],k\in[d]\Big\}\nonumber.
\end{align}
Then we can denote the functions in  $\calM_{\tf}(B_{a},B_{b},B_{QK},B_{V})$   as $\bF_{\btheta}$ and the corresponding transition kernel in $\calP_{\tf}(B_{a},B_{b},B_{QK},B_{V})$ as $P_{\btheta}$, where $\btheta\in\Theta$ is the parameter of the function.

From the perspective of the parameter space $\Theta$, the algorithm in Eqn.~\eqref{algo:mpolicy} can be equivalently stated as
\begin{align}
    P_{\hat{\btheta}_{\rm MLE}}&=\argmax_{P\in\calP_{\tf}} \frac{1}{n}\sum_{i=1}^{n} \log P(\bar{\bS}_{i}^{\prime}\, |\,  \bar{\bS}_{i},\bar{\bA}_{i}),\nonumber\\
    \hpi&=\argmax_{\pi\in\Pi} \min_{P\in\calP(\zeta)} V_{P}^{\pi}(\bar{\bS}_{0}),\nonumber
\end{align}
where the confidence region of the dynamics is defined as
\begin{align}
    \calP(\zeta)=\bigg\{P\in\calP_{\tf}\, \bigg| \, \frac{1}{n}\sum_{i=1}^{n}\tv\bigl(P(\cdot\,|\, \bar{\bS}_{i},\bar{\bA}_{i}),P_{\hat{\btheta}_{\rm MLE}}(\cdot\,|\, \bar{\bS}_{i},\bar{\bA}_{i})\bigr)^{2}\leq \zeta\bigg\}.\nonumber
\end{align}

\begin{proof}[Proof of Theorem~\ref{thm:main_mbased}]
    For some constant $c_{1}>0$, we take
    $$
    \zeta=c_{1}\bigg(\frac{1}{n}(m+1)L^{2}d^{2}\log\Big(4NLmdB_{V}B_{QK}B_{a}B_{b}n\Big)+\frac{1}{n}\log\frac{1}{\delta}\bigg).
    $$
    Our proof can be decomposed into three main parts.
    \begin{itemize}
        \item Intuitively, the nominal transition kernel $P^{*}$ should belong to the confidence region of the system dynamics set $\calP(\zeta)$ with high probability.
        \item For any $P\in\calP(\zeta)$, we expect that the population  squared total variation between $P$ and $P^{*}$, i.e.,  $\bbE_{\nu}[\tv(P(\cdot \, |\, \bar{\bS},\bar{\bA}), P^{*}(\cdot\, |\,\bar{\bS},\bar{\bA}))^{2}]$, can be controlled with high probability, which implies that any $P\in\calP(\zeta)$ is a reliable estimate of the system dynamics.
        \item The suboptimality gap of the learned policy according to the reliable dynamic estimate can be bounded in terms of  the total variation.
    \end{itemize}

    We lay out the proof by the three steps as stated in the proof sketch.
    
    \textbf{Step 1: Show that $P^{*}\in\calP(\zeta)$ with probability at least $1-\delta$.}
    
    From the definition of $\calP(\zeta)$, we need to bound the empirical total variation  between the nominal transition kernel and the \ac{mle} estimate. Thus, we need an upper bound of the population total variation  between $P^{*}$ and $\hat{P}_{\rm MLE}$ and an accompanying generalization error bound. For the population error, we state the following proposition.
    \mlebound*
    \begin{proof}
        See Appendix~\ref{app:propmlebound} for a detailed proof.
    \end{proof}
    Similar to Theorem~\ref{thm:main_concen}, we can derive the generalization error bound in terms of the total variation distance.
    \begin{proposition}\label{prop:tvconcentrate}
        For any $\btheta\in\Theta$, with probability at least $1-\delta$, we have
        \begin{align}
            &\biggl|\bbE_{\calD}\Bigl[\tv\big(P^{*}(\cdot\,|\, \bar{\bS},\bar{\bA}),P_{\btheta}(\cdot\,|\, \bar{\bS},\bar{\bA})\big)^{2}\Bigr]-\frac{1}{n}\sum_{i=1}^{n}\tv\big(P^{*}(\cdot\,|\, \bar{\bS}_{i},\bar{\bA}_{i}),P_{\btheta}(\cdot\,|\, \bar{\bS}_{i},\bar{\bA}_{i})\big)^{2} \biggr|\nonumber\\
            &\quad\leq \frac{1}{2} \bbE_{\calD}\Bigl[\tv\big(P^{*}(\cdot\,|\, \bar{\bS},\bar{\bA}),P_{\btheta}(\cdot\,|\, \bar{\bS},\bar{\bA})\big)^{2}\Bigr]\nonumber\\
            &\quad\qquad+O\biggl(\frac{1}{n}mL^{2}d^{2}\log(NLmdB_{V}B_{QK}B_{a}B_{b}n)+\frac{1}{n}\log\frac{1}{\delta}\biggr).\nonumber
        \end{align}
    \end{proposition}
    \begin{proof}
        See Appendix~\ref{app:proptvconcentrate} for a detailed proof.
    \end{proof}
    With Propositions~\ref{prop:mlebound} and~\ref{prop:tvconcentrate}, we have
    \begin{align}
        &\frac{1}{n}\sum_{i=1}^{n}\tv\bigl(P^{*}(\cdot\,|\, \bar{\bS}_{i},\bar{\bA}_{i}),P_{\hat{\btheta}_{\rm MLE}}(\cdot\,|\, \bar{\bS}_{i},\bar{\bA}_{i})\bigr)^{2}\nonumber\\
        &\quad=\bigg\{\frac{1}{n}\sum_{i=1}^{n}\tv\Big(P^{*}(\cdot\,|\, \bar{\bS}_{i},\bar{\bA}_{i}),P_{\hat{\btheta}_{\rm MLE}}(\cdot\,|\, \bar{\bS}_{i},\bar{\bA}_{i})\Big)^{2}-\frac{3}{2}\bbE_{\nu}\Big[\tv\big(P^{*}(\cdot\,|\,\bar{\bS},\bar{\bA}), P_{\hat{\btheta}_{\rm MLE}}(\cdot\,|\,\bar{\bS},\bar{\bA})\big)^{2}\Big]\bigg\}\nonumber\\
        &\quad\qquad+\frac{3}{2}\bbE_{\nu}\Big[\tv\big(P^{*}(\cdot\,|\,\bar{\bS},\bar{\bA}), P_{\hat{\btheta}_{\rm MLE}}(\cdot\,|\,\bar{\bS},\bar{\bA})\big)^{2}\Big]\label{ieq:121}\\
        &\quad\leq O\Big(\frac{1}{n}(m+1)L^{2}d^{2}\log(4NLmdB_{V}B_{QK}B_{a}B_{b}n)+\frac{1}{n}\log\frac{1}{\delta}\Big),\label{ieq:68}
    \end{align}
    where the first term in Eqn.~\eqref{ieq:121} is bounded with Proposition~\ref{prop:tvconcentrate}, and the second term in Eqn.~\eqref{ieq:121} is bounded with Proposition~\ref{prop:mlebound}.
    
    \textbf{Step 2: Show that for any $P\in\calP(\zeta)$, the population total variation  between $P$ and $P^{*}$ is bounded.}
    
    For the population total variation  between $P$ and $P^{*}$, we have 
    \begin{align}
        &\bbE_{\nu}\Big[\tv\big(P(\cdot\,|\,\bar{\bS},\bar{\bA}), P^{*}(\cdot\,|\,\bar{\bS},\bar{\bA})\big)^{2}\Big]\nonumber\\
        &\quad=\bigg\{\bbE_{\nu}\Big[\tv\big(P(\cdot\,|\,\bar{\bS},\bar{\bA}), P^{*}(\cdot\, |\,\bar{\bS},\bar{\bA})\big)^{2}\Big]-\frac{2}{n}\sum_{i=1}^{n}\tv\Big(P(\cdot\,|\, \bar{\bS}_{i},\bar{\bA}_{i}),P^{*}(\cdot\,|\, \bar{\bS}_{i},\bar{\bA}_{i})\Big)^{2}\bigg\}\nonumber\\
        &\quad\qquad+\frac{2}{n}\sum_{i=1}^{n}\tv\big(P(\cdot\,|\, \bar{\bS}_{i},\bar{\bA}_{i}),P^{*}(\cdot\,|\, \bar{\bS}_{i},\bar{\bA}_{i})\big)^{2}\nonumber\\
        &\quad\leq  O\Big(\frac{1}{n}(m+1)L^{2}d^{2}\log\Big(4NLmdB_{V}B_{QK}B_{a}B_{b}n\Big)+\frac{1}{n}\log\frac{1}{\delta}\Big)\nonumber\\
        &\quad\qquad +\frac{4}{n}\sum_{i=1}^{n}\tv\Big(P(\cdot\,|\, \bar{\bS}_{i},\bar{\bA}_{i}),P_{\hat{\btheta}_{\rm MLE}}(\cdot\,|\, \bar{\bS}_{i},\bar{\bA}_{i})\Big)^{2}\nonumber\\
        &\quad\qquad+\frac{4}{n}\sum_{i=1}^{n}\tv\Big(P_{\hat{\btheta}_{\rm MLE}}(\cdot\,|\, \bar{\bS}_{i},\bar{\bA}_{i}),P^{*}(\cdot\,|\, \bar{\bS}_{i},\bar{\bA}_{i})\Big)^{2}\nonumber\\ 
        &\quad\leq O(\zeta),\label{ieq:70}
    \end{align}
    where the first inequality follows from Proposition~\ref{prop:tvconcentrate} and triangle inequality, and the last inequality follows from inequality \eqref{ieq:68} and the fact that $P\in\calP(\zeta)$.
    
    \textbf{Step 3: Bound the suboptimality gap of the learned policy with the total variation bound.}
    
    With the results in Step 1 and 2, we have that with probability at least $1-\delta$
    \begin{align}
        V_{P^{*}}^{\pi^{*}}(\bar{\bS}_{0})-V_{P^{*}}^{\hpi}(\bar{\bS}_{0})&= V_{P^{*}}^{\pi^{*}}(\bar{\bS}_{0})-\min_{P\in\calP(\zeta)}V_{P}^{\pi^{*}}(\bar{\bS}_{0})+\min_{P\in\calP(\zeta)}V_{P}^{\pi^{*}}(\bar{\bS}_{0})-V_{P^{*}}^{\hpi}(\bar{\bS}_{0})\nonumber\\
        &\leq V_{P^{*}}^{\pi^{*}}(\bar{\bS}_{0})-\min_{P\in\calP(\zeta)}V_{P}^{\pi^{*}}(\bar{\bS}_{0})+\min_{P\in\calP(\zeta)}V_{P}^{\hpi}(\bar{\bS}_{0})-V_{P^{*}}^{\hpi}(\bar{\bS}_{0})\nonumber\\ 
        &\leq V_{P^{*}}^{\pi^{*}}(\bar{\bS}_{0})-\min_{P\in\calP(\zeta)}V_{P}^{\pi^{*}}(\bar{\bS}_{0}),\nonumber 
    \end{align}
    where the first inequality follows from the fact that $\hpi$ maximizes $\min_{P\in\calP(\zeta)}V_{P}^{\pi}(\bar{\bS}_{0})$, and the last inequality follows from the fact that $P^{*}\in\calP(\zeta)$. Define $\breve{P}=\argmin_{P\in\calP(\zeta)}V_{P}^{\pi^{*}}(\bar{\bS}_{0})$. Then we have
    \begin{align*}
        V_{P^{*}}^{\pi^{*}}(\bar{\bS}_{0})-V_{P^{*}}^{\hpi}(\bar{\bS}_{0})&\leq V_{P^{*}}^{\pi^{*}}(\bar{\bS}_{0})-V_{\breve{P}}^{\pi^{*}}(\bar{\bS}_{0})\nonumber\\
        &\leq \frac{V_{\max}}{(1-\gamma)^{2}} \bbE_{(\bar{\bS},\bar{\bA})\sim d_{P^{*}}^{\pi^{*}}}\Big[\tv\big(\breve{P}(\cdot\, |\, \bar{\bS},\bar{\bA}), P^{*}(\cdot\, |\,\bar{\bS},\bar{\bA})\big)\Big],\nonumber
    \end{align*}
    where the second inequality follows from Lemma~\ref{lem:simu}. By the Jensen's inequality, it can be further bounded as
    \begin{align*}
        V_{P^{*}}^{\pi^{*}}(\bar{\bS}_{0})-V_{P^{*}}^{\hpi}(\bar{\bS}_{0})&\leq \frac{V_{\max}}{(1-\gamma)^{2}} \sqrt{C_{\calM_{\tf}}\bbE_{(\bar{\bS},\bar{\bA})\sim \nu}\Big[\tv\big(\breve{P}(\cdot\, |\,\bar{\bS},\bar{\bA}), P^{*}(\cdot\, |\,\bar{\bS},\bar{\bA})\big)^{2}\Big]}\nonumber\\ 
        &\leq O\bigg(\frac{V_{\max}}{(1-\gamma)^{2}} \sqrt{C_{\calM_{\tf}}\zeta}\bigg),\nonumber 
    \end{align*}
    where the first inequality follows Jensen's inequality, and the last inequality follows from inequality \eqref{ieq:70}. Therefore, we conclude the proof of Theorem \ref{thm:main_mbased}.
\end{proof}

\section{Proof of Theorem~\ref{thm:main_concen}}\label{app:propmain_concern}
\begin{proof}[Proof of Theorem~\ref{thm:main_concen}]
    We adopt a PAC-Bayesian framework to derive the generalization error bound of the Bellman error of the transformer functions, in which the generalization error is bounded by the Kullback--Leibler divergence between the distributions of functions. Recall that the KL divergence between $P$ and $Q$ is defined as $\kl(P\,\|\,Q)= \int_{\calA} \log(\mathrm{d}P/\mathrm{d}Q)\,\mathrm{d}P$ if $P\ll Q$, and $+\infty$ otherwise. We start with preliminary result. 
    \begin{restatable}{proposition}{pacbayes}\label{prop:pacbayes}
        Let $\calF$ be the collection of functions of $f:\bbR^{n}\rightarrow\bbR$. For any $f\in\calF$, we define
        \begin{align}
            \mu(f)= \bbE_{X}\big[f(X)\big], \quad \sigma^{2}(f)= \bbE_{X}\big[ (f(X)-\bbE_{X} [f(X)] )^{2}\big],\nonumber
        \end{align}
        where the expectation is taken with respect to a random variable $X\sim\nu$ on $(\bbR^{n},\calB(\bbR^{n}))$. 
        Assume that $|f(X)-\mu(f)|\leq b$ a.s. for some constant $b\in\bbR$ for all $f\in\calF$. Then for any $0<\lambda\leq 1/(2b)$, given a distribution $P_{0}$ on $\calF$, with probability at least $1-\delta$, we have
        \begin{align}
            \biggl|\bbE_{Q}\biggl[\bbE_{X}[f(X)]-\frac{1}{n}\sum_{i=1}^{n}f(X_{i})\biggr]\biggr|\leq \lambda \bbE_{Q}\big[\sigma^{2}(f)\big]+\frac{1}{n\lambda}\biggl[\kl(Q\,\|\,P_{0})+\log\frac{2}{\delta}\biggr],\nonumber
        \end{align}
        for any distribution $Q$ on $\calF$, where $X_{i}$ are i.i.d.\ samples of $\nu$. If the function class $\calF$ further satisfies $\sigma^{2}(f)\leq c \mu(f)$ for some constant $c\in\bbR$ for all $f\in\calF$, we have 
        \begin{align}
            \biggl|\bbE_{Q}\biggl[\bbE_{X}\bigl[f(X)\bigr]-\frac{1}{n}\sum_{i=1}^{n}f(X_{i})\biggr]\biggr|\leq \lambda c\bbE_{Q}\big[\mu(f)\big]+\frac{1}{n\lambda}\biggl[\kl(Q\,\|\,P_{0})+\log\frac{2}{\delta}\biggr], \label{ieq:128}
        \end{align}
        with probability at least $1-\delta$.
    \end{restatable}
    \begin{proof}
        See Appendix~\ref{app:proppacbayes} for a detailed proof.
    \end{proof}
    
    Our proof can be decomposed into four main parts.
    \begin{itemize}
        \item We verify that the Bellman error satisfies the conditions in Proposition~\ref{prop:pacbayes} and apply it to the Bellman error.
        \item Since the desired result is a point-wise generalization error bound, we need to control he fluctuation of   both sides of inequality \eqref{ieq:128} with respect to any pair of functions  $(f,\tilde{f})\in\calF_{\tf} \times \calF_{\tf}.$
        \item We specify two distributions $Q$ and $P_{0}$ and calculate $\kl(Q\,\|\,P_{0})$.
        \item We implement a standard covering argument to prove the result that holds for all the policies in $\Pi$.
    \end{itemize}
    
    \textbf{Step 1: Verify the conditions in Proposition~\ref{prop:pacbayes}}
    
    Let $\bX=(\bar{\bS},\bar{\bA},\bar{\bS}^{\prime})$ for all $f,\tilde{f}\in\calF_{\tf}(B_{a},B_{b},B_{QK},B_{V},B_{w})$. We define
    \begin{align}
        l(f,\tilde{f},\pi;\bX)=\big(f(\bar{\bS},\bar{\bA})-\barr(\bar{\bS},\bar{\bA})-\gamma \tilde{f}(\bar{\bS}^{\prime},\pi)\big)^{2}-\big(\calT^{\pi}\tilde{f}(\bar{\bS},\bar{\bA})-\barr(\bar{\bS},\bar{\bA})-\gamma \tilde{f}(\bar{\bS}^{\prime},\pi)\big)^{2}.\nonumber
    \end{align}
    Then the term we consider in Theorem~\ref{thm:main_concen} can be expressed as
    \begin{align}
        \calL(f,\tilde{f},\pi;\calD)-\calL(\calT^{\pi}\tilde{f},\tilde{f},\pi;\calD)=\frac{1}{n}\sum_{i=1}^{n}l(f,\tilde{f},\pi;\bX_{i})\text{ and }\big|l(f,\tilde{f},\pi;\bX)\big|\leq 4V_{\max}^{2}.\nonumber
    \end{align}
    Since $(\bar{\bS}_{i},\bar{\bA}_{i})$ is sampled from $\nu$, and $\bar{\bS}_{i}^{\prime}\sim\bar{P^{*}}(\cdot \,| \,\bar{\bS}_{i},\bar{\bA}_{i})$, we have $(\bar{\bS}_{i},\bar{\bA}_{i},\bar{\bS}_{i}^{\prime})\sim \nu\times \bar{P^{*}}$, i.e., $\bX_{i}\sim \nu\times \bar{P^{*}}$ for $i\in[N]$. Then the expectation of $l(f,\tilde{f},\pi;\bX)$ is
    \begin{align}
        &\bbE_{\nu\times \bar{P^{*}}}\big[l(f,\tilde{f},\pi;\bX)\big]\nonumber\\
        &\quad=\bbE_{\nu\times \bar{P^{*}}}\Big[\big(f(\bar{\bS},\bar{\bA})-\calT^{\pi}\tilde{f}(\bar{\bS},\bar{\bA})\big)\big(f(\bar{\bS},\bar{\bA})+\calT^{\pi}\tilde{f}(\bar{\bS},\bar{\bA})-2\barr-2\gamma \tilde{f}(\bar{\bS}^{\prime},\pi)\big)\Big]\nonumber\\
        &\quad=\bbE_{\nu}\bigg[\bbE_{P^{*}}\Big[\big(f(\bar{\bS},\bar{\bA})-\calT^{\pi}\tilde{f}(\bar{\bS},\bar{\bA})\big)\big(f(\bar{\bS},\bar{\bA})+\calT^{\pi}\tilde{f}(\bar{\bS},\bar{\bA})-2\barr-2\gamma \tilde{f}(\bar{\bS}^{\prime},\pi)\big)\,\Big|\,\bar{\bS},\bar{\bA}\Big]\bigg]\nonumber\\
        &\quad=\bbE_{\nu}\Big[\big(f(\bar{\bS},\bar{\bA})-\calT^{\pi}\tilde{f}(\bar{\bS},\bar{\bA})\big)^{2}\Big],\nonumber
    \end{align}
    where the last equality follows from the definition of the Bellman operator. As a consequence, the variance of $l(f,\tilde{f},\pi;\bX)$ can be bounded by its expectation as
    \begin{align}
        &{\rm Var}\big(l(f,\tilde{f},\pi;\bX)\big)\nonumber\\*
        &\quad\leq \bbE_{\nu\times \bar{P^{*}}}\Big[\big(l(f,\tilde{f},\pi;\bX)\big)^{2}\Big]\nonumber\\
        &\quad=\bbE_{\nu}\bigg[\bbE_{P^{*}}\Big[\big(f(\bar{\bS},\bar{\bA})-\calT^{\pi}\tilde{f}(\bar{\bS},\bar{\bA})\big)^{2}\big(f(\bar{\bS},\bar{\bA})+\calT^{\pi}\tilde{f}(\bar{\bS},\bar{\bA})-2\barr-2\gamma \tilde{f}(\bar{\bS}^{\prime},\pi)\big)^{2}\,\Big|\,\bar{\bS},\bar{\bA}\Big]\bigg]\nonumber\\
        &\quad\leq 16V_{\max}^{2}\bbE_{\nu}\bigg[\bbE_{P^{*}}\Big[\big(f(\bar{\bS},\bar{\bA})-\calT^{\pi}\tilde{f}(\bar{\bS},\bar{\bA})\big)^{2}\,\Big|\,\bar{\bS},\bar{\bA}\Big]\bigg]\nonumber\\
        &\quad=16V_{\max}^{2}\bbE_{\nu\times \bar{P^{*}}}\big[l(f,\tilde{f},\pi;\bX)\big],\label{ieq:90}
    \end{align}
    where the last inequality follows from the fact that $f$ and $\tilde{f}$ is bounded by $V_{\max}$. Inequality \eqref{ieq:90} shows that $l(f,\tilde{f},\pi;\bX)$ satisfies the condition in Proposition~\ref{prop:pacbayes} with $b=4V_{\max}^{2}$ and $c=16V_{\max}^{2}$. In the following, we apply Proposition~\ref{prop:pacbayes} to $l(f,\tilde{f},\pi;\bX)$.
    
    For ease of notation, we denote the parameters of the neural network as \begin{align}
        \btheta= [\bW_{QK}^{1:L},\bW_{V}^{1:L},\ba^{1:L},\bb^{1:L},\bw].\nonumber
    \end{align}
    The parameter space is 
    \begin{align}
        \Theta(B_{a},B_{b},B_{QK},B_{V},B_{w})&=\Big\{\btheta \, \Big| \, \big|a_{kj}^{(i)}\big|< B_{a}, \big\|\bb_{kj}^{(i)}\big\|_{q}< B_{b}, \big\|\bW_{QK}^{(i)\top}\big\|_{p,q}< B_{QK},\nonumber\\
        &\qquad\big\|\bW_{V}^{(i)\top}\big\|_{p,q}< B_{V},\|\bw\|_{q}< B_{w} \text{ for } i\in[L],j\in[m],k\in[d]\Big\}. \nonumber
    \end{align}
     We denote the functions in  $\calF_{\tf}(B_{a},B_{b},B_{QK},B_{V},B_{w})$ equivalently as $f_{\btheta}$, where $\btheta\in\Theta$ is the parameter of the function.
    
    For a finite policy class $\tilde{\Pi}$ (which is set to be a cover of the original policy class $\Pi$ in Step 4), Proposition~\ref{prop:pacbayes} shows that: Given a distribution $P_{0}$ of $(\btheta,\btheta^{\prime})$ on $\Theta\times\Theta$, for all distribution $Q$ on $\Theta\times\Theta$ and any policy $\pi\in\tilde{\Pi}$, with probability at least $1-\delta$, we have
    \begin{align}
        &\biggl|\bbE_{Q}\biggl[\bbE_{\nu\times\bar{P^{*}}}\bigl[l(f_{\btheta},f_{\btheta^{\prime}},\pi;\bX)\bigr]-\frac{1}{n}\sum_{i=1}^{n}l(f_{\btheta},f_{\btheta^{\prime}},\pi;\bX_{i})\biggr]\biggr|\nonumber\\
        &\quad\leq 16V_{\max}^{2}\lambda \cdot\bbE_{Q,\nu\times\bar{P^{*}}}\big[l(f_{\btheta},f_{\btheta^{\prime}},\pi;\bX)\big]+\frac{1}{n\lambda}\bigg[\kl(Q\,\|\,P_{0})+\log\frac{2|\tilde{\Pi}|}{\delta}\bigg],\label{ieq:37}
    \end{align}
    where $\lambda\leq 1/(8V_{\max}^{2})$.
    
    \textbf{Step 2: Control the fluctuation of both sides of inequality \eqref{ieq:37} introduced by $Q$.}
    
    To derive a generalization error bound for any function pair $(\btheta,\btheta^{\prime})$ in $\calF_{\tf}\times\calF_{\tf}$, we set $Q$ as the uniform distribution on a neighborhood area of $(\btheta,\btheta^{\prime})$ , $P_{0}$ as the uniform distribution $\Theta\times\Theta$, and control the fluctuation of the left-hand side of inequality \eqref{ieq:37} due to the averaging according to $Q$.
    
    We define the difference between the functions of different parameter pairs $(\tilde{\btheta},\tilde{\btheta}^{\prime})$ and $(\btheta,\btheta^{\prime})$ as
    \begin{align}
        e(\tilde{\btheta},\tilde{\btheta}^{\prime},\btheta,\btheta^{\prime},\bX)&= l(f_{\tilde{\btheta}},f_{\tilde{\btheta}^{\prime}},\pi;\bX)-l(f_{\btheta},f_{\btheta^{\prime}},\pi;\bX).\nonumber
    \end{align}
    To control the fluctuation of the left-hand side of inequality \eqref{ieq:37} due to the average according to $Q$, we need to upper bound $e(\tilde{\btheta},\tilde{\btheta}^{\prime},\btheta,\btheta^{\prime},\bX)$ for all $\bX\in\bbR^{N\times d}$, which can be achieved by the following result.
    \begin{proposition}\label{prop:nnerror}
        For any input $\bX\in\bbR^{N\times d}$, any functions  $g_{\tf}(\bX;\bW_{QK}^{1:L},\bW_{V}^{1:L},\ba^{1:L},\bb^{1:L},\bw)$ and $g_{\tf}(\bX;\tilde{\bW}_{QK}^{1:L},\tilde{\bW}_{V}^{1:L},\tilde{\ba}^{1:L},\tilde{\bb}^{1:L},\tilde{\bw})\in \calF_{\tf}(B_{a},B_{b},B_{QK},B_{V},B_{w})$, and two positive conjugate numbers $p,q\in\bbR$, we have
        \begin{align}
            &\big|g_{\tf}(\bX;\bW_{QK}^{1:L},\bW_{V}^{1:L},\ba^{1:L},\bb^{1:L},\bw)-g_{\tf}(\bX;\tilde{\bW}_{QK}^{1:L},\tilde{\bW}_{V}^{1:L},\tilde{\ba}^{1:L},\tilde{\bb}^{1:L},\tilde{\bw})\big|\nonumber\\
            &\quad\leq \sum_{i=1}^{L}\alpha_{i}\cdot(\beta_{i}+\iota_{i}+\kappa_{i}+\rho_{i})+\|\bw-\tilde{\bw}\|_{q},\nonumber
        \end{align}
        where
        \begin{align}
            \alpha_{i}&=B_{w}\big[B_{V}(1+4c_{p,q}B_{QK})+d^{\frac{1}{p}}mB_{a}B_{b}\big]^{L-i},\nonumber\\
            \beta_{i}&=2c_{p,q}B_{V}\|\bW_{QK}^{(i)\top}-\tilde{\bW}_{QK}^{(i)\top}\|_{p,q},\nonumber\\
            \iota_{i}&=\|\bW_{V}^{(i)\top}-\tilde{\bW}_{V}^{(i)\top}\|_{p,q},\nonumber\\
            \kappa_{i}&=B_{b}\bigg[\sum_{k=1}^{d}\bigg(\sum_{j=1}^{m}\big|a_{kj}^{(i)}-\tilde{a}_{kj}^{(i)}\big|\bigg)^{p}\bigg]^{\frac{1}{p}},\nonumber\\
            \rho_{i}&=B_{a}\bigg[\sum_{k=1}^{d}\bigg(\sum_{j=1}^{m}\big\|\bb_{kj}^{(i)}-\tilde{\bb}_{kj}^{(i)}\big\|_{q}\bigg)^{p}\bigg]^{\frac{1}{p}},\nonumber
        \end{align}
        for $i\in[L]$.
    \end{proposition}
    \begin{proof}
        See Appendix~\ref{app:propnnerror} for a detailed proof.
    \end{proof}
    Motivated by Proposition~\ref{prop:nnerror}, we define the upper bound of the difference of functions in $\calF_{\tf}$ with different parameters $\btheta$ and $\tilde{\btheta}$ as
    \begin{align}
        \Delta(\btheta,\tilde{\btheta})&=\sum_{i=1}^{L}B_{w}\big[B_{V}(1+4c_{p,q}B_{QK})+d^{\frac{1}{p}}mB_{a}B_{b}\big]^{L-i}\bigg\{2c_{p,q}B_{V}\|\bW_{QK}^{(i)\top}-\tilde{\bW}_{QK}^{(i)\top}\|_{p,q}\nonumber\\
        &\qquad+\|\bW_{V}^{(i)\top}-\tilde{\bW}_{V}^{(i)\top}\|_{p,q}+B_{b}\bigg[\sum_{k=1}^{d}\bigg(\sum_{j=1}^{m}\big|a_{kj}^{(i)}-\tilde{a}_{kj}^{(i)}\big|\bigg)^{p}\bigg]^{\frac{1}{p}}\nonumber\\
        &\qquad+B_{a}\bigg[\sum_{k=1}^{d}\bigg(\sum_{j=1}^{m}\big\|\bb_{kj}^{(i)}-\tilde{\bb}_{kj}^{(i)}\big\|_{q}\bigg)^{p}\bigg]^{\frac{1}{p}}\bigg\}+\|\bw-\tilde{\bw}\|_{q}.\nonumber
    \end{align}
    Then we can upper bound the absolute value of $e(\tilde{\btheta},\tilde{\btheta}^{\prime},\btheta,\btheta^{\prime},\bX)$ as
    \begin{align}
        &\big|e(\tilde{\btheta},\tilde{\btheta}^{\prime},\btheta,\btheta^{\prime},\bX)\big|\nonumber\\
        &\quad\leq\Big|\big(f_{\tilde{\btheta}}(\bar{\bS},\bar{\bA})-\barr(\bar{\bS},\bar{\bA})-\gamma f_{\tilde{\btheta}^{\prime}}(\bar{\bS}^{\prime},\pi)\big)^{2}-\big(f_{\btheta}(\bar{\bS},\bar{\bA})-\barr(\bar{\bS},\bar{\bA})-\gamma f_{\btheta^{\prime}}(\bar{\bS}^{\prime},\pi)\big)^{2}\Big|\nonumber\\ 
        &\quad\qquad+\Big|\big(\calT^{\pi}f_{\tilde{\btheta}^{\prime}}(\bar{\bS},\bar{\bA})-\barr(\bar{\bS},\bar{\bA})-\gamma f_{\tilde{\btheta}^{\prime}}(\bar{\bS}^{\prime},\pi)\big)^{2}-\big(\calT^{\pi}f_{\btheta^{\prime}}(\bar{\bS},\bar{\bA})-\barr(\bar{\bS},\bar{\bA})-\gamma f_{\btheta^{\prime}}(\bar{\bS}^{\prime},\pi)\big)^{2}\Big|\nonumber\\
        &\quad\leq 4V_{\max}\big(\Delta(\tilde{\btheta},\btheta)+3\gamma\Delta(\tilde{\btheta}^{\prime},\btheta^{\prime})\big),\label{ieq:33}
    \end{align}
    where the first inequality follows from the triangle inequality, and the second inequality follows from that $f_{\btheta}\in[-V_{\max},V_{\max}]$ and $r\in[-R_{\max},R_{\max}]$.
    For any fixed pair of parameters $(\btheta,\btheta^{\prime})$, using inequality \eqref{ieq:33}, we can upper bound the generalization error for a fixed parameter pair $(\btheta,\btheta^{\prime})$ by the left-hand side of  inequality \eqref{ieq:37} as
    \begin{align}
        &\biggl|\bbE_{\nu\times\bar{P^{*}}}\bigl[l(f_{\btheta},f_{\btheta^{\prime}},\pi;\bX)\bigr]-\frac{1}{n}\sum_{i=1}^{n}l(f_{\btheta},f_{\btheta^{\prime}},\pi;\bX_{i})\biggr|\nonumber\\
        &\quad\leq \biggl|\bbE_{(\tilde{\btheta},\tilde{\btheta}^{\prime})\sim Q}\biggl[\bbE_{\nu\times\bar{P^{*}}}\bigl[l(f_{\tilde{\btheta}},f_{\tilde{\btheta}^{\prime}},\pi;\bX)\bigr]-\frac{1}{n}\sum_{i=1}^{n}l(f_{\tilde{\btheta}},f_{\tilde{\btheta}^{\prime}},\pi;\bX_{i})\biggr]\biggr|\nonumber\\
        &\quad\qquad +\bigg|\bbE_{(\tilde{\btheta},\tilde{\btheta}^{\prime})\sim Q}\biggl[\bbE_{\nu\times\bar{P^{*}}}\bigl[e(\tilde{\btheta},\tilde{\btheta}^{\prime},\btheta,\btheta^{\prime},\bX)\bigr]-\frac{1}{n}\sum_{i=1}^{n}e(\tilde{\btheta},\tilde{\btheta}^{\prime},\btheta,\btheta^{\prime},\bX_{i})\biggr]\bigg|\nonumber\\
        &\quad\leq\biggl|\bbE_{(\tilde{\btheta},\tilde{\btheta}^{\prime})\sim Q}\biggl[\bbE_{\nu\times\bar{P^{*}}}\bigl[l(f_{\tilde{\btheta}},f_{\tilde{\btheta}^{\prime}},\pi;\bX)\bigr]-\frac{1}{n}\sum_{i=1}^{n}l(f_{\tilde{\btheta}},f_{\tilde{\btheta}^{\prime}},\pi;\bX_{i})\biggr]\biggr|\nonumber\\
        &\quad\qquad+ 8V_{\max}\bbE_{(\tilde{\btheta},\tilde{\btheta}^{\prime})\sim Q}\big[\Delta(\tilde{\btheta},\btheta)+3\gamma\Delta(\tilde{\btheta}^{\prime},\btheta^{\prime})\big]\label{ieq:35}.
    \end{align}
    Similarly, for a fixed parameter pair of parameters$(\btheta,\btheta^{\prime})$, the first term in the right-hand side of inequality \eqref{ieq:37} can be upper bounded as
    \begin{align}
        &\bbE_{(\tilde{\btheta},\tilde{\btheta}^{\prime})\sim Q}\Big[\bbE_{\nu\times\bar{P^{*}}}\big[l(f_{\tilde{\btheta}},f_{\tilde{\btheta}^{\prime}},\pi;\bX)\big]\Big]\nonumber\\
        &\quad\leq \bbE_{\nu\times\bar{P^{*}}}[l(f_{\btheta},f_{\btheta^{\prime}},\pi;\bX)]+4V_{\max}\bbE_{(\tilde{\btheta},\tilde{\btheta}^{\prime})\sim Q}\big[\Delta(\tilde{\btheta},\btheta)+3\gamma\Delta(\tilde{\btheta}^{\prime},\btheta^{\prime})\big]\label{ieq:36}.
    \end{align}
    Substituting Eqn.~\eqref{ieq:35} and \eqref{ieq:36} into Eqn.~\eqref{ieq:37}, we derive that : Given a distribution $P_{0}$ of $(\btheta,\btheta^{\prime})$ on $\Theta\times\Theta$, for all distribution $Q$ on $\Theta\times\Theta$, any policy $\pi\in\tilde{\Pi}$ and any $(\btheta,\btheta^{\prime})\in\Theta\times\Theta$, with probability at least $1-\delta$, we have, 
    \begin{align}
        &\biggl|\bbE_{\nu\times\bar{P^{*}}}\bigl[l(f_{\btheta},f_{\btheta^{\prime}},\pi;\bX)\bigr]-\frac{1}{n}\sum_{i=1}^{n}l(f_{\btheta},f_{\btheta^{\prime}},\pi;\bX_{i})\biggr|\nonumber\\
        &\quad\leq V_{\max}(64V_{\max}^{2}\lambda+8)\bbE_{(\tilde{\btheta},\tilde{\btheta}^{\prime})\sim Q}\big[\Delta(\tilde{\btheta},\btheta)+3\gamma\Delta(\tilde{\btheta}^{\prime},\btheta^{\prime})\big]+16V_{\max}^{2}\lambda\bbE_{\nu\times\bar{P^{*}}}\bigl[l(f_{\btheta},f_{\btheta^{\prime}},\pi;\bX)\bigr]\nonumber\\
        &\quad\qquad+\frac{1}{n\lambda}\biggl[\kl(Q\,\|\,P_{0})+\log\frac{2|\tilde{\Pi}|}{\delta}\biggr],\nonumber
    \end{align}
    where $\lambda\leq 1/(8V_{\max}^{2})$. We take $\lambda=1/(32V_{\max}^{2})$, then
    \begin{align}
        &\biggl|\bbE_{\nu\times\bar{P^{*}}}\bigl[l(f_{\btheta},f_{\btheta^{\prime}},\pi;\bX)\bigr]-\frac{1}{n}\sum_{i=1}^{n}l(f_{\btheta},f_{\btheta^{\prime}},\pi;\bX_{i})\biggr|\nonumber\\
        &\quad\leq 10V_{\max}\bbE_{(\tilde{\btheta},\tilde{\btheta}^{\prime})\sim Q}\big[\Delta(\tilde{\btheta},\btheta)+3\gamma\Delta(\tilde{\btheta}^{\prime},\btheta^{\prime})\big]+\frac{1}{2}\bbE_{\nu\times\bar{P^{*}}}\bigl[l(f_{\btheta},f_{\btheta^{\prime}},\pi;\bX)\bigr]\nonumber\\
        &\quad\qquad+\frac{32V_{\max}^{2}}{n}\bigg[\kl(Q\,\|\,P_{0})+\log\frac{2|\tilde{\Pi}|}{\delta}\bigg].\label{ieq:38}
    \end{align}
    
    \textbf{Step 3: Specify the distributions $P_{0}$ and $Q$ on the function class $\calF_{\tf}$.}
    
    For a fixed parameters pair $(\btheta,\btheta^{\prime})$, we set $P_{0}$ as the product of the uniform distribution of each parameter on the whole space and $Q$ as the product of the uniform distribution of each parameter on the neighborhood around $(\btheta,\btheta^{\prime})$, i.e.,
    \begin{align}
        P_{0}=&\bigg\{\U\big(\bbB(0,B_{w},\|\cdot\|_{q})\big)\cdot\prod_{i=1}^{L}\bigg[\U\big(\bbB(0,B_{QK},\|\cdot\|_{p,q})\big)\cdot\U\big(\bbB(0,B_{V},\|\cdot\|_{p,q})\big)\nonumber\\
        &\qquad\cdot\Big(\U\big(\bbB(0,B_{a},|\cdot|)\big)\cdot\U\big(\bbB(0,B_{b},\|\cdot\|_{q})\big)\Big)^{md}\bigg]\bigg\}^{2},\quad \text{ and}\nonumber\\
        Q=&\bigg\{\U\big(\bbB(\bw,\varepsilon_{w},\|\cdot\|_{q})\big)\cdot\prod_{i=1}^{L}\bigg[\U\big(\bbB(\bW_{QK}^{(i)\top},\varepsilon_{QK}^{(i)},\|\cdot\|_{p,q})\big)\cdot\U\big(\bbB(\bW_{V}^{(i)\top},\varepsilon_{V}^{(i)},\|\cdot\|_{p,q})\big)\nonumber\\
        &\qquad\cdot\prod_{j\in[m],k\in[d]}\Big(\U\big(\bbB(a_{kj}^{(i)},\varepsilon_{a,kj}^{(i)},|\cdot|)\big)\cdot\U\big(\bbB(\bb_{kj}^{(i)},\varepsilon_{b,kj}^{(i)},\|\cdot\|_{q})\big)\Big)\bigg]\bigg\}\nonumber\\
        &\qquad\cdot\bigg\{\U\big(\bbB(\bw^{\prime},\varepsilon_{w},\|\cdot\|_{q})\big)\cdot\prod_{i=1}^{L}\bigg[\U\big(\bbB(\bW_{QK}^{\prime(i)\top},\varepsilon_{QK}^{(i)},\|\cdot\|_{p,q})\big)\cdot\U\big(\bbB(\bW_{V}^{\prime(i)\top},\varepsilon_{V}^{(i)},\|\cdot\|_{p,q})\big)\nonumber\\
        &\qquad\cdot\prod_{j\in[m],k\in[d]}\Big(\U\big(\bbB(a_{kj}^{(i)\prime},\varepsilon_{a,kj}^{(i)},|\cdot|)\big)\cdot\U\big(\bbB(\bb_{kj}^{(i)\prime},\varepsilon_{b,kj}^{(i)},\|\cdot\|_{q})\big)\Big)\bigg]\bigg\}\nonumber
    \end{align}
    where $\bbB(\bx,r,\|\cdot\|)$ denotes the ball $\{\by\, |\, \|\by-\bx\|<r\}$ in some metric space $(\calX,\|\cdot\|)$, and $\U(\cdot)$ denotes the uniform distribution on some set.
    For a constant $C>0$, we define $\varepsilon_{w} = \Delta= C/[(1+3\gamma)(4L+1)n]$. For $i\in[L]$, $j\in[m]$ and $k\in[d]$, we set 
    \begin{align}
        \varepsilon_{QK}^{(i)}&=\big(2c_{p,q}B_{V}B_{w}\big)^{-1}\big[B_{V}(1+4c_{p,q}B_{QK})+d^{\frac{1}{p}}mB_{a}B_{b}\big]^{-L+i}\Delta,\nonumber\\
        \varepsilon_{V}^{(i)}&=B_{w}^{-1}\big[B_{V}(1+4c_{p,q}B_{QK})+d^{\frac{1}{p}}mB_{a}B_{b}\big]^{-L+i}\Delta,\nonumber\\
        \varepsilon_{a,kj}^{(i)}&=d^{-\frac{1}{p}}\big(mB_{b}B_{w}\big)^{-1}\big[B_{V}(1+4c_{p,q}B_{QK})+d^{\frac{1}{p}}mB_{a}B_{b}\big]^{-L+i}\Delta ,\nonumber\\
        \varepsilon_{b,kj}^{(i)}&=d^{-\frac{1}{p}}\big(mB_{a}B_{w}\big)^{-1}\big[B_{V}(1+4c_{p,q}B_{QK})+d^{\frac{1}{p}}mB_{a}B_{b}\big]^{-L+i}\Delta.\nonumber
    \end{align}
    By Proposition~\ref{prop:nnerror}, we then have
    \begin{align}
        \bbE_{(\tilde{\btheta},\tilde{\btheta}^{\prime})\sim Q}\big[\Delta(\tilde{\btheta},\btheta)+3\gamma\Delta(\tilde{\btheta}^{\prime},\btheta^{\prime})\big]\leq \frac{C}{n}.\label{ieq:39}
    \end{align}
    Since the distributions $P_{0}$ and $Q$ are the products of the distributions of each parameters, $\kl(Q\,\|\,P_{0})$ is the sum of the KL-divergences between the distributions of each parameters. For $i\in[L]$, the KL divergence between the distributions of $\bW_{QK}$ can be upper bounded as
    \begin{align}
        &\kl\Big(\U\big(\bbB(\bW_{QK}^{(i)\top},\varepsilon_{QK}^{(i)},\|\cdot\|_{p,q})\big)\Big\|\U\big(\bbB(0,B_{QK},\|\cdot\|_{p,q})\big)\Big)\nonumber\\
        &\quad=d^{2}\log\biggl(\frac{B_{QK}}{\varepsilon_{QK}^{(i)}}\biggr)\nonumber\\
        &\quad\leq 2(L-i)d^{2}\log\biggl(\frac{4mdB_{V}B_{QK}B_{a}B_{b}}{\Delta}\biggr)+d^{2}\log B_{w},\nonumber
    \end{align}
    where the equality follows from the fact that $\bW_{QK}^{(i)}\in\bbR^{d\times d}$ for all $i\in[L]$, in which the logarithm of the ratio between two $\ell_{p,q}$-norm balls is equal to $d^{2}$ times the logarithm of the ratio between the radiuses.
    
    We note that the product $B_{V}B_{QK}B_{a}B_{b}B_{w}$ is defined as $\bar{B}$ in Theorem~\ref{thm:main_concen}, which is adopted to simplify the result.
    Similar bounds for the KL divergence of the distributions of parameters $\bW_{V}^{(i)}$, $a_{kj}^{(i)}$, $\bb_{kj}^{(i)}$ and $\bw$ for $i\in[L]$, $k\in[d]$ and $j\in[m]$ can be derived by replacing $d^{2}$ by the dimension of the parameter. Thus, we have
    \begin{align}
        \kl(Q\,\|\,P_{0})\leq 2(m+1)L^{2}d^{2}\log\biggl(\frac{4mdB_{V}B_{QK}B_{a}B_{b}}{\Delta}\biggr)+2(m+1)Ld^{2}\log B_{w}.\label{ieq:40}
    \end{align}
    Substituting inequalities \eqref{ieq:39} and \eqref{ieq:40} into inequality \eqref{ieq:38}, we derive that for any $(\btheta,\btheta^{\prime})\in\Theta^2$, with probability at least $1-\delta$
    \begin{align}
        &\biggl|\bbE_{\nu\times\bar{P^{*}}}\bigl[l(f_{\btheta},f_{\btheta^{\prime}},\pi;\bX)\bigr]-\frac{1}{n}\sum_{i=1}^{n}l(f_{\btheta},f_{\btheta^{\prime}},\pi;\bX_{i})\biggr|\nonumber\\
        &\quad\leq 10V_{\max}\frac{C}{n}+\frac{1}{2}\bbE_{\nu\times\bar{P^{*}}}\bigl[l(f_{\btheta},f_{\btheta^{\prime}},\pi;\bX)\bigr]+\frac{32V_{\max}^{2}}{n}\bigg[2(m+1)L^{2}d^{2}\log\biggl(\frac{4mdB_{V}B_{QK}B_{a}B_{b}}{\Delta}\biggr)\nonumber\\
        &\quad\qquad+2(m+1)Ld^{2}\log B_{w}+\log\frac{2|\tilde{\Pi}|}{\delta}\bigg].\label{ieq:41}
    \end{align}
    
    \textbf{Step 4: Cover the policy class $\Pi$.}
    
    Note that inequality \eqref{ieq:41} only applies to the situation where the policy class is finite. When the policy class is infinite, we consider the covering of the policy class with respect to $d_{\infty}(\cdot,\cdot)$. The $\varepsilon$-covering number of the policy class with respect to $d_{\infty}(\cdot,\cdot)$ is denoted as $\calN(\Pi,\varepsilon,d_{\infty})$, and the corresponding $\varepsilon$-cover is $\calC(\Pi,\varepsilon,d_{\infty})$, which is defined in Section~\ref{sec:prelim}. From the definition of $d_{\infty}(\cdot,\cdot)$, we have 
    \begin{align}
        d_{\infty}(\pi,\pi^{\prime})&= \sup_{\bar{\bS}\in \bar{\calS}} \sum_{\bar{\bA}\in\bar{\calA}} \big|\pi(\bar{\bA}\,|\,\bar{\bS})-\pi^{\prime}(\bar{\bA}\,|\,\bar{\bS})\big|\nonumber\\
        \bigl|f(\bar{\bS},\pi)-f(\bar{\bS},\pi^{\prime})\bigr|&=\biggl|\sum_{\bar{\bA}\in\bar{\calA}}\bigl[\pi(\bar{\bA}\,|\,\bar{\bS})-\pi^{\prime}(\bar{\bA}\,|\,\bar{\bS})\bigr]f(\bar{\bS},\bar{\bA})\biggr|\nonumber\\
        &\leq \sum_{\bar{\bA}\in\bar{\calA}}\bigl|\pi(\bar{\bA}\,|\,\bar{\bS})-\pi^{\prime}(\bar{\bA}\,|\,\bar{\bS})\bigr|\cdot\bigl|f(\bar{\bS},\bar{\bA})\bigr|\nonumber\\
        &\leq V_{\max}d_{\infty}(\pi,\pi^{\prime})\label{ieq:91}\\
        |\calT^{\pi}f(\bar{\bS},\bar{\bA})-\calT^{\pi^{\prime}}f(\bar{\bS},\bar{\bA})|&=\gamma\Bigl|\bbE_{(\bar{\bS}^{\prime})\sim \hat{\bar{P^{*}}}(\cdot\,|\,\bar{\bS},\bar{\bA})}\bigl[f(\bar{\bS}^{\prime},\pi)-f(\bar{\bS}^{\prime},\pi^{\prime})\,\big|\,\bar{\bS},\bar{\bA}\bigr]\Bigr|\nonumber\\
        &\leq \gamma V_{\max}d_{\infty}(\pi,\pi^{\prime}).\label{ieq:92}
    \end{align}
    Thus, we can upper bound the difference between $l(f,\tilde{f},\pi;\bX)$ and $l(f,\tilde{f},\pi^{\prime};\bX)$ by $d_{\infty}(\pi,\pi^{\prime})$ as 
    \begin{align}
        &\bigl|l(f,\tilde{f},\pi;\bX)-l(f,\tilde{f},\pi^{\prime};\bX)\bigr|\nonumber\\
        &\quad\leq \Big|\big(f(\bar{\bS},\bar{\bA})-\barr-\gamma \tilde{f}(\bar{\bS},\pi)\big)^{2}-\big(f(\bar{\bS},\bar{\bA})-\barr-\gamma \tilde{f}(\bar{\bS},\pi^{\prime})\big)^{2}\Big|\nonumber\\*
        &\quad\qquad+\Big|\big(\calT^{\pi}\tilde{f}(\bar{\bS},\bar{\bA})-\barr-\gamma \tilde{f}(\bar{\bS},\pi)\big)^{2}-\big(\calT^{\pi^{\prime}}\tilde{f}(\bar{\bS},\bar{\bA})-\barr-\gamma \tilde{f}(\bar{\bS},\pi^{\prime})\big)^{2}\Big|\nonumber\\
        &\quad=\Big|\big(\gamma \tilde{f}(\bar{\bS},\pi^{\prime})-\gamma \tilde{f}(\bar{\bS},\pi)\big)\big(2f(\bar{\bS},\bar{\bA})-2\barr-\gamma \tilde{f}(\bar{\bS},\pi^{\prime})-\gamma \tilde{f}(\bar{\bS},\pi)\big)\Big|,\nonumber\\
        &\quad\qquad+\biggl|\big(\calT^{\pi}\tilde{f}(\bar{\bS},\bar{\bA})-\calT^{\pi^{\prime}}\tilde{f}(\bar{\bS},\bar{\bA})-\gamma \tilde{f}(\bar{\bS},\pi)+\gamma \tilde{f}(\bar{\bS},\pi^{\prime})\big)\nonumber\\
        &\quad\qquad \cdot\Bigl(\calT^{\pi^{\prime}}\tilde{f}(\bar{\bS},\bar{\bA})-2\barr-\gamma \tilde{f}\bigl(\bar{\bS},\pi^{\prime}-\gamma \tilde{f}(\bar{\bS},\pi)\bigr)\Bigr)^{2}\biggr|\nonumber,
    \end{align}
    where the inequality follows from the triangle inequality. Combined with inequalities \eqref{ieq:91} and \eqref{ieq:92}, it can be further upper bounded as
    \begin{align}
        &\bigl|l(f,\tilde{f},\pi;\bX)-l(f,\tilde{f},\pi^{\prime};\bX)\bigr|\nonumber\\
        &\quad\leq \gamma V_{\max}d_{\infty}(\pi,\pi^{\prime})\cdot 4V_{\max}+2\gamma V_{\max}d_{\infty}(\pi,\pi^{\prime})\cdot 4V_{\max}\nonumber\\
        &\quad=12\gamma V_{\max}^{2}d_{\infty}(\pi,\pi^{\prime}).\label{ieq:89}
    \end{align}
    From the definition of the $\varepsilon$-cover and inequality \eqref{ieq:89}, for any $\pi\in\Pi$, there exist a policy $\pi^{\prime}\in\calC(\Pi,\varepsilon,d_{\infty})$ such that for any $f,\tilde{f}\in\calF_{\tf}$,
    \begin{align}
        \big|l(f,\tilde{f},\pi;\bX)-l(f,\tilde{f},\pi^{\prime};\bX)\big|\leq 12\gamma\varepsilon V_{\max}^{2}.\label{ieq:42}
    \end{align}
    Substituting inequality \eqref{ieq:42} into the term involving $l(f,\tilde{f},\pi;\bX)$ in inequality \eqref{ieq:41}, we have that for all $f_{\btheta},f_{\btheta^{\prime}}\in\calF_{\tf}$ and all policy $\pi\in\Pi$, with probability at least $1-\delta$, 
    \begin{align}
        &\Big|\bbE_{\nu\times\bar{P^{*}}}[l(f_{\btheta},f_{\btheta^{\prime}},\pi;\bX)]-\frac{1}{n}\sum_{i=1}^{n}l(f_{\btheta},f_{\btheta^{\prime}},\pi;\bX_{i})\Big|\nonumber\\
        &\quad\leq 30\gamma V_{\max}^{2}\varepsilon+10V_{\max}\frac{C}{n}+\frac{1}{2}\bbE_{\nu\times\bar{P^{*}}}\bigl[l(f_{\btheta},f_{\btheta^{\prime}},\pi;\bX)\bigr]\nonumber\\
        &\quad\qquad+\frac{32V_{\max}^{2}}{n}\bigg[2(m+1)L^{2}d^{2}\log\biggl(\frac{4mdB_{V}B_{QK}B_{a}B_{b}}{\Delta}\biggr)\nonumber\\
        &\quad\qquad+2(m+1)Ld^{2}\log B_{w}+\log\frac{2\calN(\Pi,\varepsilon,d_{\infty})}{\delta}\bigg].\nonumber
    \end{align}
    Setting $\varepsilon=1/n$ and $C=5V_{\max}$, we obtain the desired result. Therefore, we conclude the proof of Proposition \ref{thm:main_concen}.
\end{proof}

\section{Proof of Proposition~\ref{prop:mlebound}}\label{app:propmlebound}

\begin{proof}[Proof of Proposition~\ref{prop:mlebound}]
    We adopt a Bayesian framework to prove the desired result. The total variation is first upper bounded through Pinsker's inequality. Then the derived upper bounded is further relaxed by the bounds related to the KL divergence. For ease of notation, we denote the parameters of the neural network as 
    \begin{align}
        \btheta= [\bW_{QK}^{1:L},\bW_{V}^{1:L},\ba^{1:L},\bb^{1:L}]\nonumber.
    \end{align}
    
    \textbf{Step 1: Bound the total variation distance with Pinsker's inequality.}
    
    From Pinsker's inequality, the total variation between two conditional distribution can be bounded as
    \begin{lemma}[Lemma 25 in \cite{agarwal2020flambe}]\label{lem:tvbound}
        For any two conditional probability densities $P(\cdot\,|\,\bar{\bS},\bar{\bA}), P^{\prime}(\cdot\,|\,\bar{\bS},\bar{\bA})$ and any distribution $\nu\in\Delta(\bar{\calS}\times\bar{\calA})$,we have
        \begin{align}
            \bbE_{\nu}\Big[\tv\big(P(\cdot\,|\,\bar{\bS},\bar{\bA}), P^{\prime}(\cdot\,|\,\bar{\bS},\bar{\bA})\big)^{2}\Big]\!\leq\! -2\log\bigg(\bbE_{(\bar{\bS},\bar{\bA})\sim\nu,\bar{\bS}^{\prime}\sim P(\cdot\,|\,\bar{\bS},\bar{\bA})}\bigg[\exp\bigg(\!-\!\frac{1}{2}\log\frac{P(\bar{\bS}^{\prime}\,|\,\bar{\bS},\bar{\bA})}{P^{\prime}(\bar{\bS}^{\prime}\,|\,\bar{\bS},\bar{\bA})}\bigg)\bigg]\bigg).\nonumber
        \end{align}
    \end{lemma}
    Thus, we only need to upper bound the right-hand side of the inequality in Lemma~\ref{lem:tvbound}. We adopt a Bayesian framework to relax this upper bound.
    \begin{lemma}[Lemma 2.1 in \cite{zhang2006}]\label{lem:bayesexpect}
        Given a distribution $P$ on $\Theta$, for all $Q\gg P$ on $\Theta$ and all measurable real-valued function $L(\btheta;\calD):\Theta\times (\bar{\calS}\times\bar{\calA})^{n}\rightarrow\bbR$, we have
        \begin{align}
            \bbE_{\calD}\bigg[\exp\Big\{\bbE_{Q}\big[L(\btheta;\calD)-\log\bbE_{\calD}[ e^{L(\theta;\calD)}]\big]-\kl(Q\,\|\,P)\Big\}\bigg]\leq 1,\nonumber
        \end{align}
        where $\bbE_{\calD}[\,\cdot\,]$ is the expectation with respect to the underlying distribution of $\{(\bar{\bS}_{i},\bar{\bA}_{i},\bar{\bS}_{i}^{\prime}\}_{i=1}^{n}$, i.e., $(\nu\times P^{*})^{n}$.
    \end{lemma}
    
    By Lemma~\ref{lem:bayesexpect} and the Chernoff inequality, we have that with probability at least $1-\delta/2$, 
    \begin{align}
        -\bbE_{Q}\big[\log  \bbE_{\calD}[ e^{L(\btheta;\calD)}]\big]\leq-\bbE_{Q}\big[L(\btheta;\calD)\big]+\kl(Q\,\|\,Q_{0})+\log\frac{2}{\delta},\label{ieq:52}
    \end{align}
    where $\bbE_{\calD}[\,\cdot\,]$ is the expectation with respect to the underlying distribution of $\{(\bar{\bS}_{i},\bar{\bA}_{i},\bar{\bS}_{i}^{\prime})\}_{i=1}^{n}$, i.e., $(\nu\times P^{*})^{n}$, and $Q$ and $Q_{0}$ are two distributions on $\Theta$. 
    
    Take $L(\btheta;\calD)=-\frac{1}{4}   \sum_{i=1}^{n}\log(P^{*}(\bar{\bS}_{i}^{\prime}\,|\,\bar{\bS}_{i},\bar{\bA}_{i})/P_{\btheta}(\bar{\bS}_{i}^{\prime}\,|\,\bar{\bS}_{i},\bar{\bA}_{i}))$, where $\calD=\{(\bar{\bS}_{i},\bar{\bA}_{i},r_{i},\bar{\bS}_{i}^{\prime})\}_{i=1}^{n}$. Then the left-hand side of inequality \eqref{ieq:52} becomes
    \begin{align}
        -\bbE_{Q}\big[\log\bbE_{\calD}[ e^{L(\btheta;\calD)}]\big]&=-\bbE_{Q}\Bigg[\log\bbE_{\calD}\bigg[\exp\Big(-\frac{1}{4}\sum_{i=1}^{n}\log\frac{P^{*}(\bar{\bS}_{i}^{\prime}\,|\,\bar{\bS}_{i},\bar{\bA}_{i})}{P_{\btheta}(\bar{\bS}_{i}^{\prime}\,|\,\bar{\bS}_{i},\bar{\bA}_{i})}\Big)\bigg]\Bigg]\nonumber\\
        &=-\bbE_{Q}\Bigg[n\log\bbE_{(\bar{\bS},\bar{\bA},\bar{\bS}^{\prime})\sim\nu\times P^{*}}\bigg[\exp\Big(-\frac{1}{4}\log\frac{P^{*}(\bar{\bS}^{\prime}\,|\,\bar{\bS},\bar{\bA})}{P_{\btheta}(\bar{\bS}^{\prime}\,|\,\bar{\bS},\bar{\bA})}\Big)\bigg]\Bigg].\nonumber
    \end{align}
    
    \textbf{Step 2: Control the fluctuation of the both sides of inequality \eqref{ieq:52} introduced by $Q$.}
    
    Since $\hat{P}_{\rm MLE}$ is a random variable, we want to derive an uniform bound for all $\btheta\in\Theta$. Because the left-hand side of inequality \eqref{ieq:52} takes the expectation with respect to the distribution $Q$ on $\Theta$, which is chosen as the uniform distribution on the neighborhood around a fixed parameter $\btheta\in\Theta$, we need to control the fluctuation of the left-hand side of inequality \eqref{ieq:52} due to the distribution $Q$ around $\btheta$. For any two parameters $\btheta$ and $\tilde{\btheta}$, we define the logarithm of the ratio between the transition kernels induced by them as
    \begin{align}
        e(\btheta,\tilde{\btheta};\bar{\bS}^{\prime},\bar{\bS},\bar{\bA})&=\log\frac{P_{\tilde{\btheta}}(\bar{\bS}^{\prime}\,|\,\bar{\bS},\bar{\bA})}{P_{\btheta}(\bar{\bS}^{\prime}\,|\,\bar{\bS},\bar{\bA})}\nonumber\\
        &=\log \bigg(\exp\biggl(-\frac{\bigl\|\bar{\bS}^{\prime}-\bF_{\tilde{\btheta}}(\bar{\bS},\bar{\bA})\bigr\|_{\rmF}^{2}}{2\sigma^{2}}\biggr)\bigg/\exp\biggl(-\frac{\bigl\|\bar{\bS}^{\prime}-\bF_{\btheta}(\bar{\bS},\bar{\bA})\bigr\|_{\rmF}^{2}}{2\sigma^{2}}\biggr)\biggr)\nonumber\\
        &=\frac{\bigl\|\bar{\bS}^{\prime}-\bF_{\btheta}(\bar{\bS},\bar{\bA})\bigr\|_{\rmF}^{2}-\bigl\|\bar{\bS}^{\prime}-\bF_{\tilde{\btheta}}(\bar{\bS},\bar{\bA})\bigr\|_{\rmF}^{2}}{2\sigma^{2}}.\nonumber
    \end{align}
    
    To upper bound the absolute value of $e(\btheta,\tilde{\btheta};\bar{\bS}^{\prime},\bar{\bS},\bar{\bA})$, we need to bound the norm of the output of the neural network.
    \begin{proposition}\label{prop:normbound}
        For any $\bX\in\bbR^{N\times d}$, any $\bW_{QK},\bW_{V}\in\bbR^{d\times d}$, $\ba\in\bbR^{dm}$, $\bb\in\bbR^{d\times dm}$ and two positive conjugate numbers $p,q\in\bbR$, we have
        \begin{align}
            &\Bigl\|\bigl(\SM(\bX\bW_{QK}\bX^{\top})\bX\bW_{V}+\rff(\bX,\ba,\bb)\bigr)^{\top}\Bigr\|_{p,\infty}\nonumber\\
            &\quad\leq \|\bW_{V}^{\top}\|_{p,q}\|\bX^{\top}\|_{p,\infty}+\biggl[\sum_{k=1}^{d}\biggl(\sum_{j=1}^{m}|a_{kj}|\|\bb_{kj}\|_{q}\|\bX^{\top}\|_{p,\infty}\biggr)^{p}\biggr]^{1/p}.\nonumber
        \end{align}
    \end{proposition}
    \begin{proof}
        See Appendix~\ref{app:propnormbound} for a detailed proof.
    \end{proof}
    Proposition~\ref{prop:normbound} shows that $\|\bF_{\tilde{\btheta}}(\bar{\bS},\bar{\bA})\|_{\rmF}\leq \sqrt{N}B^{*}$ for all $\btheta\in\Theta$, $\bar{\bS}\in \bar{\calS}$ and $\bar{\bA}\in \bar{\calA}$, where $B^{*}=B_{V}+md^{1/2}B_{a}B_{b}$.
    As a consequence, we have 
    \begin{align}
        \big|e(\btheta,\tilde{\btheta};\bar{\bS}^{\prime},\bar{\bS},\bar{\bA})\big|&\leq \frac{1}{2\sigma^{2}}\Bigl(\bigl\|\bar{\bS}^{\prime}-\bF_{\btheta}(\bar{\bS},\bar{\bA})\bigr\|_{\rmF}+\bigl\|\bar{\bS}^{\prime}-\bF_{\tilde{\btheta}}(\bar{\bS},\bar{\bA})\bigr\|_{\rmF}\Bigr)\bigl\|\bF_{\btheta}(\bar{\bS},\bar{\bA})-\bF_{\tilde{\btheta}}(\bar{\bS},\bar{\bA})\bigr\|_{\rmF}\nonumber\\* 
        &\leq \frac{1}{\sigma^{2}}\big(\|\bar{\bvarepsilon}\|_{\rmF}+\sqrt{N}B^{*}\big)\bigl\|\bF_{\btheta}(\bar{\bS},\bar{\bA})-\bF_{\tilde{\btheta}}(\bar{\bS},\bar{\bA})\bigr\|_{\rmF},\label{ieq:55}
    \end{align}
    where these two inequalities follow from the triangle inequality. For two parameters $\btheta$ and $\tilde{\btheta}$, we define the upper bound of the difference between the dynamic functions induced by them as 
    \begin{align}
        \Delta(\btheta,\tilde{\btheta})= \max_{(\bar{\bS},\bar{\bA})\in\bar{\calS}\times\bar{\calA}}\bigl\|\bF_{\btheta}(\bar{\bS},\bar{\bA})-\bF_{\tilde{\btheta}}(\bar{\bS},\bar{\bA})\bigr\|_{\rmF}.\nonumber
    \end{align}
    For a fixed parameter $\btheta$, the left-hand side of inequality \eqref{ieq:52} can be lower bounded as
    \begin{align}
        &-\bbE_{Q}\big[\log\bbE_{\calD}e^{L(\tilde{\btheta};\calD)}\big]\nonumber\\
        &\quad=-\bbE_{Q}\biggl[n\log\bbE_{(\bar{\bS},\bar{\bA},\bar{\bS}^{\prime})\sim\nu\times P^{*}}\biggl[\exp\biggl(-\frac{1}{4}e(\btheta,\tilde{\btheta};\bar{\bS}^{\prime},\bar{\bS},\bar{\bA})-\frac{1}{4}\log\frac{P^{*}(\bar{\bS}^{\prime}\,|\,\bar{\bS},\bar{\bA})}{P_{\btheta}(\bar{\bS}^{\prime}\,|\,\bar{\bS},\bar{\bA})}\biggr)\biggr]\biggr]\nonumber\\
        &\quad\geq -\frac{n}{2}\log\bbE_{\nu\times P^{*}}\biggl[\exp\biggl(-\frac{1}{2}\log\frac{P^{*}(\bar{\bS}^{\prime}\,|\,\bar{\bS},\bar{\bA})}{P_{\btheta}(\bar{\bS}^{\prime}\,|\,\bar{\bS},\bar{\bA})}\biggr)\biggr]\nonumber\\
        &\quad\qquad-\bbE_{Q}\bigg[\frac{n}{2}\log\bbE_{\nu\times P^{*}}\bigg[\exp\Big(-\frac{1}{2}e(\btheta,\tilde{\btheta};\bar{\bS}^{\prime},\bar{\bS},\bar{\bA})\Big)\bigg]\bigg]\nonumber\\ 
        &\quad\geq \frac{n}{4}\bbE_{\nu}\Big[\tv\big(P^{*}(\cdot\,|\,\bar{\bS},\bar{\bA}), P_{\btheta}(\cdot\,|\,\bar{\bS},\bar{\bA})\big)^{2}\Big]-\bbE_{Q}\biggl[\frac{n}{2}\log\bbE_{\nu\times P^{*}}\biggl[\exp\biggl(-\frac{1}{2}e(\btheta,\tilde{\btheta};\bar{\bS}^{\prime},\bar{\bS},\bar{\bA})\biggr)\biggr]\biggr]\label{ieq:54},
    \end{align}
    where the first inequality follows from the Cauchy--Schwarz inequality, and the last inequality follows from Lemma~\ref{lem:tvbound}. The second term of inequality \eqref{ieq:54} can be bounded as
    \begin{align}
        &\log\bbE_{\nu\times P^{*}}\bigg[\exp\biggl(-\frac{1}{2}e(\btheta,\tilde{\btheta};\bar{\bS}^{\prime},\bar{\bS},\bar{\bA})\biggr)\bigg]\nonumber\\
        &\quad\leq \log\bbE_{\nu\times P^{*}}\bigg[\exp\biggl(\frac{1}{2}\big|e(\btheta,\tilde{\btheta};\bar{\bS}^{\prime},\bar{\bS},\bar{\bA})\big|\biggr)\bigg]\nonumber\\
        &\quad\leq \log\bbE_{\bar{\bvarepsilon}\sim\calN(0,\sigma^{2}\bI)}\bigg[\exp\biggl(\frac{1}{2\sigma^{2}}\big(\|\bar{\bvarepsilon}\|_{\rmF}+\sqrt{N}B^{*}\big)\Delta(\btheta,\tilde{\btheta})\biggr)\bigg]\nonumber\\ 
        &\quad = \sqrt{N}B^{*}\Delta(\btheta,\tilde{\btheta})+\log\bbE_{\bar{\bvarepsilon}\sim\calN(0,\sigma^{2}\bI)}\bigg[\exp\biggl(\frac{1}{2\sigma^{2}}\|\bar{\bvarepsilon}\|_{\rmF}\Delta(\btheta,\tilde{\btheta})\biggr)\bigg],\nonumber
    \end{align}
    where the second inequality follows from inequality \eqref{ieq:55}. Since Lemma \ref{lem:vecnorm} shows that $\|\bX\|_{\rmF}\leq\|\bX\|_{1,1}$, we further have
    \begin{align}
        &\log\bbE_{\nu\times P^{*}}\bigg[\exp\biggl(-\frac{1}{2}e(\btheta,\tilde{\btheta};\bar{\bS}^{\prime},\bar{\bS},\bar{\bA})\biggr)\bigg]\nonumber\\
        &\quad \leq \sqrt{N}B^{*}\Delta(\btheta,\tilde{\btheta})+\log\bbE_{\bar{\bvarepsilon}\sim\calN(0,\sigma^{2}\bI)}\biggl[\exp\biggl(\frac{1}{2\sigma^{2}}\|\bar{\bvarepsilon}\|_{1,1}\Delta(\btheta,\tilde{\btheta})\biggr)\biggr]\nonumber\\
        &\quad =\sqrt{N}B^{*}\Delta(\btheta,\tilde{\btheta})+Nd\log\bbE_{\varepsilon\sim\calN(0,\sigma^{2})}\bigg[\exp\biggl(\frac{\Delta(\btheta,\tilde{\btheta})}{2\sigma^{2}}|\varepsilon|\biggr)\bigg],\label{ieq:107}
    \end{align}
    where the inequality follows from Lemma~\ref{lem:vecnorm}. The moment generating function of the folded normal distribution is (see \cite{tsagris2014folded})
    \begin{align}
        \bbE_{\varepsilon\sim\calN(0,\sigma^{2})}\Bigl[\exp\bigl(\lambda|\zeta|\bigr)\Bigr]=2\exp(\sigma^{2}\lambda^{2}/2)\big[1-\Phi(-\sigma\lambda)\big],\label{ieq:108}
    \end{align}
    where $\Phi(\cdot)$ is the cumulative distribution function of $\calN(0,1)$. From the Taylor expansion of $\Phi(\cdot)$, we have
    \begin{align}
        2\big[1-\Phi(-\sigma\lambda)\big]\leq 1+\sqrt{\frac{3}{\pi}}\sigma\lambda\label{ieq:109}
    \end{align}
    for small enough $\lambda$. Since $\log(1+x)\leq x$ for $x>0$, substituting inequalities \eqref{ieq:107}, \eqref{ieq:108} and \eqref{ieq:109} into inequality \eqref{ieq:54}, we have  
    \begin{align}
         &-\bbE_{Q}\big[\log\bbE_{\calD}e^{L(\tilde{\btheta};\calD)}\big]\nonumber\\*
         &\quad\geq \frac{n}{4}\bbE_{\nu}\Big[\tv\big(P^{*}(\cdot\,|\,\bar{\bS},\bar{\bA}), P_{\btheta}(\cdot\,|\,\bar{\bS},\bar{\bA})\big)^{2}\Big]\nonumber\\*
         &\quad\qquad-\frac{n}{2}\bbE_{\tilde{\btheta}\sim Q}\biggl[\sqrt{N}B^{*}\Delta(\btheta,\tilde{\btheta})+Nd\biggl(\frac{\Delta^{2}(\btheta,\tilde{\btheta})}{8\sigma^{2}}+\frac{1}{2\sigma}\sqrt{\frac{3}{\pi}}\Delta(\btheta,\tilde{\btheta})\biggr) \biggr],\label{ieq:58}
    \end{align}
    for small enough $\Delta(\btheta,\tilde{\btheta})$, which is set to $O(1/n)$ later.
    
    For the scaled right-hand side of inequality \eqref{ieq:52}, we have
    \begin{align}
        &\frac{4}{n}\biggl\{-\bbE_{Q}\big[L(\btheta;\calD)]+\kl(Q\,\|\,Q_{0})+\log\frac{2}{\delta}\biggr\}\nonumber\\
        &\quad=\frac{1}{n}\sum_{i=1}^{n}\log\frac{P^{*}(\bar{\bS}_{i}^{\prime}\,|\,\bar{\bS}_{i},\bar{\bA}_{i})}{P_{\btheta}(\bar{\bS}_{i}^{\prime}\,|\,\bar{\bS}_{i},\bar{\bA}_{i})}+\bbE_{\tilde{\btheta}\sim Q}\biggl[\frac{1}{n}\sum_{i=1}^{n}e(\btheta,\tilde{\btheta};\bar{\bS}^{\prime}_{i},\bar{\bS}_{i},\bar{\bA}_{i})\biggr]+\frac{4}{n}\biggl[\kl(Q\,\|\,Q_{0})+\log\frac{2}{\delta}\biggr]\nonumber\\
        &\quad\leq \frac{1}{n}\sum_{i=1}^{n}\log\frac{P^{*}(\bar{\bS}_{i}^{\prime}\,|\,\bar{\bS}_{i},\bar{\bA}_{i})}{P_{\btheta}(\bar{\bS}_{i}^{\prime}\,|\,\bar{\bS}_{i},\bar{\bA}_{i})}+\bbE_{\tilde{\btheta}\sim Q}\biggl[\frac{1}{n}\sum_{i=1}^{n}\frac{1}{\sigma^{2}}\big(\|\bar{\bvarepsilon}_{i}\|_{\rmF}+\sqrt{N}B^{*}\big)\Delta(\btheta,\tilde{\btheta})\biggr]\nonumber\\
        &\quad\qquad+\frac{4}{n}\biggl[\kl(Q\,\|\,Q_{0})+\log\frac{2}{\delta}\biggr],\label{ieq:59}
    \end{align}
    where the last inequality follows from inequality \eqref{ieq:55} and the definition of $\Delta(\btheta,\tilde{\btheta})$. To upper bound the right-hand side of inequality \eqref{ieq:59}, we need to upper bound $\|\bar{\bvarepsilon}\|_{\rmF}$, which can be achieved by combining the upper bound of the moment generating function of $\|\bar{\bvarepsilon}\|_{\rmF}$
    \begin{align}
        \bbE\Bigl[\exp\bigl(\lambda\|\bar{\bvarepsilon}\|_{\rmF}\bigr)\Bigr]\leq \bbE\Bigl[\exp\bigl(\lambda\|\bar{\bvarepsilon}\|_{1,1}\bigr)\Bigr]=\bigg(\bbE_{\varepsilon\sim\calN(0,\sigma^{2})}\Big[\exp\big(\lambda|\varepsilon|\big)\Big]\bigg)^{Nd}\leq \big(2\exp(\sigma^{2}\lambda^{2}/2)\big)^{Nd}\nonumber
    \end{align}
    and the Chernoff inequality. Thus, with probability at least $1-\delta/2$, we have 
    \begin{align}
        \frac{1}{n}\sum_{i=1}^{n}\|\bar{\bvarepsilon}_{i}\|_{\rmF}\leq \sqrt{2N^{2}d^{2}\sigma^{2}+\frac{2Nd\sigma^{2}}{n}\log\frac{2}{\delta}}.\label{ieq:60}
    \end{align}
    Substituting inequalities \eqref{ieq:58}, \eqref{ieq:59} and \eqref{ieq:60} into inequality \eqref{ieq:52}, we have that for any $\btheta\in\Theta$ and any two distributions $Q$ and $Q_{0}$, the following inequality holds with probability at least $1-\delta$
    \begin{align}
        &\bbE_{\nu}\Big[\tv\big(P^{*}(\cdot\,|\,\bar{\bS},\bar{\bA}), P_{\btheta}(\cdot\,|\,\bar{\bS},\bar{\bA})\big)^{2}\Big]\label{ieq:61}\\
        &\quad\leq \frac{1}{n}\sum_{i=1}^{n}\log\frac{P^{*}(\bar{\bS}_{i}^{\prime}\,|\,\bar{\bS}_{i},\bar{\bA}_{i})}{P_{\btheta}(\bar{\bS}_{i}^{\prime}\,|\,\bar{\bS}_{i},\bar{\bA}_{i})}+2\bbE_{\tilde{\btheta}\sim Q}\biggl[\sqrt{N}B^{*}\Delta(\btheta,\tilde{\btheta})+Nd\biggl(\frac{\Delta^{2}(\btheta,\tilde{\btheta})}{8\sigma^{2}}+\frac{1}{2\sigma}\sqrt{\frac{3}{\pi}}\Delta(\btheta,\tilde{\btheta})\biggr) \biggr]\nonumber\\
        &\quad\qquad+\bigg(\sqrt{2N^{2}d^{2}\sigma^{2}+\frac{2Nd\sigma^{2}}{n}\log\frac{2}{\delta}}+\sqrt{N}B^{*}\bigg)\bbE_{\tilde{\btheta}\sim Q}\big[\Delta(\btheta,\tilde{\btheta})\big]+\frac{4}{n}\biggl[\kl(Q\,\|\,Q_{0})+\log\frac{2}{\delta}\biggr].\nonumber
    \end{align}
    
    For any fixed $\btheta=[\bW_{QK}^{1:L},\bW_{V}^{1:L},\ba^{1:L},\bb^{1:L}]$, we set $Q_{0}$ as the product of the uniform distribution of each parameter on the whole space and $Q$ as the product of the uniform distribution of each parameter on the neighborhood around $\btheta$, i.e.,
    \begin{align}
        Q_{0}=&\U\big(\bbB(0,B_{w},\|\cdot\|_{q})\big)\cdot\prod_{i=1}^{L}\bigg[\U\big(\bbB(0,B_{QK},\|\cdot\|_{p,q})\big)\cdot\U\big(\bbB(0,B_{V},\|\cdot\|_{p,q})\big)\nonumber\\*
        &\qquad\cdot\Big(\U\big(\bbB(0,B_{a},|\cdot|)\big)\cdot\U\big(\bbB(0,B_{b},\|\cdot\|_{q})\big)\Big)^{md}\bigg],\quad\text{ and}\nonumber\\
        Q=&\U\big(\bbB(\bw,\varepsilon_{w},\|\cdot\|_{q})\big)\cdot\prod_{i=1}^{L}\bigg[\U\big(\bbB(\bW_{QK}^{(i)\top},\varepsilon_{QK}^{(i)},\|\cdot\|_{p,q})\big)\cdot\U\big(\bbB(\bW_{V}^{(i)\top},\varepsilon_{V}^{(i)},\|\cdot\|_{p,q})\big)\nonumber\\
        &\qquad\cdot\prod_{j\in[m],k\in[d]}\Big(\U\big(\bbB(a_{kj}^{(i)},\varepsilon_{a,kj}^{(i)},|\cdot|)\big)\cdot\U\big(\bbB(\bb_{kj}^{(i)},\varepsilon_{b,kj}^{(i)},\|\cdot\|_{q})\big)\Big)\bigg],\nonumber
    \end{align}
    For a constant $C>0$, we define $\Delta= C/(4LnN^{3/2}dB^{*})$. For $i\in[L]$, $j\in[m]$, and $k\in[d]$, we set
    \begin{align}
        \varepsilon_{QK}^{(i)}&=\big(2B_{V}\big)^{-1}\big[B_{V}(1+4B_{QK})+d^{\frac{1}{2}}mB_{a}B_{b}\big]^{-L+i}\Delta,\nonumber\\
        \varepsilon_{V}^{(i)}&=\big[B_{V}(1+4B_{QK})+d^{\frac{1}{2}}mB_{a}B_{b}\big]^{-L+i}\Delta,\nonumber\\
        \varepsilon_{a,kj}^{(i)}&=d^{-\frac{1}{2}}\big(mB_{b}\big)^{-1}\big[B_{V}(1+4B_{QK})+d^{\frac{1}{2}}mB_{a}B_{b}\big]^{-L+i}\Delta ,\nonumber\\
        \varepsilon_{b,kj}^{(i)}&=d^{-\frac{1}{2}}\big(mB_{a}\big)^{-1}\big[B_{V}(1+4B_{QK})+d^{\frac{1}{2}}mB_{a}B_{b}\big]^{-L+i}\Delta.\nonumber
    \end{align}
    By Proposition~\ref{prop:nnerror}, we have
    \begin{align}
        \bbE_{\tilde{\btheta}\sim Q}\big[\Delta(\btheta,\tilde{\btheta})\big]\leq \frac{C}{nNdB^{*}}, \quad \bbE_{\tilde{\btheta}\sim Q}\big[\Delta^{2}(\btheta,\tilde{\btheta})\big]\leq \frac{C^{2}}{n^{2}N^{2}d^{2}B^{*2}}.\label{ieq:62}
    \end{align}
    Sine $Q$ and $Q_{0}$ are product  distributions, $\kl(Q\,\|\,Q_{0})$ is the sum of the KL-divergences between each constituent distribution. For the KL-divergence between the distributions of $\bW_{QK}$, 
    \begin{align}
        &\kl\Big(\U\big(\bbB(\bW_{QK}^{(i)\top},\varepsilon_{QK}^{(i)},\|\cdot\|_{p,q})\big)\,\Big\|\,\U\big(\bbB(0,B_{QK},\|\cdot\|_{p,q})\big)\Big)\leq 2(L-i)d^{2}\log\biggl(\frac{4mdB_{V}B_{QK}B_{a}B_{b}}{\Delta}\biggr),\nonumber
    \end{align}
    for $i\in[L]$. Similar bounds for the KL divergence of the distributions of parameters $\bW_{V}^{(i)}$, $a_{kj}^{(i)}$, $\bb_{kj}^{(i)}$ and $\bw$ for $i\in[L]$, $k\in[d]$ and $j\in[m]$ can be derived by replacing $d^{2}$ by the dimension of the parameter. Thus, we have
    \begin{align}
        \kl(Q\,\|\,Q_{0})\leq 2(m+1)L^{2}d^{2}\log\biggl(\frac{4mdB_{V}B_{QK}B_{a}B_{b}}{\Delta}\biggr).\label{ieq:63}
    \end{align}
    Substituting Eqn.~\eqref{ieq:62} and \eqref{ieq:63} into Eqn.~\eqref{ieq:61}, we have that for any $\btheta\in\Theta$ and any two distributions $Q$ and $Q_{0}$, the following inequality holds with probability at least $1-\delta$
    \begin{align}
        &\bbE_{\nu}\Big[\tv\big(P^{*}(\cdot\,|\,\bar{\bS},\bar{\bA}), P_{\btheta}(\cdot\,|\,\bar{\bS},\bar{\bA})\big)^{2}\Big]\nonumber\\
        &\quad\leq \frac{1}{n}\sum_{i=1}^{n}\log\frac{P^{*}(\bar{\bS}_{i}^{\prime}\,|\,\bar{\bS}_{i},\bar{\bA}_{i})}{P_{\btheta}(\bar{\bS}_{i}^{\prime}\,|\,\bar{\bS}_{i},\bar{\bA}_{i})}\nonumber\\
        &\quad\qquad+O\bigg(\frac{C}{n}+\frac{1}{n}(m+1)L^{2}d^{2}\log\biggl(\frac{4NLmdB^{*}B_{V}B_{QK}B_{a}B_{b}n}{C}\biggr)+\frac{1}{n}\log\frac{1}{\delta}\bigg).\nonumber
    \end{align}
    Take $\btheta=\hat{\btheta}_{\rm MLE}$, which is the estimate derived in Eqn.~\eqref{algo:mpolicy}. Since it is the maximum likelihood estimate, we have
    \begin{align}
        \frac{1}{n}\sum_{i=1}^{n}\log\frac{P^{*}(\bar{\bS}_{i}^{\prime}\,|\,\bar{\bS}_{i},\bar{\bA}_{i})}{P_{\hat{\btheta}_{\rm MLE}}(\bar{\bS}_{i}^{\prime}\,|\,\bar{\bS}_{i},\bar{\bA}_{i})}\leq 0,\nonumber
    \end{align}
    which proves the desired result. Therefore, this concludes the proof of Proposition \ref{prop:mlebound}.
\end{proof}

\section{Proof of Proposition~\ref{prop:tvconcentrate}}\label{app:proptvconcentrate}
\begin{proof}[Proof of Proposition~\ref{prop:tvconcentrate}]
    We adopt the PAC-Bayes framework to prove the desired result. Define $l(\btheta,\bar{\bS},\bar{\bA})=\tv(P^{*}(\cdot\,|\, \bar{\bS},\bar{\bA}),P_{\btheta}(\cdot\,|\, \bar{\bS},\bar{\bA}))^{2}$. Then we have 
    \begin{align}
        {\rm Var}\big(l(\btheta,\bar{\bS},\bar{\bA})\big)\leq\bbE_{(\bar{\bS},\bar{\bA})\sim\nu}\big[l(\btheta,\bar{\bS},\bar{\bA})^{2}\big]\leq  \bbE_{(\bar{\bS},\bar{\bA})\sim\nu}\big[l(\btheta,\bar{\bS},\bar{\bA})\big],\nonumber
    \end{align}
    which implies that $l(\btheta,\bar{\bS},\bar{\bA})$ satisfies the conditions of Proposition~\ref{prop:pacbayes} with $b=c=1$. Thus, Proposition~\ref{prop:pacbayes} shows that for any distributions $Q$ and $Q_{0}$ on $\Theta$, the following inequality holds with probability at least $1-\delta$
    \begin{align}
        \biggl|\bbE_{Q}\biggl[\bbE_{\calD}\bigl[l(\btheta,\bar{\bS},\bar{\bA})\bigr]-\frac{1}{n}\sum_{i=1}^{n}l(\btheta,\bar{\bS}_{i},\bar{\bA}_{i})\biggr]\biggr|\leq \lambda \bbE_{Q,\calD}\big[l(\btheta,\bar{\bS},\bar{\bA})\big]+\frac{1}{n\lambda}\bigg[\kl(Q\,\|\,Q_{0})+\log\frac{2}{\delta}\bigg],\label{ieq:64}
    \end{align}
    for $0<\lambda\leq 1/2$. Since we want to derive the generalization error bound for all $\btheta\in\Theta$ uniformly, we set $Q$ as the uniform distribution on the neighborhood of any fixed $\btheta$ and $Q_{0}$ as the uniform distribution on $\Theta$. To derive the uniform generalization bound for any $\btheta\in\Theta$, we need to control the fluctuation of inequality \eqref{ieq:64} induced by $Q$. 
    
    With triangle inequality, for any $\btheta,\tilde{\btheta}\in\Theta$, we have 
    \begin{align}
        &\tv\big(P^{*}(\cdot| \bar{\bS},\bar{\bA}),P_{\tilde{\btheta}}(\cdot\,|\, \bar{\bS},\bar{\bA})\big)^{2}\nonumber\\
        &\quad\leq \tv\big(P^{*}(\cdot\,|\, \bar{\bS},\bar{\bA}),P_{\btheta}(\cdot\,|\, \bar{\bS},\bar{\bA})\big)^{2}+\tv\big(P_{\tilde{\btheta}}(\cdot\,|\, \bar{\bS},\bar{\bA}),P_{\btheta}(\cdot\,|\, \bar{\bS},\bar{\bA})\big)^{2}\nonumber\\
        &\quad\qquad+2\tv\big(P^{*}(\cdot\,|\, \bar{\bS},\bar{\bA}),P_{\btheta}(\cdot\,|\, \bar{\bS},\bar{\bA})\big)\tv\big(P_{\tilde{\btheta}}(\cdot\,|\, \bar{\bS},\bar{\bA}),P_{\btheta}(\cdot\,|\, \bar{\bS},\bar{\bA})\big)\label{ieq:65}\\
        &\tv\big(P^{*}(\cdot\,|\, \bar{\bS},\bar{\bA}),P_{\tilde{\btheta}}(\cdot\,|\, \bar{\bS},\bar{\bA})\big)^{2}\nonumber\\
        &\quad\geq \tv\big(P^{*}(\cdot\,|\, \bar{\bS},\bar{\bA}),P_{\btheta}(\cdot\,|\, \bar{\bS},\bar{\bA})\big)^{2}+\tv\big(P_{\tilde{\btheta}}(\cdot\,|\, \bar{\bS},\bar{\bA}),P_{\btheta}(\cdot\,|\, \bar{\bS},\bar{\bA})\big)^{2}\nonumber\\
        &\quad\qquad-2\tv\big(P^{*}(\cdot\,|\, \bar{\bS},\bar{\bA}),P_{\btheta}(\cdot\,|\, \bar{\bS},\bar{\bA})\big)\tv\big(P_{\tilde{\btheta}}(\cdot\,|\, \bar{\bS},\bar{\bA}),P_{\btheta}(\cdot\,|\, \bar{\bS},\bar{\bA})\big).\label{ieq:66}
    \end{align}
    For two parameters $\btheta$ and $\tilde{\btheta}$, we define the upper bound of the difference between the dynamic functions induced by them as $\Delta(\btheta,\tilde{\btheta})= \max_{(\bar{\bS},\bar{\bA})\in\bar{\calS}\times\bar{\calA}}\|\bF_{\btheta}(\bar{\bS},\bar{\bA})-\bF_{\tilde{\btheta}}(\bar{\bS},\bar{\bA})\|_{\rmF}$. By Pinsker's inequality, we then have
    \begin{align}
        \max_{(\bar{\bS},\bar{\bA})\in\bar{\calS}\times\bar{\calA}}\tv\big(P_{\tilde{\btheta}}(\cdot\,|\, \bar{\bS},\bar{\bA}),(P_{\btheta}(\cdot\,|\, \bar{\bS},\bar{\bA})\big)&\leq \max_{(\bar{\bS},\bar{\bA})\in\bar{\calS}\times\bar{\calA}}\sqrt{\kl\big(P_{\tilde{\btheta}}(\cdot\,|\, \bar{\bS},\bar{\bA})\,\|\,P_{\btheta}(\cdot\,|\, \bar{\bS},\bar{\bA})\big)/2}\nonumber\\
        &=\frac{1}{2\sigma}\max_{(\bar{\bS},\bar{\bA})\in\bar{\calS}\times\bar{\calA}}\bigl\|\bF_{\btheta}(\bar{\bS},\bar{\bA})-\bF_{\tilde{\btheta}}(\bar{\bS},\bar{\bA})\bigr\|_{\rmF}\nonumber\\ 
        &=\frac{\Delta(\btheta,\tilde{\btheta})}{2\sigma},\label{ieq:67}
    \end{align}
    where the first equality follows from the expression of the KL divergence between two Gaussian random vectors. Substituting inequalities \eqref{ieq:65}, \eqref{ieq:66} and \eqref{ieq:67} into the left-hand side of inequality~\eqref{ieq:64}, for a fixed $\btheta\in\Theta$ we have
    \begin{align}
        &\biggl|\bbE_{Q}\biggl[\bbE_{\calD}\bigl[l(\tilde{\btheta},\bar{\bS},\bar{\bA})\bigr]-\frac{1}{n}\sum_{i=1}^{n}l(\tilde{\btheta},\bar{\bS}_{i},\bar{\bA}_{i})\biggr] \biggr|\nonumber\\
        &\quad\geq \bigg|\bbE_{\calD}\bigl[l(\btheta,\bar{\bS},\bar{\bA})\bigr]-\frac{1}{n}\sum_{i=1}^{n}l(\btheta,\bar{\bS}_{i},\bar{\bA}_{i}) \bigg|-\frac{5\bbE_{Q}\big[\Delta(\btheta,\tilde{\btheta})\big]}{2\sigma}.\label{ieq:110}
    \end{align}
    Similarly, for the right-hand side of inequality \eqref{ieq:64}, we have 
    \begin{align}
        &\lambda \bbE_{Q,\calD}\big[l(\tilde{\btheta},\bar{\bS},\bar{\bA})\big]+\frac{1}{n\lambda}\bigg[\kl(Q\,\|\,Q_{0})+\log\frac{2}{\delta}\bigg]\nonumber\\
        &\quad \leq \lambda \bbE_{\calD}\big[l(\btheta,\bar{\bS},\bar{\bA})\big]+\frac{1}{n\lambda}\bigg[\kl(Q\,\|\,Q_{0})+\log\frac{2}{\delta}\bigg]+\frac{3\lambda\bbE_{Q}\big[\Delta(\btheta,\tilde{\btheta})\big]}{2\sigma}.\label{ieq:111}
    \end{align}
    Substituting inequalities \eqref{ieq:110} and \eqref{ieq:111} into inequality \eqref{ieq:64}, we have that for any distributions $Q$ and $Q_{0}$ on $\Theta$
    \begin{align}
        &\biggl|\bbE_{\calD}[l(\btheta,\bar{\bS},\bar{\bA})]-\frac{1}{n}\sum_{i=1}^{n}l(\btheta,\bar{\bS}_{i},\bar{\bA}_{i}) \biggr|\nonumber\\
        &\quad\leq \lambda \bbE_{\calD}\big[l(\btheta,\bar{\bS},\bar{\bA})\big]+\frac{1}{n\lambda}\bigg[\kl(Q\,\|\,Q_{0})+\log\frac{2}{\delta}\bigg]+(3\lambda+5)\frac{\bbE_{Q}\big[\Delta(\btheta,\tilde{\btheta})\big]}{2\sigma}\label{ieq:114}
    \end{align}
    holds with probability at least $1-\delta$.
    
    For any fixed $\btheta=[\bW_{QK}^{1:L},\bW_{V}^{1:L},\ba^{1:L},\bb^{1:L}]$, we set $Q$ and $Q_{0}$ as
    \begin{align}
        Q_{0}&=\U\big(\bbB(0,B_{w},\|\cdot\|_{q})\big)\cdot\prod_{i=1}^{L}\bigg[\U\big(\bbB(0,B_{QK},\|\cdot\|_{p,q})\big)\cdot\U\big(\bbB(0,B_{V},\|\cdot\|_{p,q})\big)\nonumber\\
        &\qquad\cdot\Big(\U\big(\bbB(0,B_{a},|\cdot|)\big)\cdot\U\big(\bbB(0,B_{b},\|\cdot\|_{q})\big)\Big)^{md}\bigg]\nonumber\\
        Q&=\U\big(\bbB(\bw,\varepsilon_{w},\|\cdot\|_{q})\big)\cdot\prod_{i=1}^{L}\bigg[\U\big(\bbB(\bW_{QK}^{(i)\top},\varepsilon_{QK}^{(i)},\|\cdot\|_{p,q})\big)\cdot\U\big(\bbB(\bW_{V}^{(i)\top},\varepsilon_{V}^{(i)},\|\cdot\|_{p,q})\big)\nonumber\\
        &\qquad\cdot\prod_{j\in[m],k\in[d]}\Big(\U\big(\bbB(a_{kj}^{(i)},\varepsilon_{a,kj}^{(i)},|\cdot|)\big)\cdot\U\big(\bbB(\bb_{kj}^{(i)},\varepsilon_{b,kj}^{(i)},\|\cdot\|_{q})\big)\Big)\bigg],\nonumber
    \end{align}
     For a constant $C>0$, we define $\Delta= C/(4LnN^{1/2})$. For $i\in[L]$, $j\in[m]$, and $k\in[d]$, we set 
    \begin{align}
        \varepsilon_{QK}^{(i)}&=\big(2c_{p,q}B_{V}\big)^{-1}\big[B_{V}(1+4c_{p,q}B_{QK})+d^{\frac{1}{p}}mB_{a}B_{b}\big]^{-L+i}\Delta,\nonumber\\
        \varepsilon_{V}^{(i)}&=\big[B_{V}(1+4c_{p,q}B_{QK})+d^{\frac{1}{p}}mB_{a}B_{b}\big]^{-L+i}\Delta,\nonumber\\
        \varepsilon_{a,kj}^{(i)}&=d^{-\frac{1}{p}}\big(mB_{b}\big)^{-1}\big[B_{V}(1+4c_{p,q}B_{QK})+d^{\frac{1}{p}}mB_{a}B_{b}\big]^{-L+i}\Delta ,\nonumber\\
        \varepsilon_{b,kj}^{(i)}&=d^{-\frac{1}{p}}\big(mB_{a}\big)^{-1}\big[B_{V}(1+4c_{p,q}B_{QK})+d^{\frac{1}{p}}mB_{a}B_{b}\big]^{-L+i}\Delta .\nonumber
    \end{align}
    By Proposition~\ref{prop:nnerror}, we then have
    \begin{align}
        \bbE_{\tilde{\btheta}\sim Q}\big[\Delta(\btheta,\tilde{\btheta})\big]\leq \frac{C}{n}.\label{ieq:112}
    \end{align}
    Following the similar procedure in the proof of Proposition~\ref{prop:mlebound}, we have 
    \begin{align}
        \kl(Q\,\|\, Q_{0})\leq 2(m+1)L^{2}d^{2}\log\biggl(\frac{4mdB_{V}B_{QK}B_{a}B_{b}}{\Delta}\biggr).\label{ieq:113}
    \end{align}
    Substituting inequalities \eqref{ieq:112} and \eqref{ieq:113} into inequality \eqref{ieq:114}, we derive that for any $\btheta\in\Theta$, with probability at least $1-\delta$, the following inequality holds
    \begin{align}
        &\bigg|\bbE_{\calD}[l(\btheta,\bar{\bS},\bar{\bA})]-\frac{1}{n}\sum_{i=1}^{n}l(\btheta,\bar{\bS}_{i},\bar{\bA}_{i}) \bigg|\nonumber\\
        &\quad\leq \frac{1}{2} \bbE_{\calD}\big[l(\btheta,\bar{\bS},\bar{\bA})\big]+O\bigg(\frac{C}{n}+\frac{1}{n}(m+1)L^{2}d^{2}\log\biggl(\frac{4mdB_{V}B_{QK}B_{a}B_{b}}{\Delta}\biggr)+\frac{1}{n}\log\frac{1}{\delta}\bigg),\nonumber
    \end{align}
    where we take $\lambda=1/2$. Therefore, this concludes the proof of Proposition \ref{prop:tvconcentrate}.
\end{proof}

\section{Proof of Lemmas in Appendix~\ref{app:thmmain_mfree}}
\subsection{Proof of Lemma~\ref{lem:errcontrol1}}\label{app:lemerrcontrol1}
\begin{proof}[Proof of Lemma~\ref{lem:errcontrol1}]
    Let $g^{*}=\arginf_{f\in\calF_{\tf}}\calL(f,f_{\pi}^{*},\pi^{*};\calD)$. Then the Bellman error of the best approximation $f_{\pi^{*}}^{*}$ can be decomposed as
    \begin{align}
        \calE(f_{\pi^{*}}^{*},\pi^{*};\calD)&=\calL(f_{\pi^{*}}^{*},f_{\pi^{*}}^{*},\pi^{*};\calD)-\calL(g^{*},f_{\pi^{*}}^{*},\pi^{*};\calD)\nonumber\\
        &=\calL(f_{\pi^{*}}^{*},f_{\pi^{*}}^{*},\pi^{*};\calD)-\calL(\calT^{\pi^{*}}f_{\pi^{*}}^{*},f_{\pi^{*}}^{*},\pi^{*};\calD)\nonumber\\
        &\qquad+\calL(\calT^{\pi^{*}}f_{\pi^{*}}^{*},f_{\pi^{*}}^{*},\pi^{*};\calD)-\calL(g^{*},f_{\pi^{*}}^{*},\pi^{*};\calD).\label{ieq:95}
    \end{align}
    Note that the terms in inequality \eqref{ieq:95} can be bounded with their population version and the generalization error shown in Theorem~\ref{thm:main_concen}. With probability at least $1-\delta$, we have 
    \begin{align}
        \calL(f_{\pi^{*}}^{*},f_{\pi^{*}}^{*},\pi^{*};\calD)-\calL(\calT^{\pi^{*}}f_{\pi^{*}}^{*},f_{\pi^{*}}^{*},\pi^{*};\calD)&\leq \frac{3}{2}\varepsilon_{\calF}+\frac{e(\calF_{\tf},\pi^{*},\delta,n)}{n},\label{ieq:96}\\
        \calL(\calT^{\pi^{*}}f_{\pi^{*}}^{*},f_{\pi^{*}}^{*},\pi^{*};\calD)-\calL(g^{*},f_{\pi^{*}}^{*},\pi^{*};\calD)&\leq \frac{e(\calF_{\tf},\Pi,\delta,n)}{n}\label{ieq:97},
    \end{align}
    where inequality \eqref{ieq:96} follows from the definition of $f_{\pi^{*}}^{*}$, and inequality \eqref{ieq:97} follows from that $(g^{*}(\bar{\bS},\bar{\bA})-\calT^{\pi^{*}}f_{\pi^{*}}^{*}(\bar{\bS},\bar{\bA}))^{2}\geq 0$.
    Substituting inequalities \eqref{ieq:96} and \eqref{ieq:97} into inequality \eqref{ieq:95}, we have
    \begin{align}
        \calE(f_{\pi^{*}}^{*},\pi^{*};\calD)\leq \frac{3}{2}\varepsilon_{\calF}+\frac{2e(\calF_{\tf},\Pi,\delta,n)}{n}.\nonumber
    \end{align}
    This concludes the proof of Lemma \ref{lem:errcontrol1}.
\end{proof}

\subsection{Proof of Lemma~\ref{lem:errcontrol2}}\label{app:lemerrcontrol2}
\begin{proof}[Proof of Lemma~\ref{lem:errcontrol2}]
    Let $h_{\pi}^{*}=\arginf_{g\in\calF_{\tf}}\bbE_{\nu}[(g(\bar{\bS},\bar{\bA})-\calT^{\pi}f(\bar{\bS},\bar{\bA}))^{2}]$, which is the best approximation of $\calT^{\pi}f$. Then Assumption~\ref{assump:mfree1} implies that
    \begin{align}
        \bbE_{\nu}\Big[\big(h_{\pi}^{*}(\bar{\bS},\bar{\bA})-\calT^{\pi}f(\bar{\bS},\bar{\bA})\big)^{2}\Big]\leq \varepsilon_{\calF,\calF}.\label{ieq:98}
    \end{align}
    For any $f\in\calF(\pi,\varepsilon)$, the Bellman error of $f$ with respect to the policy $\pi$ can be decomposed as
    \begin{align}
        \calE(f,\pi;\calD)&=\calL(f,f,\pi;\calD)-\inf_{g\in\calF_{\tf}}\calL(g,f,\pi;\calD)\nonumber\\
        &\geq \calL(f,f,\pi;\calD)-\calL(h_{\pi}^{*},f,\pi;\calD)\nonumber\\
        &=\calL(f,f,\pi;\calD)-\calL(\calT^{\pi}f,f,\pi;\calD)+\calL(\calT^{\pi}f,f,\pi;\calD)-\calL(h_{\pi}^{*},f,\pi;\calD).\label{ieq:99}
    \end{align}
    
    Similar to Step 1, we bound the terms in inequality \eqref{ieq:99} with their population version and the generalization error bound in Theorem~\ref{thm:main_concen}. With probability at least $1-\delta$, we have 
    \begin{align}
        \calL(f,f,\pi;\calD)-\calL(\calT^{\pi}f,f,\pi;\calD)&\geq \frac{1}{2}\bbE_{\nu}\Big[\big(f(\bar{\bS},\bar{\bA})-\calT^{\pi}f(\bar{\bS},\bar{\bA})\big)^{2}\Big]-\frac{e(\calF_{\tf},\Pi,\delta,n)}{n},\quad \text{ and}\label{ieq:100}\\
        \calL(\calT^{\pi}f,f,\pi;\calD)-\calL(h_{\pi}^{*},f,\pi;\calD)&\geq -\frac{3}{2}\bbE_{\nu}\Big[\big(h_{\pi}^{*}(\bar{\bS},\bar{\bA})-\calT^{\pi}f(\bar{\bS},\bar{\bA})\big)^{2}\Big]-\frac{e(\calF_{\tf},\Pi,\delta,n)}{n}.\label{ieq:101}
    \end{align}
    Substituting inequalities \eqref{ieq:100} and \eqref{ieq:101} into inequality \eqref{ieq:99}, we have
    \begin{align}
        \bbE_{\nu}\Big[\big(f(\bar{\bS},\bar{\bA})-\calT^{\pi}f(\bar{\bS},\bar{\bA})\big)^{2}\Big]&\leq 2\calE(f,\pi;\calD)+4\frac{e(\calF_{\tf},\Pi,\delta,n)}{n}+3\bbE_{\nu}\Big[\big(h_{\pi}^{*}(\bar{\bS},\bar{\bA})-\calT^{\pi}f(\bar{\bS},\bar{\bA})\big)^{2}\Big]\nonumber\\
        &\leq 2\calE(f,\pi;\calD)+4\frac{e(\calF_{\tf},\Pi,\delta,n)}{n}+3\varepsilon_{\calF,\calF},\label{ieq:102}
    \end{align}
    where inequality \eqref{ieq:102} follows from inequality \eqref{ieq:98}. This concludes the proof of Lemma~\ref{lem:errcontrol2}.
\end{proof}

\section{Proofs of Supporting Propositions}
\subsection{Proof of Proposition~\ref{prop:pacbayes}}\label{app:proppacbayes}
To prove Proposition~\ref{prop:pacbayes}, we need the variational definition of the Kullback--Leibler divergence.
\begin{theorem}[Donsker--Varadhan representation \citep{belghazi2018mutual}]\label{thm:DV}
    Let $P$ and $Q$ be distributions on a common space $\calX$. Then
    \begin{align}
        \kl(P\,\|\,Q)=\sup_{g\in\calG} \bigg\{\bbE_{P}\big[g(X)\big]-\log\bbE_{Q}\Big[\exp\big(g(X)\big)\Big]\bigg\},\nonumber
    \end{align}
    where $\calG=\{g:\calX\rightarrow\bbR \ | \ \bbE_{Q}[\exp(g(X))]<\infty\}$.
\end{theorem}
\begin{proof}[Proof of Proposition~\ref{prop:pacbayes}]
    Since $|f(X)-\mu(f)|\leq b$ a.s., $f(X)$ is a bounded random variable. Then by \cite{wainwright2019high}, we have for $|\lambda|\leq 1/(2b)$,
    \begin{align}
        \bbE_{X}\Big[\exp\Big(\lambda \big(f(X)-\mu(f)\big)\Big)\Big]\leq \exp\big(\lambda^{2}\sigma^{2}(f)\big).\nonumber
    \end{align}
    Consequently, set $\varepsilon_{n}(f,\lambda)= \lambda\big[\mu(f)-\frac{1}{n}\sum_{i=1}^{n}f(X_{i})-\lambda \sigma^{2}(f)\big]$, then we have 
    \begin{align}
        \bbE_{X_{1:n}}\Big[\exp\big(n\varepsilon_{n}(f,\lambda)\big)\Big]=\bbE_{X}\bigg[\exp\Big(\lambda\big(\mu(f)-f(X)\big)-\lambda^{2}\sigma^{2}(f)\Big)\bigg]^{n}\leq 1\nonumber
    \end{align}
    for all $f\in\calF$ and $0<\lambda\leq \frac{1}{2b}$.
    
    By Markov's inequality, we have that for any distribution $P_{0}$ on  the function class $\calF$, the random variable $\varepsilon_{n}$ induced by random variables $\{X_{i}\}_{i=1}^{n}$ satisfies 
    \begin{align}
        \Pr\bigg(\bbE_{P_{0}}\Big[\exp\big(n\varepsilon_{n}(f,\lambda)\big)\Big]\geq \frac{2}{\delta}\bigg)\leq \frac{\delta}{2},\label{ieq:86}
    \end{align}
    where the probability is taken with respect to the distribution of $X_{i}$ for $i\in[n]$.
    
    Setting $g(f)=n\varepsilon_{n}(f,\lambda)$ in Theorem~\ref{thm:DV}, we have 
    \begin{align}
        \bbE_{Q}\big[n\varepsilon_{n}(f,\lambda)\big]\leq \kl(Q\,\|\,P_{0})+\log\bbE_{P_{0}}\Big[\exp\big(n\varepsilon_{n}(f,\lambda)\big)\Big].\label{ieq:87}
    \end{align}
    Combining inequalities \eqref{ieq:86} and \eqref{ieq:87}, with prob at least $1-\frac{\delta}{2}$, for $0<\lambda\leq \frac{1}{2b}$, we have
    \begin{align}
        \bbE_{Q}\Big[\bbE_{X}[f(X)]-\frac{1}{n}\sum_{i=1}^{n}f(X_{i})\Big]\leq \lambda \bbE_{Q}\big[\sigma^{2}(f)\big]+\frac{1}{n\lambda}\bigg[\kl(Q\,\|\,P_{0})+\log\frac{2}{\delta}\bigg],\nonumber
    \end{align}
    for all $Q$.
    Similarly, setting $\varepsilon_{n}^{\prime}(f,\lambda)= \lambda\big[\frac{1}{n}\sum_{i=1}^{n}f(X_{i})-\mu(f)-\lambda \sigma^{2}(f)\big]$, we have 
    \begin{align}
        \bbE_{Q}\Big[\frac{1}{n}\sum_{i=1}^{n}f(X_{i})-\bbE_{X}[f(X)]\Big]\leq \lambda \bbE_{Q}\big[\sigma^{2}(f)\big]+\frac{1}{n\lambda}\bigg[\kl(Q\,\|\,P_{0})+\log\frac{2}{\delta}\bigg],\label{ieq:131}
    \end{align}
    with probability at least $1-\frac{\delta}{2}$. The desired result can be proved using the union bound. When $\sigma^{2}(f)\leq c \mu(f)$ for all $f\in\calF$, the result follows from 
    substituting this condition into inequality \eqref{ieq:131}. Therefore, we conclude the proof of Proposition \ref{prop:pacbayes}.
\end{proof}
\subsection{Proof of Proposition~\ref{prop:nnerror}}\label{app:propnnerror}
\begin{proof}[Proof of Proposition~\ref{prop:nnerror}]
    To prove the desired result, we first analyze the error propagation through each layer. Then we combine the error propagation of each layer to derive the error bound of the whole network.
    
    \textbf{Step 1: Bound the difference of each layer.}
    
    For $i\in[L-1]$, we can bound the difference of the output of the $(i+1)^{\rm st}$ as
    \begin{align}
        &\Big\|\big(\bG_{\tf}^{(i+1)}(\bX;\bW_{QK}^{1:i+1},\bW_{V}^{1:i+1},\ba^{1:i+1},\bb^{1:i+1})-\bG_{\tf}^{(i+1)}(\bX;\tilde{\bW}_{QK}^{1:i+1},\tilde{\bW}_{V}^{1:i+1},\tilde{\ba}^{1:i+1},\tilde{\bb}^{1:i+1})\big)^{\top}\Big\|_{p,\infty}\nonumber\\
        &\quad=\Bigg\|\bigg( \SM\big(\bG_{\tf}^{(i)}\bW_{QK}^{(i+1)}\bG_{\tf}^{(i)\top}\big)\bG_{\tf}^{(i)}\bW_{V}^{(i+1)}+\rff\big(\bG_{\tf}^{(i)},\ba^{(i+1)},\bb^{(i+1)}\big)\nonumber\\*
        &\quad\qquad-\SM\big(\tilde{\bG}_{\tf}^{(i)}\tilde{\bW}_{QK}^{(i+1)}\tilde{\bG}_{\tf}^{(i)\top}\big)\tilde{\bG}_{\tf}^{(i)}\tilde{\bW}_{V}^{(i+1)}-\rff\big(\tilde{\bG}_{\tf}^{(i)},\tilde{\ba}^{(i+1)},\tilde{\bb}^{(i+1)}\big)\bigg)^{\top}\Bigg\|_{p,\infty}\nonumber\\
        &\quad\leq \bigg\|\Big( \SM\big(\bG_{\tf}^{(i)}\bW_{QK}^{(i+1)}\bG_{\tf}^{(i)\top}\big)\bG_{\tf}^{(i)}\bW_{V}^{(i+1)}-\SM\big(\tilde{\bG}_{\tf}^{(i)}\tilde{\bW}_{QK}^{(i+1)}\tilde{\bG}_{\tf}^{(i)\top}\big)\tilde{\bG}_{\tf}^{(i)}\tilde{\bW}_{V}^{(i+1)}\Big)^{\top}\bigg\|_{p,\infty}\nonumber\\
        &\quad\qquad +\bigg\|\Big(\rff\big(\bG_{\tf}^{(i)},\ba^{(i+1)},\bb^{(i+1)}\big)-\rff\big(\tilde{\bG}_{\tf}^{(i)},\tilde{\ba}^{(i+1)},\tilde{\bb}^{(i+1)}\big)\Big)^{\top}\bigg\|_{p,\infty},\label{ieq:20}
    \end{align}
    where $\bG_{\tf}^{(i)}$ and $\tilde{\bG}_{\tf}^{(i)}$ are shorthands for $\bG_{\tf}^{(i)}(\bX;\bW_{QK}^{1:i},\bW_{V}^{1:i},\ba^{1:i},\bb^{1:i})$ and $\bG_{\tf}^{(i)}(\bX;\tilde{\bW}_{QK}^{1:i},\tilde{\bW}_{V}^{1:i},\tilde{\ba}^{1:i},\tilde{\bb}^{1:i})$, respectively, and inequality \eqref{ieq:20} follows from the triangle inequality.
    
    Now we consider the first term in inequality \eqref{ieq:20}. For $i\in [L-1]$, with the triangle  inequality, we have 
    \begin{align}
        &\bigg\|\Big( \SM\big(\bG_{\tf}^{(i)}\bW_{QK}^{(i+1)}\bG_{\tf}^{(i)\top}\big)\bG_{\tf}^{(i)}\bW_{V}^{(i+1)}-\SM\big(\tilde{\bG}_{\tf}^{(i)}\tilde{\bW}_{QK}^{(i+1)}\tilde{\bG}_{\tf}^{(i)\top}\big)\tilde{\bG}_{\tf}^{(i)}\tilde{\bW}_{V}^{(i+1)}\Big)^{\top}\bigg\|_{p,\infty}\label{ieq:118}\\
        &\quad\leq \bigg\|\Big( \SM\big(\bG_{\tf}^{(i)}\bW_{QK}^{(i+1)}\bG_{\tf}^{(i)\top}\big)\bG_{\tf}^{(i)}\bW_{V}^{(i+1)}-\SM\big(\tilde{\bG}_{\tf}^{(i)}\bW_{QK}^{(i+1)}\tilde{\bG}_{\tf}^{(i)\top}\big)\tilde{\bG}_{\tf}^{(i)}\bW_{V}^{(i+1)}\Big)^{\top}\bigg\|_{p,\infty}\nonumber\\
        &\quad\qquad +\bigg\|\Big( \SM\big(\tilde{\bG}_{\tf}^{(i)}\bW_{QK}^{(i+1)}\tilde{\bG}_{\tf}^{(i)\top}\big)\tilde{\bG}_{\tf}^{(i)}\bW_{V}^{(i+1)}-\SM\big(\tilde{\bG}_{\tf}^{(i)}\tilde{\bW}_{QK}^{(i+1)}\tilde{\bG}_{\tf}^{(i)\top}\big)\tilde{\bG}_{\tf}^{(i)}\tilde{\bW}_{V}^{(i+1)}\Big)^{\top}\bigg\|_{p,\infty}.\nonumber
    \end{align}
    Thus, we need the upper bounds of the two terms in the right-hand side of inequality \eqref{ieq:118}, which are stated as following.
    \begin{proposition}\label{prop:translip}
    For any $\bX,\tilde{\bX}\in\bbR^{N\times d}$, any $\bW_{V},\bW_{QK},\tilde{\bW}_{V},\tilde{\bW}_{QK}\in\bbR^{d\times d}$ and two positive conjugate numbers $p,q\in\bbR$, if $\|\bX^{\top}\|_{p,\infty},\|\tilde{\bX}^{\top}\|_{p,\infty}\leq B_{X}$, $\|\bW_{QK}^{\top}\|_{p,q}\leq B_{QK}$, and $\|\bW_{V}^{\top}\|_{p,q}\leq B_{V}$, then we have 
        \begin{align}
            &\Big\|\big(\SM(\bX\bW_{QK}\bX^{\top})\bX\bW_{V}-\SM(\tilde{\bX}\bW_{QK}\tilde{\bX}^{\top})\tilde{\bX}\bW_{V}\big)^{\top}\Big\|_{p,\infty}\nonumber\\*
            &\quad\leq B_{V}\big(1+4c_{p,q}B_{X}^{2}\cdot B_{QK}\big)\|\bX^{\top}-\tilde{\bX}^{\top}\|_{p,\infty},\quad \text{ and }\nonumber\\
            &\Big\|\big(\SM(\bX\bW_{QK}\bX^{\top})\bX\bW_{V}-\SM(\bX\tilde{\bW}_{QK}\bX^{\top})\bX\tilde{\bW}_{V}\big)^{\top}\Big\|_{p,\infty}\nonumber\\
            &\quad\leq 2c_{p,q}B_{X}^{3}\cdot B_{V}\cdot\|\bW_{QK}^{\top}-\tilde{\bW}_{QK}^{\top}\|_{p,q}+B_{X}\big\|\bW_{V}^{\top}-\tilde{\bW}_{V}^{\top}\|_{p,q}.\nonumber
        \end{align}
        where $c_{p,q}=1$ if $p\leq q$, and $c_{p,q}=d^{1/q-1/p}$ otherwise.
    \end{proposition}
    \begin{proof}
        See Appendix~\ref{app:proptranslip} for a detailed proof.
    \end{proof}
    Thus, we have
    \begin{align}
        &\bigg\|\Big( \SM\big(\bG_{\tf}^{(i)}\bW_{QK}^{(i+1)}\bG_{\tf}^{(i)\top}\big)\bG_{\tf}^{(i)}\bW_{V}^{(i+1)}-\SM\big(\tilde{\bG}_{\tf}^{(i)}\tilde{\bW}_{QK}^{(i+1)}\tilde{\bG}_{\tf}^{(i)\top}\big)\tilde{\bG}_{\tf}^{(i)}\tilde{\bW}_{V}^{(i+1)}\Big)^{\top}\bigg\|_{p,\infty}\nonumber\\
        &\quad\leq B_{V}(1+4c_{p,q}B_{QK})\big\|\bG_{\tf}^{(i)\top}-\tilde{\bG}_{\tf}^{(i)\top}\big\|_{p,\infty}+2c_{p,q}B_{V}\|\bW_{QK}^{(i+1)\top}-\tilde{\bW}_{QK}^{(i+1)\top}\|_{p,q}\nonumber\\
        &\quad\qquad+\|\bW_{V}^{(i+1)\top}-\tilde{\bW}_{V}^{(i+1)\top}\|_{p,q},\label{ieq:22}
    \end{align}
    where the inequality follows from the fact that the radius of parameters are bounded and the norm of $\|\tilde{\bG}_{\tf}^{(i)\top}\big\|_{p,\infty}$ is bounded by 1 due to the normalization procedure.
    
    Now we consider the second term in inequality \eqref{ieq:20}. For $i\in [L-1]$, we have 
    \begin{align}
        &\bigg\|\Big(\rff\big(\bG_{\tf}^{(i)},\ba^{(i+1)},\bb^{(i+1)}\big)-\rff\big(\tilde{\bG}_{\tf}^{(i)},\tilde{\ba}^{(i+1)},\tilde{\bb}^{(i+1)}\big)\Big)^{\top}\bigg\|_{p,\infty}\nonumber\\
        &\quad\leq \bigg\|\Big(\rff\big(\bG_{\tf}^{(i)},\ba^{(i+1)},\bb^{(i+1)}\big)-\rff\big(\tilde{\bG}_{\tf}^{(i)},\ba^{(i+1)},\bb^{(i+1)}\big)\Big)^{\top}\bigg\|_{p,\infty}\nonumber\\
        &\quad\qquad +\bigg\|\Big(\rff\big(\tilde{\bG}_{\tf}^{(i)},\ba^{(i+1)},\bb^{(i+1)}\big)-\rff\big(\tilde{\bG}_{\tf}^{(i)},\tilde{\ba}^{(i+1)},\tilde{\bb}^{(i+1)}\big)\Big)^{\top}\bigg\|_{p,\infty}\label{ieq:119}
    \end{align}
    Thus, we need to upper bound   the two terms in the right-hand side of inequality \eqref{ieq:119}. These upper bounds  are stated as follows.
    \begin{proposition}\label{prop:rfflip}
        For any $\bX,\tilde{\bX}\in\bbR^{N\times d}$, $\ba,\tilde{\ba}\in\bbR^{dm}$, $\bb,\tilde{\bb}\in\bbR^{d\times dm}$ and two positive conjugate numbers $p,q\in\bbR$, if $\|\bX^{\top}\|_{p,\infty}\leq B_{X}$, $|a_{kj}|,|\tilde{a}_{kj}|\leq B_{a}$, and $\|\bb_{kj}\|_{q},\|\tilde{\bb}_{kj}\|_{q}\leq B_{b}$ for $k\in[d]$ and $j\in[m]$, then we have 
        \begin{align}
            &\Big\|\big(\rff(\bX,\ba,\bb)-\rff(\tilde{\bX},\ba,\bb)\big)^{\top}\Big\|_{p,\infty}\leq d^{\frac{1}{p}}mB_{a}\cdot B_{b}\cdot\big\|\bX^{\top}-\tilde{\bX}^{\top}\big\|_{p,\infty}, \text{ and }\nonumber\\
            &\Big\|\big(\rff(\bX,\ba,\bb)-\rff(\bX,\tilde{\ba},\tilde{\bb})\big)^{\top}\Big\|_{p,\infty}\nonumber\\
            &\quad\leq B_{b}\cdot B_{X}\bigg[\sum_{k=1}^{d}\bigg(\sum_{j=1}^{m}|a_{kj}-\tilde{a}_{kj}|\bigg)^{p}\bigg]^{\frac{1}{p}}+B_{a}\cdot B_{X}\bigg[\sum_{k=1}^{d}\bigg(\sum_{j=1}^{m}\|\bb_{kj}-\tilde{\bb}_{kj}\|_{q}\bigg)^{p}\bigg]^{\frac{1}{p}}.\nonumber
        \end{align}
    \end{proposition}
    \begin{proof}
        See Appendix~\ref{app:proprfflip} for a detailed proof.
    \end{proof}
    Thus, we have
    \begin{align}
        &\bigg\|\Big(\rff\big(\bG_{\tf}^{(i)},\ba^{(i+1)},\bb^{(i+1)}\big)-\rff\big(\tilde{\bG}_{\tf}^{(i)},\tilde{\ba}^{(i+1)},\tilde{\bb}^{(i+1)}\big)\Big)^{\top}\bigg\|_{p,\infty}\nonumber\\
        &\quad\leq d^{\frac{1}{p}}mB_{a}B_{b}\big\|\bG_{\tf}^{(i)\top}-\tilde{\bG}_{\tf}^{(i)\top}\big\|_{p,\infty}+B_{b}\bigg[\sum_{k=1}^{d}\bigg(\sum_{j=1}^{m}|a_{kj}^{(i+1)}-\tilde{a}_{kj}^{(i+1)}|\bigg)^{p}\bigg]^{\frac{1}{p}}\nonumber\\
        &\quad\qquad+B_{a}\bigg[\sum_{k=1}^{d}\bigg(\sum_{j=1}^{m}\|\bb_{kj}^{(i+1)}-\tilde{\bb}_{kj}^{(i+1)}\|_{q}\bigg)^{p}\bigg]^{\frac{1}{p}}\label{ieq:24}
    \end{align}
    where the inequality follows from the fact that the radius of parameters are bounded and the norm of $\|\tilde{\bG}_{\tf}^{(i)\top}\big\|_{p,\infty}$ is bounded by 1 due to the normalization procedure.
    
    Substituting inequalities \eqref{ieq:22} and \eqref{ieq:24} into inequality \eqref{ieq:20}, we have
    \begin{align}
        &\Big\|\big(\bG_{\tf}^{(i+1)}(\bX;\bW_{QK}^{1:i+1},\bW_{V}^{1:i+1},\ba^{1:i+1},\bb^{1:i+1})-\bG_{\tf}^{(i+1)}(\bX;\tilde{\bW}_{QK}^{1:i+1},\tilde{\bW}_{V}^{1:i+1},\tilde{\ba}^{1:i+1},\tilde{\bb}^{1:i+1})\big)^{\top}\Big\|_{p,\infty}\nonumber\\
        &\quad\leq \big[B_{V}(1+4c_{p,q}B_{QK})+d^{\frac{1}{p}}mB_{a}B_{b}\big]\big\|g_{tf}^{(i)\top}-\tilde{\bG}_{\tf}^{(i)\top}\big\|_{p,\infty}+2c_{p,q}B_{V}\|\bW_{QK}^{(i+1)\top}-\tilde{\bW}_{QK}^{(i+1)\top}\|_{p,q}\nonumber\\
        &\quad\qquad +\|\bW_{V}^{(i+1)\top}-\tilde{\bW}_{V}^{(i+1)\top}\|_{p,q}+B_{b}\bigg[\sum_{k=1}^{d}\bigg(\sum_{j=1}^{m}|a_{kj}^{(i+1)}-\tilde{a}_{kj}^{(i+1)}|\bigg)^{p}\bigg]^{\frac{1}{p}}\nonumber\\
        &\quad\qquad+B_{a}\bigg[\sum_{k=1}^{d}\bigg(\sum_{j=1}^{m}\|\bb_{kj}^{(i+1)}-\tilde{\bb}_{kj}^{(i+1)}\|_{q}\bigg)^{p}\bigg]^{\frac{1}{p}}.\label{ieq:25}
    \end{align}
    
    \textbf{Step 2: Combine the error bound of each layer in inequality \eqref{ieq:25}.}
    
    Repeating inequality \eqref{ieq:25} for $i\in[L-1]$, we derive
    \begin{align}
        &\Big\|\big(\bG_{\tf}^{(L)}(\bX;\bW_{QK}^{1:L},\bW_{V}^{1:L},\ba^{1:L},\bb^{1:L})-\bG_{\tf}^{(L)}(\bX;\tilde{\bW}_{QK}^{1:L},\tilde{\bW}_{V}^{1:L},\tilde{\ba}^{1:L},\tilde{\bb}^{1:L})\big)^{\top}\Big\|_{p,\infty}\nonumber\\
        &\quad\leq \sum_{i=1}^{L}\big[B_{V}(1+4c_{p,q}B_{QK})+d^{\frac{1}{p}}mB_{a}B_{b}\big]^{L-i}\bigg\{2c_{p,q}B_{V}\|\bW_{QK}^{(i)\top}-\tilde{\bW}_{QK}^{(i)\top}\|_{p,q}\nonumber\\
        &\quad\qquad+\|\bW_{V}^{(i)\top}-\tilde{\bW}_{V}^{(i)\top}\|_{p,q} +B_{b}\bigg[\sum_{k=1}^{d}\bigg(\sum_{j=1}^{m}|a_{kj}^{(i)}-\tilde{a}_{kj}^{(i)}|\bigg)^{p}\bigg]^{\frac{1}{p}}\nonumber\\
        &\quad\qquad+B_{a}\bigg[\sum_{k=1}^{d}\bigg(\sum_{j=1}^{m}\|\bb_{kj}^{(i)}-\tilde{\bb}_{kj}^{(i)}\|_{q}\bigg)^{p}\bigg]^{\frac{1}{p}}\bigg\}.\label{ieq:26}
    \end{align}
    For the output of the neural network, we have
    \begin{align}
        &\big|g_{\tf}(\bX;\bW_{QK}^{1:L},\bW_{V}^{1:L},\ba^{1:L},\bb^{1:L},\bw)-g_{\tf}(\bX;\tilde{\bW}_{QK}^{1:L},\tilde{\bW}_{V}^{1:L},\tilde{\ba}^{1:L},\tilde{\bb}^{1:L},\tilde{\bw})\big|\nonumber\\
        &\quad=\bigg|\Pi_{V_{\max}}\big(\frac{1}{N}\bbI_{N}\bG_{\tf}^{(L)}(\bX;\bW_{QK}^{1:L},\bW_{V}^{1:L},\ba^{1:L},\bb^{1:L})\bw\big)\nonumber\\
        &\quad\qquad-\Pi_{V_{\max}}\big(\frac{1}{N}\bbI_{N}\bG_{\tf}^{(L)}(\bX;\tilde{\bW}_{QK}^{1:L},\tilde{\bW}_{V}^{1:L},\tilde{\ba}^{1:L},\tilde{\bb}^{1:L})\tilde{\bw}\big)\bigg|\nonumber\\
        &\quad\leq \bigg|\frac{1}{N}\bbI_{N}\bG_{\tf}^{(L)}(\bX;\bW_{QK}^{1:L},\bW_{V}^{1:L},\ba^{1:L},\bb^{1:L})\bw-\frac{1}{N}\bbI_{N}\bG_{\tf}^{(L)}(\bX;\tilde{\bW}_{QK}^{1:L},\tilde{\bW}_{V}^{1:L},\tilde{\ba}^{1:L},\tilde{\bb}^{1:L})\tilde{\bw}\bigg|,\nonumber 
    \end{align}
    where the inequality follows from the contraction property of the normalization function. It can be further upper bounded as
    \begin{align}
        &\big|g_{\tf}(\bX;\bW_{QK}^{1:L},\bW_{V}^{1:L},\ba^{1:L},\bb^{1:L},\bw)-g_{\tf}(\bX;\tilde{\bW}_{QK}^{1:L},\tilde{\bW}_{V}^{1:L},\tilde{\ba}^{1:L},\tilde{\bb}^{1:L},\tilde{\bw})\big|\nonumber\\
        &\quad\leq \Big\|\bG_{\tf}^{(L)}(\bX;\bW_{QK}^{1:L},\bW_{V}^{1:L},\ba^{1:L},\bb^{1:L})\bw-\bG_{\tf}^{(L)}(\bX;\tilde{\bW}_{QK}^{1:L},\tilde{\bW}_{V}^{1:L},\tilde{\ba}^{1:L},\tilde{\bb}^{1:L})\tilde{\bw}\Big\|_{\infty}\nonumber\\
        &\quad\leq \big\|\bG_{\tf}^{(L)\top}-\tilde{\bG}_{\tf}^{(L)\top}\big\|_{p,\infty}\cdot\|\bw\|_{q}+\big\|\tilde{\bG}_{\tf}^{(L)\top}\big\|_{p,\infty}\cdot\|\bw-\tilde{\bw}\|_{q}\nonumber\\
        &\quad\leq B_{w}\big\|\bG_{\tf}^{(L)\top}-\tilde{\bG}_{\tf}^{(L)\top}\big\|_{p,\infty}+\|\bw-\tilde{\bw}\|_{q},\label{ieq:29}
    \end{align}
    where first inequality follows from H\"older's inequality, and the second inequality follows from Lemma~\ref{lem:matvec} with $u=p$, $v=q$ and $p=\infty$.
    
    Combining inequalities \eqref{ieq:26} and \eqref{ieq:29}, we have
    \begin{align}
        &\big|g_{\tf}(\bX;\bW_{QK}^{1:L},\bW_{V}^{1:L},\ba^{1:L},\bb^{1:L},\bw)-g_{\tf}(\bX;\tilde{\bW}_{QK}^{1:L},\tilde{\bW}_{V}^{1:L},\tilde{\ba}^{1:L},\tilde{\bb}^{1:L},\tilde{\bw})\big|\nonumber\\
        &\leq \|\bw-\tilde{\bw}\|_{q}+\sum_{i=1}^{L}B_{w}\big[B_{V}(1+4c_{p,q}B_{QK})+d^{\frac{1}{p}}mB_{a}B_{b}\big]^{L-i}\bigg\{2c_{p,q}B_{V}\|\bW_{QK}^{(i)\top}-\tilde{\bW}_{QK}^{(i)\top}\|_{p,q}\nonumber\\*
        &\quad+\|\bW_{V}^{(i)\top}\!-\!\tilde{\bW}_{V}^{(i)\top}\|_{p,q}\!+\!B_{b}\bigg[\sum_{k=1}^{d}\bigg(\sum_{j=1}^{m}|a_{kj}^{(i)}\!-\!\tilde{a}_{kj}^{(i)}|\bigg)^{p}\bigg]^{\frac{1}{p}}\!+\! B_{a}\bigg[\sum_{k=1}^{d}\bigg(\sum_{j=1}^{m}\|\bb_{kj}^{(i)}\!-\!\tilde{\bb}_{kj}^{(i)}\|_{q}\bigg)^{p}\bigg]^{\frac{1}{p}}\bigg\}.\nonumber
    \end{align}
    This concludes the proof.
\end{proof}
\subsection{Proof of Proposition~\ref{prop:translip}}\label{app:proptranslip}
\begin{proof}[Proof of Proposition~\ref{prop:translip}]
Let $\tau\in[N]$ and $\bx_{\tau}^{\top}$ be the $\tau^{\rm{th}}$ row of $\bX$. For the first inequality, we have 
    \begin{align}
        &\Big\|\big(\SM(\bX\bW_{QK}\bX^{\top})\bX\bW_{V}-\SM(\tilde{\bX}\bW_{QK}\tilde{\bX}^{\top})\tilde{\bX}\bW_{V}\big)^{\top}\Big\|_{p,\infty}\nonumber\\
        &\quad=\max_{\tau\in[N]}\|\SM(\bx_{\tau}^{\top}\bW_{QK}\bX^{\top})\bX\bW_{V}-\SM(\tilde{\bx}_{\tau}^{\top}\bW_{QK}\tilde{\bX}^{\top})\tilde{\bX}\bW_{V}\|_{p}\nonumber\\
        &\quad\leq \max_{\tau\in[N]}\|\SM(\bx_{\tau}^{\top}\bW_{QK}\bX^{\top})\bX\bW_{V}-\SM(\bx_{\tau}^{\top}\bW_{QK}\bX^{\top})\tilde{\bX}\bW_{V}\|_{p}\nonumber\\
        &\quad\qquad +\|\SM(\bx_{\tau}^{\top}\bW_{QK}\bX^{\top})\tilde{\bX}\bW_{V}-\SM(\tilde{\bx}_{\tau}^{\top}\bW_{QK}\tilde{\bX}^{\top})\tilde{\bX}\bW_{V}\|_{p}\nonumber\\ 
        &\quad= \max_{\tau\in[N]}\Big\|\bW_{V}^{\top}(\bX^{\top}-\tilde{\bX}^{\top})\big(\SM(\bx_{\tau}^{\top}\bW_{QK}\bX^{\top})\big)^{\top}\Big\|_{p}\nonumber\\
        &\quad\qquad +\Big\|\bW_{V}^{\top}\tilde{\bX}^{\top}\big(\SM(\bx_{\tau}^{\top}\bW_{QK}\bX^{\top})-\SM(\tilde{\bx}_{\tau}^{\top}\bW_{QK}\tilde{\bX}^{\top})\big)^{\top}\Big\|_{p},\nonumber
    \end{align}
    where the inequality follows from the triangle inequality. We further upper bounded it as
    \begin{align}
        &\Big\|\big(\SM(\bX\bW_{QK}\bX^{\top})\bX\bW_{V}-\SM(\tilde{\bX}\bW_{QK}\tilde{\bX}^{\top})\tilde{\bX}\bW_{V}\big)^{\top}\Big\|_{p,\infty}\nonumber\\
        &\quad\leq \max_{\tau\in[N]}\|\bW_{V}^{\top}(\tilde{\bX}^{\top}-\bX^{\top})\|_{p,\infty}+\|\bW_{V}^{\top}\tilde{\bX}^{\top}\|_{p,\infty}\cdot\|\SM(\bx_{\tau}^{\top}\bW_{QK}\bX^{\top})-\SM(\tilde{\bx}_{\tau}^{\top}\bW_{QK}\tilde{\bX}^{\top})\|_{1}\nonumber\\ 
        &\quad\leq \max_{\tau\in[N]}2\|\bW_{V}^{\top}\|_{p,q}\cdot\|\tilde{\bX}^{\top}\|_{p,\infty}\cdot\|\bx_{\tau}^{\top}\bW_{QK}\bX^{\top}-\tilde{\bx}_{\tau}^{\top}\bW_{QK}\tilde{\bX}^{\top}\|_{\infty}\nonumber\\
        &\quad\qquad +\|\bW_{V}^{\top}\|_{p,q}\cdot\|(\tilde{\bX}^{\top}-\bX^{\top})\|_{p,\infty},\label{ieq:6}
    \end{align}
    where the first inequality follows from Lemma~\ref{lem:matvec} with $u=\infty$ and $v=1$, and the last inequality follows from Lemma~\ref{lem:matmat} and Lemma~\ref{lem:smlip}. Now we consider the second term of inequality \eqref{ieq:6}, and we have
    \begin{align}
        &\|\bx_{\tau}^{\top}\bW_{QK}\bX^{\top}-\tilde{\bx}_{\tau}^{\top}\bW_{QK}\tilde{\bX}^{\top}\|_{\infty}\nonumber\\
        &\quad\leq\|\bx_{\tau}^{\top}\bW_{QK}\bX^{\top}-\bx_{\tau}^{\top}\bW_{QK}\tilde{\bX}^{\top}\|_{\infty}+\|\bx_{\tau}^{\top}\bW_{QK}\tilde{\bX}^{\top}-\tilde{\bx}_{\tau}^{\top}\bW_{QK}\tilde{\bX}^{\top}\|_{\infty}\nonumber\\
        &\quad=\big\|(\bX-\tilde{\bX})\bW_{QK}^{\top}\bx_{\tau}\big\|_{\infty}+\big\|\tilde{\bX}(\bW_{QK}^{\top}\bx_{\tau}-\bW_{QK}^{\top}\tilde{\bx}_{\tau})\big\|_{\infty}\nonumber\\
        &\quad \leq \|\bX^{\top}-\tilde{\bX}^{\top}\|_{p,\infty}\cdot\|\bW_{QK}^{\top}\bx_{\tau}\|_{q}+\|\tilde{\bX}^{\top}\|_{p,\infty}\cdot\|\bW_{QK}^{\top}(\bx_{\tau}-\tilde{\bx}_{\tau})\|_{q}\nonumber,
    \end{align}
    where the last inequality follows from Lemma~\ref{lem:matvec} with $u=p$, $v=q$ and $p=\infty$. We then bound the $\ell_{q}$ norm with the $\ell_{p}$ norm as
    \begin{align}
        &\|\bx_{\tau}^{\top}\bW_{QK}\bX^{\top}-\tilde{\bx}_{\tau}^{\top}\bW_{QK}\tilde{\bX}^{\top}\|_{\infty}\label{ieq:10}\\
        &\quad\leq c_{p,q}\Big[\|\bX^{\top}-\tilde{\bX}^{\top}\|_{p,\infty}\cdot\|\bW_{QK}^{\top}\bx_{\tau}\|_{p}+\|\tilde{\bX}^{\top}\|_{p,\infty}\cdot\|\bW_{QK}^{\top}(\bx_{\tau}-\tilde{\bx}_{\tau})\|_{p}\Big]\nonumber\\ 
        &\quad\leq c_{p,q}\Big[\|\bX^{\top}-\tilde{\bX}^{\top}\|_{p,\infty}\cdot\|\bW_{QK}^{\top}\|_{p,q}\cdot\|\bx_{\tau}\|_{p}+\|\tilde{\bX}^{\top}\|_{p,\infty}\cdot\|\bW_{QK}^{\top}\|_{p,q}\cdot\|\bx_{\tau}-\tilde{\bx}_{\tau}\|_{p}\Big]\nonumber\\ 
        &\quad\leq c_{p,q}\Big[\|\bX^{\top}-\tilde{\bX}^{\top}\|_{p,\infty}\cdot\|\bW_{QK}^{\top}\|_{p,q}\cdot\|\bX^{\top}\|_{p,\infty}+\|\tilde{\bX}^{\top}\|_{p,\infty}\cdot\|\bW_{QK}^{\top}\|_{p,q}\cdot\|\bX^{\top}-\tilde{\bX}^{\top}\|_{p,\infty}\Big]\nonumber,
    \end{align}
    where $c_{p,q}=1$ if $p\leq q$, and $c_{p,q}=d^{1/q-1/p}$ otherwise, the first inequality follows from Lemma~\ref{lem:vecnorm}, and the second inequality follows from Lemma~\ref{lem:matvec} with $u=q$ and $v=p$.
    
    Substituting inequality \eqref{ieq:10} into inequality \eqref{ieq:6}, we obtain
    \begin{align}
        &\|\SM(\bx_{\tau}^{\top}\bW_{QK}\bX^{\top})\bX\bW_{V}-\SM(\tilde{\bx}_{\tau}^{\top}\bW_{QK}\tilde{\bX}^{\top})\tilde{\bX}\bW_{V}\|_{p}\nonumber\\
        &\quad\leq \|\bW_{V}^{\top}\|_{p,q}\Big(1+2c_{p,q}\|\tilde{\bX}^{\top}\|_{p,\infty}\cdot\|\bW_{QK}^{\top}\|_{p,q}\big(\|\tilde{\bX}^{\top}\|_{p,\infty}+\|\bX^{\top}\|_{p,\infty}\big)\Big)\|\bX^{\top}-\tilde{\bX}^{\top}\|_{p,\infty} \nonumber
    \end{align}
as desired.

For the second inequality, we have 
    \begin{align}
        &\Big\|\big(\SM(\bX\bW_{QK}\bX^{\top})\bX\bW_{V}-\SM(\bX\tilde{\bW}_{QK}\bX^{\top})\bX\tilde{\bW}_{V}\big)^{\top}\Big\|_{p,\infty}\nonumber\\
        &\quad=\max_{\tau\in[N]}\|\SM(\bx_{\tau}^{\top}\bW_{QK}\bX^{\top})\bX\bW_{V}-\SM(\bx_{\tau}^{\top}\tilde{\bW}_{QK}\bX^{\top})\bX\tilde{\bW}_{V}\|_{p}\nonumber\\
        &\quad\leq\max_{\tau\in[N]}\|\SM(\bx_{\tau}^{\top}\bW_{QK}\bX^{\top})\bX\bW_{V}-\SM(\bx_{\tau}^{\top}\tilde{\bW}_{QK}\bX^{\top})\bX\bW_{V}\|_{p}\nonumber\\
        &\quad\qquad +\|\SM(\bx_{\tau}^{\top}\tilde{\bW}_{QK}\bX^{\top})\bX\bW_{V}-\SM(\bx_{\tau}^{\top}\tilde{\bW}_{QK}\bX^{\top})\bX\tilde{\bW}_{V}\|_{p}\nonumber\\ 
        &\quad=\max_{\tau\in[N]}\Big\|\bW_{V}^{\top}\bX^{\top}\big(\SM(\bx_{\tau}^{\top}\bW_{QK}\bX^{\top})-\SM(\bx_{\tau}^{\top}\tilde{\bW}_{QK}\bX^{\top})\big)^{\top}\Big\|_{p}\nonumber\\
        &\quad\qquad+\big\|(\bW_{V}^{\top}\bX^{\top}-\tilde{\bW}_{V}^{\top}\bX^{\top})\SM(\bx_{\tau}^{\top}\tilde{\bW}_{QK}\bX^{\top})^{\top}\big\|_{p},\nonumber
    \end{align}
    where the inequality follows from the triangle inequality. It can be further upper bounded as
    \begin{align}
        &\Big\|\big(\SM(\bX\bW_{QK}\bX^{\top})\bX\bW_{V}-\SM(\bX\tilde{\bW}_{QK}\bX^{\top})\bX\tilde{\bW}_{V}\big)^{\top}\Big\|_{p,\infty}\label{ieq:13}\\
        &\quad\leq\max_{\tau\in[N]} \|\bW_{V}^{\top}\bX^{\top}\|_{p,\infty}\cdot\|\SM(\bx_{\tau}^{\top}\bW_{QK}\bX^{\top})-\SM(\bx_{\tau}^{\top}\tilde{\bW}_{QK}\bX^{\top})\|_{1}\nonumber\\
        &\quad\qquad+\big\|(\bW_{V}^{\top}-\tilde{\bW}_{V}^{\top})\bX^{\top}\big\|_{p,\infty}\cdot\|\SM(\bx_{\tau}^{\top}\tilde{\bW}_{QK}\bX^{\top})\|_{1}\nonumber\\ 
        &\quad\leq\max_{\tau\in[N]} 2\|\bW_{V}^{\top}\bX^{\top}\|_{p,\infty}\cdot\|\bx_{\tau}^{\top}\bW_{QK}\bX^{\top}-\bx_{\tau}^{\top}\tilde{\bW}_{QK}\bX^{\top}\|_{\infty}+\big\|(\bW_{V}^{\top}-\tilde{\bW}_{V}^{\top})\bX^{\top}\big\|_{p,\infty}\nonumber\\ 
        &\quad\leq\max_{\tau\in[N]} 2\|\bW_{V}^{\top}\|_{p,q}\cdot\|\bX^{\top}\|_{p,\infty}\cdot\|\bx_{\tau}^{\top}\bW_{QK}\bX^{\top}-\bx_{\tau}^{\top}\tilde{\bW}_{QK}\bX^{\top}\|_{\infty}+\big\|\bW_{V}^{\top}-\tilde{\bW}_{V}^{\top}\|_{p,q}\cdot\|\bX^{\top}\big\|_{p,\infty}\nonumber,
    \end{align}
    where the first inequality follows from Lemma~\ref{lem:matvec} with $u=\infty$ and $v=1$, the second inequality follows from Lemma~\ref{lem:smlip}, and the last inequality follows from Lemma~\ref{lem:matmat}. Now we consider the first term of inequality \eqref{ieq:13} and have
    \begin{align}
        &\|\bx_{\tau}^{\top}\bW_{QK}\bX^{\top}-\bx_{\tau}^{\top}\tilde{\bW}_{QK}\bX^{\top}\|_{\infty}\nonumber\\
        &\quad\leq \|\bX^{\top}\|_{p,\infty}\cdot\|\bx_{\tau}^{\top}\bW_{QK}-\bx_{\tau}^{\top}\tilde{\bW}_{QK}\|_{q}\nonumber\\ 
        &\quad\leq c_{p,q}\|\bX^{\top}\|_{p,\infty}\cdot\|\bx_{\tau}^{\top}\bW_{QK}-\bx_{\tau}^{\top}\tilde{\bW}_{QK}\|_{p}\nonumber\\ 
        &\quad\leq c_{p,q}\|\bX^{\top}\|_{p,\infty}\cdot\|\bW_{QK}^{\top}-\tilde{\bW}_{QK}^{\top}\|_{p,q}\cdot\|\bx_{\tau}\|_{p}\nonumber\\ 
        &\quad\leq c_{p,q}\|\bX^{\top}\|_{p,\infty}^{2}\cdot\|\bW_{QK}^{\top}-\tilde{\bW}_{QK}^{\top}\|_{p,q},\label{ieq:17}
    \end{align}
    where $c_{p,q}=1$ if $p\leq q$, and $c_{p,q}=d^{1/q-1/p}$ otherwise, the first and third inequalities follows from Lemma~\ref{lem:matvec}, and the second inequality follows from Lemma~\ref{lem:vecnorm}.
    
    Combining Eqn.~\eqref{ieq:13} and \eqref{ieq:17}, we have
    \begin{align}
        &\|\SM(\bx_{\tau}^{\top}\bW_{QK}\bX^{\top})\bX\bW_{V}-\SM(\bx_{\tau}^{\top}\tilde{\bW}_{QK}\bX^{\top})\bX\tilde{\bW}_{V}\|_{p}\nonumber\\
        &\quad\leq 2c_{p,q}\|\bX^{\top}\|_{p,\infty}^{3}\cdot\|\bW_{V}^{\top}\|_{p,q}\cdot\|\bW_{QK}^{\top}-\tilde{\bW}_{QK}^{\top}\|_{p,q}+\big\|\bW_{V}^{\top}-\tilde{\bW}_{V}^{\top}\|_{p,q}\cdot\|\bX^{\top}\big\|_{p,\infty}.\nonumber
    \end{align}
    This concludes the proof.
\end{proof}
\subsection{Proof of Proposition~\ref{prop:rfflip}}\label{app:proprfflip}
\begin{proof}[Proof of Proposition~\ref{prop:rfflip}]
Let $\tau\in[N]$ and $\bx_{\tau}^{\top}$ be the $\tau^{\rm{th}}$ row of $\bX$. For the first inequality, we have 
    \begin{align}
        &\Big\|\big(\rff(\bX,\ba,\bb)-\rff(\tilde{\bX},\ba,\bb)\big)^{\top}\Big\|_{p,\infty}\nonumber\\
        &\quad=\max_{\tau\in[N]}\|\rff(\bx_{\tau},\ba,\bb)-\rff(\tilde{\bx}_{\tau},\ba,\bb)\|_{p}\nonumber\\
        &\quad=\max_{\tau\in[N]}\bigg[\sum_{k=1}^{d}\Big|\sum_{j=1}^{m}a_{kj}\big[\relu(\bb_{kj}^{\top}\bx_{\tau})-\relu(\bb_{kj}^{\top}\tilde{\bx}_{\tau})\big]\Big|^{p}\bigg]^{\frac{1}{p}},\nonumber
    \end{align}
    which follows from the definition of the rFF network. It can be upper bounded as
    \begin{align}
        &\Big\|\big(\rff(\bX,\ba,\bb)-\rff(\tilde{\bX},\ba,\bb)\big)^{\top}\Big\|_{p,\infty}\nonumber\\
        &\quad\leq \max_{\tau\in[N]}\bigg[\sum_{k=1}^{d}\bigg(\sum_{j=1}^{m}|a_{kj}|\cdot\big|\bb_{kj}^{\top}\bx_{\tau}-\bb_{kj}^{\top}\tilde{\bx}_{\tau}\big|\bigg)^{p}\bigg]^{\frac{1}{p}}\nonumber\\
        &\quad\leq \max_{\tau\in[N]}\bigg[\sum_{k=1}^{d}\bigg(\sum_{j=1}^{m}|a_{kj}|\cdot\|\bb_{kj}\|_{q}\cdot\|\bx_{\tau}-\tilde{\bx}_{\tau}\|_{p}\bigg)^{p}\bigg]^{\frac{1}{p}}\nonumber\\ 
        &\quad\leq\bigg[\sum_{k=1}^{d}\bigg(\sum_{j=1}^{m}|a_{kj}|\cdot\|\bb_{kj}\|_{q}\bigg)^{p}\bigg]^{\frac{1}{p}}\big\|\bX^{\top}-\tilde{\bX}^{\top}\big\|_{p,\infty},\nonumber 
    \end{align}
    where the first inequality follows from the fact that $\relu(\cdot)$ is $1$-Lipschitz, the second inequality follows from H\"older's inequality, and the last inequality follows from the definition of $\ell_{p,\infty}$ norm.

    For the second inequality, we have 
    \begin{align}
        &\Big\|\big(\rff(\bX,\ba,\bb)-\rff(\bX,\tilde{\ba},\tilde{\bb})\big)^{\top}\Big\|_{p,\infty}\nonumber\\
        &\quad=\max_{\tau\in[N]}\|\rff(\bx_{\tau},\ba,\bb)-\rff(\bx_{\tau},\tilde{\ba},\tilde{\bb})\|_{p}\nonumber\\
        &\quad=\max_{\tau\in[N]}\bigg[\sum_{k=1}^{d}\Big|\sum_{j=1}^{m}a_{kj}\relu(\bb_{kj}^{\top}\bx_{\tau})-\tilde{a}_{kj}\relu(\tilde{\bb}_{kj}^{\top}\bx_{\tau})\Big|^{p}\bigg]^{\frac{1}{p}}\nonumber\\
        &\quad\leq \max_{\tau\in[N]}\bigg[\sum_{k=1}^{d}\Big|\sum_{j=1}^{m}a_{kj}\relu(\bb_{kj}^{\top}\bx_{\tau})-\tilde{a}_{kj}\relu(\bb_{kj}^{\top}\bx_{\tau})\Big|^{p}\bigg]^{\frac{1}{p}}\nonumber\\
        &\quad\qquad+\bigg[\sum_{k=1}^{d}\Big|\sum_{j=1}^{m}\tilde{a}_{kj}\relu(\bb_{kj}^{\top}\bx_{\tau})-\tilde{a}_{kj}\relu(\tilde{\bb}_{kj}^{\top}\bx_{\tau})\Big|^{p}\bigg]^{\frac{1}{p}},\nonumber 
    \end{align}
    where the inequality follows from triangle inequality. Using the Lipschitz property of the $\relu$ function, it can be upper bounded as
    \begin{align}
        &\Big\|\big(\rff(\bX,\ba,\bb)-\rff(\bX,\tilde{\ba},\tilde{\bb})\big)^{\top}\Big\|_{p,\infty}\nonumber\\
        &\quad\leq \max_{\tau\in[N]}\bigg[\sum_{k=1}^{d}\bigg(\sum_{j=1}^{m}|a_{kj}-\tilde{a}_{kj}|\cdot|\bb_{kj}^{\top}\bx_{\tau}|\bigg)^{p}\bigg]^{\frac{1}{p}}+\bigg[\sum_{k=1}^{d}\bigg(\sum_{j=1}^{m}|\tilde{a}_{kj}|\cdot|\bb_{kj}^{\top}\bx_{\tau}-\tilde{\bb}_{kj}^{\top}\bx_{\tau}|\bigg)^{p}\bigg]^{\frac{1}{p}}\nonumber\\ 
        &\quad\leq \max_{\tau\in[N]}\bigg[\sum_{k=1}^{d}\bigg(\sum_{j=1}^{m}|a_{kj}-\tilde{a}_{kj}|\cdot\|\bb_{kj}\|_{q}\cdot\|\bx_{\tau}\|_{p}\bigg)^{p}\bigg]^{\frac{1}{p}}\nonumber\\
        &\quad\qquad+\bigg[\sum_{k=1}^{d}\bigg(\sum_{j=1}^{m}|\tilde{a}_{kj}|\cdot\|\bb_{kj}-\tilde{\bb}_{kj}\|_{q}\cdot\|\bx_{\tau}\|_{p}\bigg)^{p}\bigg]^{\frac{1}{p}}\nonumber\\ 
        &\quad\leq \bigg[\sum_{k=1}^{d}\bigg(\sum_{j=1}^{m}|a_{kj}-\tilde{a}_{kj}|\cdot\|\bb_{kj}\|_{q}\bigg)^{p}\bigg]^{\frac{1}{p}}\big\|\bX^{\top}\big\|_{p,\infty}\nonumber\\
        &\quad\qquad+\bigg[\sum_{k=1}^{d}\bigg(\sum_{j=1}^{m}|\tilde{a}_{kj}|\cdot\|\bb_{kj}-\tilde{\bb}_{kj}\|_{q}\bigg)^{p}\bigg]^{\frac{1}{p}}\big\|\bX^{\top}\big\|_{p,\infty},\nonumber 
    \end{align}
    where the first inequality follows from the fact  that $\relu(\cdot)$ is 1-Lipschitz, the second inequality follows from H\"older's inequality, and the last inequality follows from the definition of $\ell_{p,\infty}$ norm. This concludes the proof.
\end{proof}

\subsection{Proof of Proposition~\ref{prop:normbound}}\label{app:propnormbound}
\begin{proof}[Proof of Proposition~\ref{prop:normbound}]
    With triangle inequality, we have
    \begin{align}
        &\bigg\|\Big(\SM\big(\bX\bW_{QK}\bX^{\top}\big)\bX\bW_{V}+\rff\big(\bX,\ba,\bb\big)\Big)^{\top}\bigg\|_{p,\infty}\nonumber\\
        &\quad\leq \bigg\|\Big(\SM\big(\bX\bW_{QK}\bX^{\top}\big)\bX\bW_{V}\Big)^{\top}\bigg\|_{p,\infty}+\bigg\|\Big(\rff\big(\bX,\ba,\bb\big)\Big)^{\top}\bigg\|_{p,\infty}.\label{ieq:74}
    \end{align}
    Let $\tau\in[N]$ and $\bx_{\tau}^{\top}$ be the $\tau^{\rm{th}}$ row of $\bX$. Then the first term in the right-hand side of Eqn.~\eqref{ieq:74} is
    \begin{align}
        \bigg\|\Big(\SM\big(\bX\bW_{QK}\bX^{\top}\big)\bX\bW_{V}\Big)^{\top}\bigg\|_{p,\infty}&=\max_{\tau\in[N]}\bigg\|\Big(\SM\big(\bx_{\tau}^{\top}\bW_{QK}\bX^{\top}\big)\bX\bW_{V}\Big)^{\top}\bigg\|_{p}\nonumber\\
        &\leq \max_{\tau\in [N]} \|\bW_{V}^{\top}\bX^{\top}\|_{p,\infty}\cdot\|\SM\big(\bx_{\tau}^{\top}\bW_{QK}\bX^{\top}\big)\|_{1}\nonumber\\ 
        &\leq \|\bW_{V}^{\top}\|_{p,q}\cdot\|\bX^{\top}\|_{p,\infty}\label{ieq:76},
    \end{align}
    where the first inequality follows from Lemma~\ref{lem:matvec} with $u=\infty$ and $v=1$, and the last inequality follows from Lemma~\ref{lem:matmat}. The second term in the right-hand side of inequality \eqref{ieq:74} is
    \begin{align}
        \bigg\|\Big(\rff\big(\bX,\ba,\bb\big)\Big)^{\top}\bigg\|_{p,\infty}&=\max_{\tau\in[N]} \bigg\|\Big(\rff\big(\bx_{\tau},\ba,\bb\big)\Big)^{\top}\bigg\|_{p}\nonumber\\
        &=\max_{\tau\in[N]} \Big[\sum_{k=1}^{d}\bigg(\sum_{j=1}^{m}a_{kj}\relu(\bb_{kj}^{\top}\bx_{\tau})\bigg)^{p}\Big]^{1/p}\nonumber\\
        &\leq \max_{\tau\in[N]} \Big[\sum_{k=1}^{d}\bigg(\sum_{j=1}^{m}|a_{kj}|\cdot\|\bb_{kj}\|_{q}\cdot\|\bx_{\tau}\|_{p})\bigg)^{p}\Big]^{1/p}\nonumber\\ 
        &= \Big[\sum_{k=1}^{d}\bigg(\sum_{j=1}^{m}|a_{kj}|\cdot\|\bb_{kj}\|_{q}\cdot\|\bX^{\top}\|_{p,\infty})\bigg)^{p}\Big]^{1/p},\label{ieq:85}
    \end{align}
    where the inequality follows from H\"older's inequality and that $\relu(\cdot)$ is 1-Lipchitz. Combining inequalities \eqref{ieq:76} and \eqref{ieq:85}, we prove the desired result.
\end{proof}

\subsection{Technical Lemmas}

\begin{lemma}[Lemma 1 in \cite{xie2020q}]\label{lem:regretdecomp}
    For any policy $\pi\in\Pi$ and any function $f:\bar{\calS}\times\bar{\calA}\rightarrow \bbR$, we have
    \begin{align}
        f(\bar{\bS}_{0},\pi)-V_{P^{*}}^{\pi}(\bar{\bS}_{0})=\frac{\bbE_{d^{\pi}_{P^{*}}}\big[f(\bar{\bS},\bar{\bA})-r(\bar{\bS},\bar{\bA})-f(\bar{\bS},\pi)\big]}{1-\gamma}.
    \end{align}
\end{lemma}

\begin{lemma}[Lemma 10 in \cite{sun2019model}]\label{lem:simu}
    For any two transition kernels $P$ and $P^{\prime}$ and any policy $\pi\in\Pi$, we have
    \begin{align}
        \big|V_{P}^{\pi}(\bar{\bS}_{0})-V_{P^{\prime}}^{\pi}(\bar{\bS}_{0})\big|&\leq \frac{1}{1-\gamma}\Big|\bbE_{(\bar{\bS},\bar{\bA})\sim d_{P}^{\pi}}\big[\bbE_{\bar{\bS}^{\prime}\sim P(\cdot\,|\, \bar{\bS},\bar{\bA})}V_{P^{\prime}}^{\pi}(\bar{\bS}^{\prime})-\bbE_{\bar{\bS}^{\prime}\sim P^{\prime}(\cdot\,|\, \bar{\bS},\bar{\bA})}V_{P^{\prime}}^{\pi}(\bar{\bS}^{\prime})\big]\Big|\nonumber\\
        &\leq \frac{V_{\max}}{(1-\gamma)^{2}}\bbE_{(\bar{\bS},\bar{\bA})\sim d_{P}^{\pi}}\Big[\tv\big(\breve{P}(\cdot\,|\,\bar{\bS},\bar{\bA}), P^{\prime}(\cdot\,|\,\bar{\bS},\bar{\bA})\big)\Big].\nonumber
    \end{align}
\end{lemma}

\begin{lemma}\label{lem:vecnorm}
    For any $\bx\in\bbR^{d}$ and $0<p<q$, $\|\bx\|_{q}\leq \|\bx\|_{p}\leq d^{1/p-1/q}\|\bx\|_{q}$.
\end{lemma}
\begin{proof}[Proof of Lemma~\ref{lem:vecnorm}]
    $\|\bx\|_{q}\leq \|\bx\|_{p}$ simply follows from H\"older's inequality. For the right inequality, when $q<\infty$, we have
    \begin{align}
        \|\bx\|_{p}=\bigg(\sum_{i=1}^{d}|x_{i}|^{p}\bigg)^{1/p}\leq \biggl[\bigg(\sum_{i=1}^{d}\big(|x_{i}|^{p}\big)^{q/p}\bigg)^{p/q}\bigg(\sum_{i=1}^{d}1^{q/(q-p)}\bigg)^{1-p/q}\Bigg]^{1/p}=d^{1/p-1/q}\|\bx\|_{q},\nonumber
    \end{align}
    where the inequality follows from H\"older's inequality.
    When $q=\infty$, $\|\bx\|_{p}\leq d^{1/p}\|\bx\|_{\infty}$.
\end{proof}

\begin{lemma}\label{lem:matvec}
    Given any two conjugate numbers $u,v\in [1,\infty]$, i.e., $\frac{1}{u}+\frac{1}{v}=1$, and $1\leq p\leq \infty$, for any $\bA\in\bbR^{r\times c}$ and $\bx\in \bbR^{c}$, we have
    \begin{align}
        \|\bA\bx\|_{p}\leq \|\bA\|_{p,u}\|\bx\|_{v}\quad\mbox{and}\quad  \|\bA\bx\|_{p}\leq \|\bA^{\top}\|_{u,p}\|\bx\|_{v}\nonumber
    \end{align}
\end{lemma}
\begin{proof}[Proof of Lemma~\ref{lem:matvec}]
    To prove the first inequality, we write $\bA=[\ba_{1}\ldots \ba_{c}]$, where $\ba_{i}\in\bbR^{r}$ for $i\in[c]$. Then we have 
    \begin{align}
        \|\bA\bx\|_{p}=\bigg\|\sum_{i=1}^{c}\ba_{i}x_{i}\bigg\|_{p}\overset{(a)}{\leq} \sum_{i=1}^{c} |x_{i}|\|\ba_{i}\|_{p}\overset{(b)}{\leq} \|\bA\|_{p,u}\|\bx\|_{v},\nonumber
    \end{align}
    where inequality (a) comes from the triangle inequality, and inequality (b) comes from H\"older's inequality.

    To prove the second inequality, we write $\bA=[\ba_{1}^{\top}\ldots \ba_{r}^{\top}]^{\top}$, where $\ba_{i}\in\bbR^{c}$ for $i\in[r]$. Then we have 
    \begin{align}
        \|\bA\bx\|_{p}^{p}=\sum_{i=1}^{r} |\ba_{i}^{\top}\bx|^{p}\overset{(c)}{\leq} \sum_{i=1}^{r} \|\ba_{i}\|_{u}^{p}\|\bx\|_{v}^{p}=\|\bA\|_{p,u}^{p}\|\bx\|_{v}^{p},\nonumber
    \end{align}
    for $1\leq p<\infty$, where inequality $(c)$ follows from H\"older's inequality. When $p=\infty$, we have
    \begin{align}
        \|\bA\bx\|_{\infty}=\max_{i\in[r]} |\ba_{i}^{\top}\bx|\leq \max_{i\in[r]} \|\ba_{i}\|_{u}\|\bx\|_{v}=\|\bA\|_{\infty,u}\|\bx\|_{v}.\nonumber
    \end{align}
\end{proof}

\begin{lemma}\label{lem:matmat}
    Given any two conjugate numbers $p,q\in [1,\infty]$, i.e., $\frac{1}{p}+\frac{1}{q}=1$, for any $\bA\in\bbR^{r\times c}$ and $\bB\in \bbR^{c\times d}$, we have
    \begin{align}
        \|\bA\bB\|_{p,\infty}\leq \|\bA\|_{p,q}\|\bB\|_{p,\infty}.\nonumber
    \end{align}
\end{lemma}
\begin{proof}[Proof of Lemma~\ref{lem:matmat}]
    To prove the result, we write $\bB=[\bb_{1},\ldots,\bb_{d}]$, where $\bb_{i}\in\bbR^{c}$ for $i\in[d]$.
    \begin{align}
        \|\bA\bB\|_{p,\infty}=\max_{i\in[d]} \|\bA\bb_{i}\|_{p}\leq \max_{i\in[d]} \|\bA\|_{p,q}\|\bb_{i}\|_{p}=\|\bA\|_{p,q}\|\bB\|_{p,\infty},\nonumber
    \end{align}
    where the inequality follows from Lemma~\ref{lem:matvec}.
\end{proof}
\begin{lemma}\label{lem:smlip}
    For any $\bx,\by\in\bbR^{d}$, we have
    \begin{align}
        \|\SM(\bx)-\SM(\by)\|_{1}\leq 2\|\bx-\by\|_{\infty}.\nonumber
    \end{align}
\end{lemma}
\begin{proof}[Proof of Lemma~\ref{lem:smlip}]
The Jacobian matrix of the softmax function is
\begin{align}
    \frac{\rmd \SM(\bx)}{\rmd \bx}=\rm{diag}(\SM(\bx))-\SM(\bx)\SM(\bx)^{\top}.\nonumber
\end{align}
The $\ell_{1,1}$ norm of the Jacobian matrix can be bounded as
\begin{align}
    \Big\|\frac{\rmd \SM(\bx)}{\rmd \bx}\Big\|_{1,1}&=\sum_{i=1}^{d}\sum_{j=1}^{d}\bigg|\big[\SM(\bx)\big]_{i}\Big(\bbI_{i=j}-\big[\SM(\bx)\big]_{j}\Big)\bigg|\nonumber\\
    &=2\sum_{i=1}^{d}\big[\SM(\bx)\big]_{i}\Big(1-\big[\SM(\bx)\big]_{i}\Big)\nonumber\\
    &\leq 2\label{ieq:1}.
\end{align}
Then the $\ell_{1}$-norm of the difference between $\SM(\bx)$ and $\SM(\by)$ can be bounded as
\begin{align}
    \|\SM(\bx)-\SM(\by)\|_{1}&=\bigg\|\int_{0}^{1} \frac{\rmd \SM(\bz)}{\rmd \bz}|_{\bz=t\bx+(1-t)\by}(\by-\bx)\mathrm{d}t\bigg\|_{1}\nonumber\\
    &\leq \int_{0}^{1} \bigg\|\frac{\rmd \SM(\bz)}{\rmd \bz}|_{\bz=t\bx+(1-t)\by}(\by-\bx)\bigg\|_{1}\mathrm{d}t\nonumber\\ 
    &\leq \int_{0}^{1} \bigg\|\frac{\rmd \SM(\bz)}{\rmd \bz}|_{\bz=t\bx+(1-t)\by}\bigg\|_{1,1}\| \by-\bx \|_{\infty}\mathrm{d}t\nonumber\\ 
    &\leq \int_{0}^{1} 2\| \by-\bx \|_{\infty}\mathrm{d}t\nonumber\\ 
    &=2\|\by-\bx \|_{\infty},\nonumber
\end{align}
where the first inequality follows from triangle inequality, the second inequality follows from Lemma~\ref{lem:matvec} by setting $p=1$, $u=1$ and $v=\infty$, and the last inequality follows from inequality \eqref{ieq:1}. This concludes the proof.
\end{proof}

\section{Some Extensions}\label{app:ext}
\subsection{Extension to Multi-Head Attention}
Our results in Theorem~\ref{thm:main_concen} can be extended to the neural network with multi-head attention, which is defined as
\begin{align}
    f(\bX,\bW_{QK},\bW_{V})&=\SM\big(\bX\bW_{QK}\bX^{\top}\big)\bX\bW_{V},\nonumber\\
    \mha(\bX,\bW_{QK}^{1:h},\bW_{V}^{1:h},\bW_{O}^{1:h})&= \sum_{i=1}^{h}f(\bX,\bW_{QK,i},\bW_{V,i})\bW_{O,i},\nonumber
\end{align}
where $\bW_{QK,i}\in\bbR^{d\times d},\bW_{V,i}\in\bbR^{d\times\frac{d}{h}},\bW_{O,i}\in\bbR^{\frac{d}{h}\times d}$ for $i\in [h]$. Note that we only need to reprove the results in Propositions~\ref{prop:translip} and~\ref{prop:normbound} for the multi-head attention.

\begin{proposition}\label{prop:mhalip1}
    For any $\bX,\tilde{\bX}\in\bbR^{N\times d}$, and any $\bW_{QK,i}\in\bbR^{d\times d},\bW_{V,i}\in\bbR^{d\times\frac{d}{h}},\bW_{O,i}\in\bbR^{\frac{d}{h}\times d}$ for $i\in [h]$ and two positive conjugate numbers $p,q\in\bbR$, if $\|\bX^{\top}\|_{p,\infty},\|\tilde{\bX}^{\top}\|_{p,\infty}\leq B_{X}$, $\|\bW_{QK,i}^{\top}\|_{p,q}\leq B_{QK}$, $\|\bW_{V,i}^{\top}\|_{p,q}\leq B_{V}$, and $\|\bW_{O,i}^{\top}\|_{p,q}\leq B_{O}$ for $i\in[h]$, then we have 
    \begin{align}
        &\Big\|\big(\mha(\bX,\bW_{QK}^{1:h},\bW_{V}^{1:h},\bW_{O}^{1:h})-\mha(\tilde{\bX},\bW_{QK}^{1:h},\bW_{V}^{1:h},\bW_{O}^{1:h})\big)^{\top}\Big\|_{p,\infty}\nonumber\\
        &\quad\leq hB_{O}\cdot B_{V}\big(1+4c_{p,q}B_{X}^{2}\cdot B_{QK}\big)\|\bX^{\top}-\tilde{\bX}^{\top}\|_{p,\infty}.\nonumber
    \end{align}
\end{proposition}
\begin{proof}[Proof of Proposition~\ref{prop:mhalip1}]
    For the difference between the outputs of the multi-head attention with different inputs, we have
    \begin{align}
        &\Big\|\big(\mha(\bX,\bW_{QK}^{1:h},\bW_{V}^{1:h},\bW_{O}^{1:h})-\mha(\tilde{\bX},\bW_{QK}^{1:h},\bW_{V}^{1:h},\bW_{O}^{1:h})\big)^{\top}\Big\|_{p,\infty}\nonumber\\
        &\quad\leq \sum_{i=1}^{h}\bigg\|\Big(f(\bX,\bW_{QK,i},\bW_{V,i})\bW_{O,i}-f(\tilde{\bX},\bW_{QK,i},\bW_{V,i})\bW_{O,i})\Big)^{\top}\bigg\|_{p,\infty}\nonumber\\ 
        &\quad\leq \sum_{i=1}^{h}\|\bW_{O,i}^{\top}\|_{p,q}\cdot\Big\|\Big(f(\bX,\bW_{QK,i},\bW_{V,i})-f(\tilde{\bX},\bW_{QK,i},\bW_{V,i})\Big)^{\top}\Big\|_{p,\infty}\nonumber\\ 
        &\quad\leq \sum_{i=1}^{h}\|\bW_{O,i}^{\top}\|_{p,q}\cdot\|\bW_{V,i}^{\top}\|_{p,q}\Big(1+2c_{p,q}\|\tilde{\bX}^{\top}\|_{p,\infty}\cdot\|\bW_{QK,i}^{\top}\|_{p,q}\big(\|\tilde{\bX}^{\top}\|_{p,\infty}\nonumber\\
        &\quad\qquad+\|\bX^{\top}\|_{p,\infty}\big)\Big)\|\bX^{\top}-\tilde{\bX}^{\top}\|_{p,\infty}\nonumber, 
    \end{align}
    where the first inequality follows from triangle inequality, the second inequality follows from Lemma~\ref{lem:matmat}, and the last inequality follows from Proposition~\ref{prop:translip}.
\end{proof}
\begin{proposition}\label{prop:mhalip2}
    For any $\bX\in\bbR^{N\times d}$, and any $\bW_{QK,i},\tilde{\bW}_{QK,i}\in\bbR^{d\times d},\bW_{V,i},\tilde{\bW}_{V,i}\in\bbR^{d\times\frac{d}{h}},\bW_{O,i},\tilde{\bW}_{O,i}\in\bbR^{\frac{d}{h}\times d}$ for $i\in [h]$ and two positive conjugate numbers $p,q\in\bbR$, if $\|\bX^{\top}\|_{p,\infty}\leq B_{X}$, $\|\bW_{V,i}^{\top}\|_{p,q},\|\tilde{\bW}_{V,i}^{\top}\|_{p,q}\leq B_{V}$, and $\|\bW_{O,i}^{\top}\|_{p,q},\|\tilde{\bW}_{O,i}^{\top}\|_{p,q}\leq B_{O}$ for $i\in[h]$, then we have 
    \begin{align}
        &\Big\|\big(\mha(\bX,\tilde{\bW}_{QK}^{1:h},\tilde{\bW}_{V}^{1:h},\tilde{\bW}_{O}^{1:h})-\mha(\bX,\bW_{QK}^{1:h},\bW_{V}^{1:h},\bW_{O}^{1:h})\big)^{\top}\Big\|_{p,\infty}\nonumber\\
        &\quad\leq \sum_{i=1}^{h} B_{V}\cdot B_{X}\big\|(\tilde{\bW}_{O,i}-\bW_{O,i})^{\top}\big\|_{p,q}+B_{O}\cdot B_{X}\big\|\bW_{V,i}^{\top}-\tilde{\bW}_{V,i}^{\top}\|_{p,q}\nonumber\\
        &\quad\qquad +2c_{p,q}B_{X}^{3}\cdot B_{V}\cdot B_{O}\|\bW_{QK,i}^{\top}-\tilde{\bW}_{QK,i}^{\top}\|_{p,q}\nonumber
    \end{align}
\end{proposition}
\begin{proof}[Proof of Proposition~\ref{prop:mhalip2}]
    For the difference between the outputs of the multi-head attention with different parameters, we have
    \begin{align}
        &\Big\|\big(\mha(\bX,\tilde{\bW}_{QK}^{1:h},\tilde{\bW}_{V}^{1:h},\tilde{\bW}_{O}^{1:h})-\mha(\bX,\bW_{QK}^{1:h},\bW_{V}^{1:h},\bW_{O}^{1:h})\big)^{\top}\Big\|_{p,\infty}\nonumber\\
        &\quad=\bigg\|\Big(\sum_{i=1}^{h}f(\bX,\tilde{\bW}_{QK,i},\tilde{\bW}_{V,i})\tilde{\bW}_{O,i}-\sum_{i=1}^{h}f(\bX,\bW_{QK,i},\bW_{V,i})\bW_{O,i})\Big)^{\top}\bigg\|_{p,\infty}\nonumber\\
        &\quad\leq \bigg\|\Big(\sum_{i=1}^{h}f(\bX,\tilde{\bW}_{QK,i},\tilde{\bW}_{V,i})\tilde{\bW}_{O,i}-\sum_{i=1}^{h}f(\bX,\tilde{\bW}_{QK,i},\tilde{\bW}_{V,i})\bW_{O,i})\Big)^{\top}\bigg\|_{p,\infty}\nonumber\\
        &\quad\qquad +\bigg\|\Big(\sum_{i=1}^{h}f(\bX,\tilde{\bW}_{QK,i},\tilde{\bW}_{V,i})\bW_{O,i}-\sum_{i=1}^{h}f(\bX,\bW_{QK,i},\bW_{V,i})\bW_{O,i})\Big)^{\top}\bigg\|_{p,\infty}\nonumber\\
        &\quad\leq \sum_{i=1}^{h}\Big\|\big(f(\bX,\tilde{\bW}_{QK,i},\tilde{\bW}_{V,i})\big)^{\top}\Big\|_{p,\infty}\cdot\big\|(\tilde{\bW}_{O,i}-\bW_{O,i})^{\top}\big\|_{p,q}\nonumber\\
        &\quad\qquad +\sum_{i=1}^{h}\bigg\|\Big(f(\bX,\tilde{\bW}_{QK,i},\tilde{\bW}_{V,i})-f(\bX,\bW_{QK,i},\bW_{V,i})\Big)^{\top}\bigg\|_{p,\infty}\cdot\big\|\bW_{O,i}^{\top}\big\|_{p,q}\label{ieq:45},
    \end{align}
    where the first inequality follows from triangle inequality, and the second inequality follows from Lemma~\ref{lem:matmat}.
    
    For the first term in inequality \eqref{ieq:45}, let $\tau\in[N]$ and $\bx_{\tau}^{\top}$ be the $\tau^{\rm{th}}$ row of $\bX$, then we have
    \begin{align}
        \Big\|\big(f(\bX,\tilde{\bW}_{QK,i},\tilde{\bW}_{V,i})\big)^{\top}\Big\|_{p,\infty}&=\max_{\tau\in[N]} \|\SM\big(\bx_{\tau}^{\top}\tilde{\bW}_{QK,i}\bX^{\top}\big)\bX\tilde{\bW}_{V,i}\|\nonumber\\
        &\leq \max_{\tau\in[N]} \big\|\tilde{\bW}_{V,i}^{\top}\bX^{\top}\big\|_{p,\infty}\cdot\Big\|\SM\big(\bx_{\tau}^{\top}\tilde{\bW}_{QK,i}\bX^{\top}\big)\Big\|_{1}\nonumber\\ 
        &\leq \big\|\tilde{\bW}_{V,i}^{\top}\big\|_{p,q}\cdot\|\bX^{\top}\big\|_{p,\infty}\label{ieq:47},
    \end{align}
    where the first inequality follows from Lemma~\ref{lem:matvec}. For the second term in inequality \eqref{ieq:45}, recall Proposition~\ref{prop:translip}, then we have
    \begin{align}
        &\bigg\|\Big(f(\bX,\tilde{\bW}_{QK,i},\tilde{\bW}_{V,i})-f(\bX,\bW_{QK,i},\bW_{V,i})\Big)^{\top}\bigg\|_{p,\infty}\nonumber\\
        &\quad\leq 2c_{p,q}\|\bX^{\top}\|_{p,\infty}^{3}\cdot\|\bW_{V,i}^{\top}\|_{p,q}\cdot\|\bW_{QK,i}^{\top}-\tilde{\bW}_{QK,i}^{\top}\|_{p,q}+\big\|\bW_{V,i}^{\top}-\tilde{\bW}_{V,i}^{\top}\|_{p,q}\cdot\|\bX^{\top}\big\|_{p,\infty}\label{ieq:48}.
    \end{align}
    The desired result follows by substituting inequalities \eqref{ieq:47} and \eqref{ieq:48} into inequality \eqref{ieq:45}. This concludes the proof.
\end{proof}
\begin{proposition}\label{prop:mhanormbound}
    For any $\bX\in\bbR^{N\times d}$, and any $\bW_{QK,i}\in\bbR^{d\times d},\bW_{V,i}\in\bbR^{d\times\frac{d}{h}},\bW_{O,i}\in\bbR^{\frac{d}{h}\times d}$ for $i\in [h]$ and two positive conjugate numbers $p,q\in\bbR$, we have
    \begin{align}
        \Big\|\big(\mha(\bX,\bW_{QK}^{1:h},\bW_{V}^{1:h},\bW_{O}^{1:h})\big)^{\top}\Big\|_{p,\infty}\leq \sum_{i=1}^{h}\|\bW_{O,i}^{\top}\|_{p,q}\|\bW_{V,i}^{\top}\|_{p,q}\|\bX^{\top}\|_{p,\infty}\nonumber.
    \end{align}
\end{proposition}
\begin{proof}[Proof of Proposition~\ref{prop:mhanormbound}]
    For the $\ell_{p,\infty}$-norm of the multi-head attention, we have
    \begin{align}
        &\Big\|\big(\mha(\bX,\bW_{QK}^{1:h},\bW_{V}^{1:h},\bW_{O}^{1:h})\big)^{\top}\Big\|_{p,\infty}\nonumber\\
        &\quad\leq \sum_{i=1}^{h}\bigg\|\Big(f(\bX,\bW_{QK,i},\bW_{V,i})\bW_{O,i}\Big)^{\top}\bigg\|_{p,\infty}\nonumber\\ 
        &\quad\leq \sum_{i=1}^{h}\|\bW_{O,i}^{\top}\|_{p,q}\cdot\Big\|\big(f(\bX,\bW_{QK,i},\bW_{V,i})\big)^{\top}\Big\|_{p,\infty}\nonumber\\ 
        &\quad\leq \sum_{i=1}^{h}\|\bW_{O,i}^{\top}\|_{p,q}\cdot\|\bW_{V,i}^{\top}\|_{p,q}\cdot\|\bX^{\top}\|_{p,\infty},\nonumber 
    \end{align}
    where the first inequality follows from triangle inequality, the second inequality follows from Lemma~\ref{lem:matmat}, and the final inequality follows from inequality \eqref{ieq:76} in Proposition~\ref{prop:normbound}.
\end{proof}

\subsection{Extension to Non-i.i.d.\ Sampling}\label{app:extnoniid}
The dataset $\calD$ is collected in an i.i.d.\ manner in the main paper. In this this section, we extend our result to the non-i.i.d.\ case. Specifically, we collect the dataset $\calD^{\prime}=\{(\bar{\bS}_{t},\bar{\bA}_{t},r_{t})\}_{t=0}^{n}$ by 
implementing a policy $\pi_{0}$, i.e., the action is taken as $\bar{\bA}_{t}\sim\pi_{0}(\cdot\,|\, \bar{\bS}_{t})$, and the sequence of states is updated  as $\bar{\bS}_{t+1}\sim P^{*}(\cdot\, | \,\bar{\bS}_{t},\bar{\bA}_{t})$ for $t\in[n]$. We assume that 
the initial state $\bar{\bS}_{0}$ is generated according to a distribution $q_{0}$, i.e., the initial state-action pair is distributed as $(\bar{\bS}_{0},\bar{\bA}_{0})\sim q_{0}\pi_{0}$. We denote the \emph{stationary distribution} 
on the state-action pair of the Markov chain induced by the policy $\pi_{0}$ as $q_{P^{*}}^{\pi_{0}}(\bar{\bS},\bar{\bA})$. Note that the initial distribution $q_{0}\pi_{0}$ may not equal to the stationary distribution $q_{P^{*}}^{\pi_{0}}$. 
To distinguish these two different cases, we will use $P_{q_{0}\pi_{0}}$ and $P_{q_{P^{*}}^{\pi_{0}}}$ to denote the probability distributions with respect to the Markov chains with initial state distributed as $q_{0}\pi_{0}$ and $q_{P^{*}}^{\pi_{0}}$ respectively.

In such setting, we define the mismatch between two functions $f$ and $\tilde{f}$ on $\calD$ for a fixed policy $\pi$ as $\calL^{\prime}(f,\tilde{f},\pi;\calD^{\prime})= \frac{1}{n}\sum_{t=0}^{n-1}(f(\bar{\bS}_{t},\bar{\bA}_{t})-\barr_{t}-\gamma \tilde{f}(\bar{\bS}_{t+1},\pi))^{2}$, then the Bellman error of a function $f$ with respect to  the policy $\pi$ is defined as $\calE^{\prime}(f,\pi;\calD^{\prime})=\calL^{\prime}(f,f,\pi;\calD^{\prime})-\inf_{\tilde{f}\in\calF_{\tf}}\calL^{\prime}(\tilde{f},f,\pi;\calD^{\prime})$. The corresponding model-free algorithm can be written as
\begin{align}
    \hpi^{\prime}=\argmax_{\pi\in\Pi}\min_{f\in \calF^{\prime}(\pi,\varepsilon)} f(\bar{\bS}_{0},\pi),\quad \text{where}\quad \calF^{\prime}(\pi,\varepsilon)= \big\{f\in\calF_{\tf}(B) \, \big|\, \calE^{\prime}(f,\pi;\calD^{\prime})\leq \varepsilon \big\}.\label{algo:mfree2}
\end{align}
In the dataset $\calD$ collected by implementing policy $\pi_{0}$, the mismatch between the distribution induced by the optimal policy $d_{P^{*}}^{\pi^{*}}$ and the stationary distribution $q_{P^{*}}^{\pi_{0}}$ is captured by
\begin{align}
    C_{\calF_{\tf}}^{\prime}(\pi_{0})=\max_{f\in\calF_{\tf}}\bbE_{d^{\pi^{*}}_{P^{*}}}\big[\big(f(\bar{\bS},\bar{\bA})-\calT^{\pi^{*}}f(\bar{\bS},\bar{\bA})\big)^{2}\big]\big/\bbE_{q_{P^{*}}^{\pi_{0}}}\big[\big(f(\bar{\bS},\bar{\bA})-\calT^{\pi^{*}}f(\bar{\bS},\bar{\bA})\big)^{2}\big],\label{eqn:nCFtf}
\end{align}
where $\calF_{\tf}$ is the transformer function class defined in Section~\ref{sec:mfreeRL}.

To analyze the concentration behavior of the action-value function estimate under such sampling method, we need to define additional quantities to describe how fast the Markov chain approximates its stationary distribution. For a Markov chain with finite state space $\Omega$ and transition probability matrix $P$, we label the eigenvalues of $P$ in decreasing order: $1=\lambda_{1}\geq\ldots\geq \lambda_{|\Omega|}\geq -1$. Define $\lambda^{*}=\max \{|\lambda|\, : \, \lambda \text{ is an eigenvalue of }P\text{ and }\lambda\neq1\}.$ The \emph{absolute spectral gap} of $P$ is defined as $1-\lambda^{*}$. The notion of the absolute spectral gap and our following results can also be generalized to the Markov chain with infinite state space by treating of transition kernel $P$ as an operator of a Hilbert space. For two distributions $p$ and $q$ on $\Omega$, we define 
\begin{align*}
    N(p,q)=\int_{\Omega} \frac{\rmd p}{\rmd q}(x) \, p(\rmd x).
\end{align*}
Inspired by the ubiquitous change-of-measure technique, we will use $N(q_{0},q)$ to capture the difference between the non-stationary Markov chain with initial distribution $q_{0}$ and the stationary Markov chain with stationary distribution $q$.

To analyze the algorithm in Eqn.~\eqref{algo:mfree2}, we first derive a generalization error bound of the estimate of the Bellman error using the PAC-Bayesian framework.
\begin{restatable}{proposition}{mainconcennoniid}\label{thm:main_concen2}
    Consider the dataset $\calD^{\prime}$ collected by implementing a policy $\pi_{0}$. Let $\bar{B}=B_{V}B_{QK}B_{a}B_{b}B_{\bw}$.  For all $f,\tilde{f}\in\calF_{\tf}(\bB)$ and all policies $\pi\in\Pi$, with probability at least $1-\delta$, we have 
    \begin{align}
        &\Big|\bbE_{q_{P^{*}}^{\pi_{0}}}\Big[\big(f(\bar{\bS},\bar{\bA})-\calT^{\pi}\tilde{f}(\bar{\bS},\bar{\bA})\big)^{2}\Big]-\calL^{\prime}(f,\tilde{f},\pi;\calD^{\prime})+\calL^{\prime}(\calT^{\pi}\tilde{f},\tilde{f},\pi;\calD^{\prime})\Big|\nonumber\\
        &\quad\leq \!\frac{C+(2-C)\lambda}{2}\bbE_{q_{P^{*}}^{\pi_{0}}}\!\Big[\big(f(\bar{\bS},\bar{\bA})\!-\!\calT^{\pi}\tilde{f}(\bar{\bS},\bar{\bA})\big)^{2}\Big]\!\nonumber\\
        &\quad\qquad+\!O\bigg(\frac{V_{\max}^{2}}{(1-\lambda)n}\biggl[mL^{2}d^{2}\log\frac{mdL\bar{B} n}{V_{\max}} +\log\frac{N(q_{0}\pi_{0},q_{P^{*}}^{\pi_{0}})\calN(\Pi,1/n,d_{\infty})}{\delta}\biggr]\bigg),\label{eqn:gen_appendix}
    \end{align}
    where $1-\lambda$ is the absolute spectral gap of the Markov chain $\{(\bar{\bS}_{t},\bar{\bA}_{t})\}_{t=0}^{\infty}$ induced by the policy $\pi_{0}$, and $0<C<e^{1/10}$ is an absolute constant.
\end{restatable}
For ease of notation, we define $\tilde{e}(\calF_{\tf},\Pi,\pi_{0},\delta,n)$ to be $(1-\lambda)n$ times the second term of the generalization error bound in \eqref{eqn:gen_appendix}. We note that Proposition~\ref{thm:main_concen2} is a generalization  of Theorem~\ref{thm:main_concen}. When the dataset $\calD$ consists of i.i.d.\ samples drawn according to $\mu$, the dataset $\calD$ can be treated as a Markov chain with $\lambda=0$, and $N(\mu,\mu)=1$. In this case, our result in Proposition~\ref{thm:main_concen2} particularizes to the result in  Theorem~\ref{thm:main_concen} up to a constant.

Before stating the suboptimality bound, we require two additional assumptions on the function class and the policy $\pi_{0}$. We first state the standard regularity assumption of the transformer function class. We assume that the collected dataset $\calD^{\prime}$ provides a good coverage of the optimal policy.
\begin{assumption}\label{assump:mfree3}
    For the policy $\pi_{0}$, the coefficient $C_{\calF_{\tf}}^{\prime}(\pi_{0})$ defined in Eqn.~\eqref{eqn:nCFtf} is finite. 
\end{assumption}
Correspondingly, we slightly adjust the approximate realizability and complete assumption as follows:
\begin{assumption}\label{assump:mfree4}
    For any $\pi\in\Pi$, we have $\inf_{f\in\calF_{\tf}}\sup_{\mu\in q_{\Pi}} \bbE_{\mu}[(f(\bar{\bS},\bar{\bA})-\calT^{\pi}f(\bar{\bS},\bar{\bA}))^{2}]\leq \varepsilon_{\calF}^{\prime}$ and $\sup_{f\in\calF_{\tf}}\inf_{\tilde{f}\in\calF_{\tf}} \bbE_{q_{P^{*}}^{\pi_{0}}}[(\tilde{f}(\bar{\bS},\bar{\bA})-\calT^{\pi}f(\bar{\bS},\bar{\bA}))^{2}]\leq \varepsilon_{\calF,\calF}^{\prime}$, where $q_{\Pi}=\{\mu\ | \ \exists\,  \pi\in \Pi \text{ s.t. }\mu=q^{\pi}_{P^{*}}\}$ is the set of stationary distributions of the state and the action pair induced by any policy $\pi\in\Pi$.
\end{assumption}
Then the suboptimality gap of the learned policy can be upper bounded as follows.
\begin{restatable}{theorem}{mfreethmnoniid}\label{thm:main_mfree2}
    If Assumptions~\ref{assump:mfree3} and~\ref{assump:mfree4} hold, and we take $\varepsilon=[2+C+(2-C)\lambda]\varepsilon_{\calF}^{\prime}/2 +2\tilde{e}(\calF_{\tf},\Pi,\pi_{0},\delta,n)/[(1-\lambda)n]$, then with probability at least $1-\delta$, the suboptimality gap of the policy derived in the algorithm shown in Eqn.~\eqref{algo:mfree2} is upper bounded as 
    \begin{align}
        &\!V_{P^{*}}^{\pi^{*}}(\bar{\bS}_{0})\!-\!V_{P^{*}}^{\hpi}(\bar{\bS}_{0})\! \leq \! O\Bigg(\sqrt{\frac{C_{\calF_{\tf}}^{\prime}(\pi_{0})\tilde{\varepsilon}}{(1-\gamma)^{2}(1-\lambda)}}\!\nonumber\\
        &\qquad\quad \!+\frac{V_{\max}\sqrt{ C_{\calF_{\tf}}^{\prime}(\pi_{0})  }}{(1-\gamma)(1-\lambda)\sqrt{n} }\sqrt{mL^{2}d^{2}\log\frac{mdL\bar{B}n}{V_{\max}}\!+\!\log \frac{2N(q_{0}\pi_{0},q_{P^{*}}^{\pi_{0}})\calN(\Pi,1/n,d_{\infty})}{\delta} }\Bigg),\nonumber
    \end{align}
    where $d=d_{\calS}+d_{\calA}$, $\tilde{\varepsilon} = \varepsilon_{\calF}^{\prime} + \varepsilon_{\calF,\calF}^{\prime}$, $\bar{B}$ is defined in Proposition \ref{thm:main_concen2}, $0<C<e^{1/10}$ is an absolute constant, and $1-\lambda$ is the absolute spectral gap of the Markov chain $\{(\bar{\bS}_{t},\bar{\bA}_{t})\}_{t=0}^{\infty}$ induced by the policy $\pi_{0}$.
\end{restatable}
We note that Theorem~\ref{thm:main_mfree2} is a generalization of Theorem~\ref{thm:main_mfree}. Sampling in an i.i.d.\  manner according to $\mu$ can be regarded as a Markov chain with $\lambda=0$, and $N(\mu,\mu)=1$. In this case, our result in Theorem~\ref{thm:main_mfree2} particularizes to the result in  Theorem~\ref{thm:main_mfree}.

\begin{proof}[Proof of Theorem~\ref{thm:main_mfree2}]
    The proof follows along similar lines as that of Theorem~\ref{thm:main_mfree}. Recall the definition below Proposition~\ref{thm:main_concen2}, i.e., 
    \begin{align}
        \tilde{e}(\calF_{\tf},\Pi,\pi_{0},\delta,n)&=C^{\prime}V_{\max}^{2}\biggl[mL^{2}d^{2}\log\frac{mdL\bar{B} n}{V_{\max}} +\log\frac{N(q_{0}\pi_{0},q_{P^{*}}^{\pi_{0}})\calN(\Pi,1/n,d_{\infty})}{\delta}\biggr],\nonumber
    \end{align}
    where $C^{\prime}>0$ is an absolute constant.
    To simplify the proof, we define 
    \begin{align}
        f_{\pi^{*}}^{*}&=\arginf_{f\in\calF_{\tf}}\sup_{\mu\in q_{\Pi}} \bbE_{\mu}\Bigl[(f(\bar{\bS},\bar{\bA})-\calT^{\pi^{*}}f\bigl(\bar{\bS},\bar{\bA})\bigr)^{2}\Bigr],\nonumber\\
        \varepsilon&=\frac{2+C+(2-C)\lambda}{2} \varepsilon_{\calF}^{\prime}+\frac{2\tilde{e}(\calF_{\tf},\Pi,\pi_{0},\delta,n)}{(1-\lambda)n},\nonumber
    \end{align}
    where $0<C<e^{1/10}$ is an absolute constant.
    
    Our proof can be decomposed into three main parts.
    \begin{itemize}
        \item Since $f_{\pi^{*}}^{*}$ is the best approximation of action-value function of the optimal policy $\pi^{*}$, we expect that it should belong to the confidence region of the action-value functions $\calF^{\prime}(\pi^{*},\varepsilon)$ with high probability. We show this in Step 1. 
        \item For any $\pi\in\Pi$ and any $f\in\calF^{\prime}(\pi,\varepsilon)$, since the empirical Bellman error is bounded $\calE^{\prime}(f,\pi;\calD^{\prime})\leq \varepsilon$, we expect that the population Bellman error $\bbE_{q_{P^{*}}^{\pi_{0}}}[(f(\bar{\bS},\bar{\bA})-\calT^{\pi}f(\bar{\bS},\bar{\bA}))^{2}]$ can be controlled with high probability, which implies that $f$ is a reliable estimate of the action-value function of $\pi$. We show this in Step 2.
        \item The suboptimality gap of the learned policy according to the reliable action-value function estimate can be bounded using the estimation error bound. We do this in Step 3. 
    \end{itemize}
    
    We lay out the proof by the three steps as stated in the above proof sketch.

    \textbf{Step 1: Show that $f_{\pi^{*}}^{*}\in\calF^{\prime}(\pi^{*},\varepsilon)$ with high probability.}
    
    From the definition of $f_{\pi^{*}}^{*}$ and Assumption~\ref{assump:mfree4}, we note that the population Bellman error of $f_{\pi^{*}}^{*}$ with respect to $\pi^{*}$ is bounded by $\varepsilon_{\calF}^{\prime}$. To bound the empirical Bellman error $\calE^{\prime}(f_{\pi^{*}}^{*},\pi^{*};\calD^{\prime})$ of $f_{\pi^{*}}^{*}$, we utilize the generalization error bound of the action-value function with the transformer function class.
    \mainconcennoniid*
    \begin{proof}
        See Appendix~\ref{app:proofmain_concen2} for a detailed proof.
    \end{proof}
    We can decompose the empirical Bellman error $\calE^{\prime}(f_{\pi^{*}}^{*},\pi^{*};\calD^{\prime})$ as the sum of the population Bellman error and the generalization error, where the population Bellman error can be controlled with $\varepsilon_{\calF}^{\prime}$ according to Assumption~\ref{assump:mfree4}, and the generalization error can be controlled with Proposition~\ref{thm:main_concen2}. Thus, we have the following lemma.
    \begin{lemma}\label{lem:errcontrol3}
        For any $\pi\in\Pi$, let $f_{\pi}^{*}=\arginf_{f\in\calF_{\tf}}\sup_{\mu\in q_{\Pi}} \bbE_{\mu}[(f(\bar{\bS},\bar{\bA})-\calT^{\pi}f(\bar{\bS},\bar{\bA}))^{2}]$. If Assumption~\ref{assump:mfree4} holds, the following inequality holds with probability at least $1-\delta$,
        \begin{align}
            \calE^{\prime}(f_{\pi}^{*},\pi;\calD^{\prime})\leq \frac{2+C+(2-C)\lambda}{2} \varepsilon_{\calF}^{\prime}+\frac{2\tilde{e}(\calF_{\tf},\Pi,\pi_{0},\delta,n)}{(1-\lambda)n}.\nonumber
        \end{align}
    \end{lemma}
    \begin{proof}
        The proof is same as the proof of Lemma~\ref{lem:errcontrol1} except using the concentration inequality in Proposition~\ref{thm:main_concen2}.
    \end{proof}

    \textbf{Step 2: For any policy $\pi\in\Pi$ and $f\in\calF^{\prime}(\pi,\varepsilon)$, show $\bbE_{q_{P^{*}}^{\pi_{0}}}[(f(\bar{\bS},\bar{\bA})-\calT^{\pi}f(\bar{\bS},\bar{\bA}))^{2}]$ is small with high probability.}
    
    To prove the desired result, we relate the population Bellman error $\bbE_{q_{P^{*}}^{\pi_{0}}}[(f(\bar{\bS},\bar{\bA})-\calT^{\pi}f(\bar{\bS},\bar{\bA}))^{2}]$ with $\calE^{\prime}(f,\pi;\calD^{\prime})$ using Proposition~\ref{thm:main_concen2}, where we bound the population Bellman error as the difference between the empirical Bellman error and the generalization error. Thus, we have the following lemma.
    \begin{lemma}\label{lem:errcontrol4}
        For any $\pi\in\Pi$ and $f\in\calF_{\tf}$, if $\calE^{\prime}(f,\pi;\calD^{\prime})\leq \varepsilon$ for some $\varepsilon>0$, and Assumption~\ref{assump:mfree4} holds, the following inequality holds with probability at least $1-\delta$,
        \begin{align}
            &\bbE_{q_{P^{*}}^{\pi_{0}}}\Big[\big(f(\bar{\bS},\bar{\bA})-\calT^{\pi}f(\bar{\bS},\bar{\bA})\big)^{2}\Big]\nonumber\\
            &\quad\leq \frac{2}{(2-C)(1-\lambda)}\varepsilon+\frac{2+C+(2-C)\lambda}{(2-C)(1-\lambda)}\varepsilon_{\calF,\calF}^{\prime}+\frac{4\tilde{e}(\calF_{\tf},\Pi,\pi_{0},\delta,n)}{(2-C)(1-\lambda)^{2}n}.\nonumber
        \end{align}
    \end{lemma}
    \begin{proof}
        The proof is same as the proof of Lemma~\ref{lem:errcontrol2} except using the cencentration inequality in Proposition~\ref{thm:main_concen2}.
    \end{proof}

    \textbf{Step 3: Bound the suboptimality gap of the learned policy with the population Bellman error bound in Step 2.}
    
    We define 
    \begin{align}
        \hat{f}_{\pi^{*}}&=\argmax_{f\in\calF^{\prime}(\pi^{*},\varepsilon)}f(\bar{\bS}_{0},\pi^{*}),\nonumber\\
        \breve{f}_{\pi^{*}}&=\argmin_{f\in\calF^{\prime}(\pi^{*},\varepsilon)}f(\bar{\bS}_{0},\pi^{*}),\nonumber
    \end{align}
    Following the same procedures in step 3 of the proof of Theorem~\ref{thm:main_mfree}, we can show that
    \begin{align}
        V_{P^{*}}^{\pi^{*}}(\bar{\bS}_{0})-V_{P^{*}}^{\hpi}(\bar{\bS}_{0})\leq\hat{f}_{\pi^{*}}(\bar{\bS}_{0},\pi^{*})-V_{P^{*}}^{\pi^{*}}(\bar{\bS}_{0})+V_{P^{*}}^{\pi^{*}}(\bar{\bS}_{0})-\breve{f}_{\pi^{*}}(\bar{\bS}_{0},\pi^{*})+\frac{2\sqrt{\varepsilon_{\calF}^{\prime}}}{1-\gamma}.\label{ieq:141}
    \end{align}
    Applying the suboptimality gap decomposition in Lemma~\ref{lem:regretdecomp} to inequality \eqref{ieq:141}, we have 
    \begin{align}
        &V_{P^{*}}^{\pi^{*}}(\bar{\bS}_{0})-V_{P^{*}}^{\hpi}(\bar{\bS}_{0})\nonumber\\
        &\quad\leq\frac{1}{1-\gamma}\Big\{\bbE_{d^{\pi^{*}}_{P^{*}}}\big[\hat{f}_{\pi^{*}}(\bar{\bS},\bar{\bA})-\calT^{\pi^{*}}\hat{f}_{\pi^{*}}(\bar{\bS},\bar{\bA})\big]\nonumber\\
        &\quad\qquad-\bbE_{d^{\pi^{*}}_{P^{*}}}\big[\breve{f}_{\pi^{*}}(\bar{\bS},\bar{\bA})-\calT^{\pi^{*}}\breve{f}_{\pi^{*}}(\bar{\bS},\bar{\bA})\big]\Big\}+\frac{2\sqrt{\varepsilon_{\calF}^{\prime}}}{1-\gamma}\nonumber\\
        &\quad\leq \frac{1}{1-\gamma}\bigg\{\sqrt{C_{\calF_{\tf}}^{\prime}(\pi_{0})\bbE_{q_{P^{*}}^{\pi_{0}}}\Big[\big(\hat{f}_{\pi^{*}}(\bar{\bS},\bar{\bA})-\calT^{\pi^{*}}\hat{f}_{\pi^{*}}(\bar{\bS},\bar{\bA})\big)^{2}\Big]}\nonumber\\
        &\quad\qquad+\sqrt{C_{\calF_{\tf}}^{\prime}(\pi_{0})\bbE_{q_{P^{*}}^{\pi_{0}}}\Big[\big(\breve{f}_{\pi^{*}}(\bar{\bS},\bar{\bA})-\calT^{\pi^{*}}\breve{f}_{\pi^{*}}(\bar{\bS},\bar{\bA})\big)^{2}\Big]}\bigg\}+\frac{2\sqrt{\varepsilon_{\calF}^{\prime}}}{1-\gamma}\nonumber,
    \end{align}
    where the first inequality follows from Lemma~\ref{lem:regretdecomp}, and the second inequality follows from Jensen's inequality and the definition of $C_{\calF_{\tf}}^{\prime}(\pi_{0})$. Combined with the result in Step 2, we have 
    \begin{align}
        &V_{P^{*}}^{\pi^{*}}(\bar{\bS}_{0})-V_{P^{*}}^{\hpi}(\bar{\bS}_{0})\nonumber\\
        &\quad\leq \frac{\sqrt{C_{\calF_{\tf}}^{\prime}(\pi_{0})}}{1-\gamma}\sqrt{\frac{2}{(2-C)(1-\lambda)}\varepsilon+\frac{2+C+(2-C)\lambda}{(2-C)(1-\lambda)}\varepsilon_{\calF,\calF}^{\prime}+\frac{4\tilde{e}(\calF_{\tf},\Pi,\pi_{0},\delta,n)}{(2-C)(1-\lambda)^{2}n}}+\frac{2\sqrt{\varepsilon_{\calF}}}{1-\gamma}\nonumber\\ 
        &\quad\leq O\biggl(\sqrt{\frac{C_{\calF_{\tf}}^{\prime}(\pi_{0})(\varepsilon_{\calF}^{\prime}+\varepsilon_{\calF,\calF}^{\prime})}{(1-\gamma)^{2}(1-\lambda)}}+\frac{\sqrt{C_{\calF_{\tf}}^{\prime}(\pi_{0})}}{(1-\gamma)(1-\lambda)}\sqrt{\frac{\tilde{e}(\calF_{\tf},\Pi,\pi_{0},\delta,n)}{n}}\biggr).\nonumber
    \end{align}
    Therefore, we conclude the proof of Theorem \ref{thm:main_mfree2}.
\end{proof}

\subsection{Proofs of Supporting Propositions in Section~\ref{app:extnoniid}}\label{app:extproof}
\subsubsection{Proof of Proposition~\ref{thm:main_concen2}}\label{app:proofmain_concen2}
\begin{proof}[Proof of Proposition~\ref{thm:main_concen2}]
    Similar to the proof of Theorem~\ref{thm:main_concen}, we adopt a PAC-Bayesian framework to derive our desired generalization error bound. We first state a preliminary result.
    \begin{proposition}\label{prop:pacbayes2}
        Let $\{X_{i}\}_{i\geq 1}$ be a Markov chain with state space $\Omega$, stationary distribution $q$, initial distribution $X_{1}\sim q_{0}$, and absolute spectral gap $1-\lambda$. Set $\calF$ be the collection of functions of 
        $f:\Omega\rightarrow\bbR$. For any $f\in\calF$, we define
        \begin{align*}
            q(f)=\bbE_{q}\big[f(X)\big],\quad \sigma^{2}(f)={\rm Var}_{q}\big(f(X)\big),
        \end{align*}
        where the expectation is taken with respect to the stationary distribution $q$. Let $Q$ be the distribution of the random function $f$. Assume that $|f(X))-q(f)|\leq c$ almost surely with respect to $Q$ for some constant $c>0$. Then we have that with probability at least $1-\delta$, the following inequality holds.
        \begin{align}
            \bigg|\bbE_{Q}\bigg[q(f)-\frac{1}{n}\sum_{i=1}^{n}f(X_{i})\bigg]\bigg|\leq \frac{C+(2-C)\lambda}{10c}\bbE_{Q}[\sigma^{2}(f)]+\frac{10c}{(1-\lambda)n}\bigg[\kl(Q\| P_{0})+\log \frac{2N(q_{0},q)}{\delta^{2}}\bigg],\label{ieq:137}
        \end{align}
        where $C$ is an absolute constant such that $0<C<e^{1/10}$.
    \end{proposition}
    \begin{proof}
        See Appendix~\ref{app:proofpacbayes2}.
    \end{proof}
    Our proof can be decomposed into two main parts.
    \begin{itemize}
        \item We verify that the Bellman error satisfies the conditions in Proposition~\ref{prop:pacbayes2} and apply it to the Bellman error.
        \item We adopt the similar procedure in the proof of Theorem~\ref{thm:main_concen} to control the fluctuation of both sides in inequality~\eqref{ieq:137} and calculate $\kl(Q\| P_{0})$.
    \end{itemize}
    \textbf{Step 1: Verify the conditions in Proposition~\ref{thm:main_concen2}}
    
    We consider the Markov chain formed by $\{(\bar{\bS}_{t},\bar{\bA}_{t},\bar{\bS}_{t+1},\bar{\bA}_{t+1})\}_{t=0}^{\infty}$. Note that this Markov chain shares the same absolute spectral gap with the Markov chain $\{(\bar{\bS}_{t},\bar{\bA}_{t})\}_{t=0}^{\infty}$ when $\calS$ and $\calA$ are finite.
    
    Let $\bX_{t}=(\bar{\bS}_{t},\bar{\bA}_{t},\bar{\bS}_{t+1},\bar{\bA}_{t+1})$ for all $f,\tilde{f}\in\calF_{\tf}(B_{a},B_{b},B_{QK},B_{V},B_{w})$. We define
    \begin{align}
        l^{\prime}(f,\tilde{f},\pi;\bX_{t}) &=\big(f(\bar{\bS}_{t},\bar{\bA}_{t})-\barr(\bar{\bS}_{t},\bar{\bA}_{t})-\gamma \tilde{f}(\bar{\bS}_{t+1},\pi)\big)^{2} \nonumber\\*
        &\qquad\qquad-\big(\calT^{\pi}\tilde{f}(\bar{\bS}_{t},\bar{\bA}_{t})-\barr(\bar{\bS}_{t},\bar{\bA}_{t})-\gamma \tilde{f}(\bar{\bS}_{t+1},\pi)\big)^{2}.\nonumber
    \end{align}
    Then the term we consider in Theorem~\ref{thm:main_concen} can be expressed as
    \begin{align}
        \calL^{\prime}(f,\tilde{f},\pi;\calD^{\prime})-\calL^{\prime}(\calT^{\pi}\tilde{f},\tilde{f},\pi;\calD^{\prime})=\frac{1}{n}\sum_{i=1}^{n}l^{\prime}(f,\tilde{f},\pi;\bX_{i})\text{ and }\big|l^{\prime}(f,\tilde{f},\pi;\bX)\big|\leq 4V_{\max}^{2}.\nonumber
    \end{align}
    Then the expectation of $l(f,\tilde{f},\pi;\bX)$ with respect to the stationary distribution $(\bar{\bS}_{t},\bar{\bA}_{t},\bar{\bS}_{t+1},\bar{\bA}_{t+1})\sim q_{P^{*}}^{\pi_{0}}\times P^{*}\times \pi_{0}$ is
    \begin{align}
        &\bbE_{q_{P^{*}}^{\pi_{0}}\times P^{*}\times \pi_{0}}\big[l^{\prime}(f,\tilde{f},\pi;\bX_{t})\big]\nonumber\\
        &\quad=\bbE_{q_{P^{*}}^{\pi_{0}}\times \bar{P^{*}}}\Big[\big(f(\bar{\bS}_{t},\bar{\bA}_{t})-\calT^{\pi}\tilde{f}(\bar{\bS}_{t},\bar{\bA}_{t})\big)\big(f(\bar{\bS}_{t},\bar{\bA}_{t})+\calT^{\pi}\tilde{f}(\bar{\bS}_{t},\bar{\bA}_{t})-2\barr-2\gamma \tilde{f}(\bar{\bS}_{t+1},\pi)\big)\Big]\nonumber\\
        &\quad=\bbE_{q_{P^{*}}^{\pi_{0}}}\bigg[\bbE_{P^{*}}\Big[\big(f(\bar{\bS}_{t},\bar{\bA}_{t})-\calT^{\pi}\tilde{f}(\bar{\bS}_{t},\bar{\bA}_{t})\big)\big(f(\bar{\bS}_{t},\bar{\bA}_{t})+\calT^{\pi}\tilde{f}(\bar{\bS}_{t},\bar{\bA}_{t}) \nonumber\\*
        &\qquad \qquad \qquad \qquad \qquad -2\barr-2\gamma \tilde{f}(\bar{\bS}_{t+1},\pi)\big)\,\Big|\,\bar{\bS}_{t},\bar{\bA}_{t}\Big]\bigg]\nonumber\\
        &\quad=\bbE_{q_{P^{*}}^{\pi_{0}}}\Big[\big(f(\bar{\bS}_{t},\bar{\bA}_{t})-\calT^{\pi}\tilde{f}(\bar{\bS}_{t},\bar{\bA}_{t})\big)^{2}\Big],\label{ieq:138}
    \end{align}
    where the last equality follows from the definition of the Bellman operator. As a consequence, the variance of $l^{\prime}(f,\tilde{f},\pi;\bX)$ can be bounded by its expectation as
    \begin{align}
        &{\rm Var}_{q_{P^{*}}^{\pi_{0}}\times P^{*}\times \pi_{0}}\big(l^{\prime}(f,\tilde{f},\pi;\bX_{t})\big)\nonumber\\
        &\quad\leq \bbE_{q_{P^{*}}^{\pi_{0}}\times P^{*}\times \pi_{0}}\Big[\big(l^{\prime}(f,\tilde{f},\pi;\bX_{t})\big)^{2}\Big]\nonumber\\
        &\quad=\bbE_{q_{P^{*}}^{\pi_{0}}}\bigg[\bbE_{P^{*}}\Big[\big(f(\bar{\bS}_{t},\bar{\bA}_{t})-\calT^{\pi}\tilde{f}(\bar{\bS}_{t},\bar{\bA}_{t})\big)^{2}\big(f(\bar{\bS}_{t},\bar{\bA}_{t})+\calT^{\pi}\tilde{f}(\bar{\bS}_{t},\bar{\bA}_{t}) \nonumber\\*
        &\qquad \qquad \qquad \qquad \qquad -2\barr-2\gamma \tilde{f}(\bar{\bS}_{t+1},\pi)\big)^{2}\,\Big|\,\bar{\bS}_{t},\bar{\bA}_{t}\Big]\bigg]\nonumber\\
        &\quad\leq 16V_{\max}^{2}\bbE_{q_{P^{*}}^{\pi_{0}}}\Big[\big(f(\bar{\bS}_{t},\bar{\bA}_{t})-\calT^{\pi}\tilde{f}(\bar{\bS}_{t},\bar{\bA}_{t})\big)^{2}\Big]\label{ieq:139}
    \end{align}
    where the last inequality follows from the fact that $f$ and $\tilde{f}$ is bounded by $V_{\max}$. Eq.~\eqref{ieq:138} shows that $l^{\prime}(f,\tilde{f},\pi;\bX_{t})$ satisfies the condition in Proposition~\ref{prop:pacbayes2} with $c=4V_{\max}^{2}$. Applying Proposition~\ref{prop:pacbayes} and inequality~\eqref{ieq:139} to $l^{\prime}(f,\tilde{f},\pi;\bX_{t})$, we have with probability at least $1-\delta$,
    \begin{align}
        &\bigg|\bbE_{Q}\bigg[\bbE_{q_{P^{*}}^{\pi_{0}}}\Big[\big(f(\bar{\bS}_{t},\bar{\bA}_{t})-\calT^{\pi}\tilde{f}(\bar{\bS}_{t},\bar{\bA}_{t})\big)^{2}\Big]-\frac{1}{n}\sum_{t=0}^{n-1}l^{\prime}(f,\tilde{f},\pi;\bX_{t})\bigg]\bigg|\nonumber\\
        &\quad\leq \frac{C+(2-C)\lambda}{2}\bbE_{Q,q_{P^{*}}^{\pi_{0}}}\Big[\big(f(\bar{\bS}_{t},\bar{\bA}_{t})-\calT^{\pi}\tilde{f}(\bar{\bS}_{t},\bar{\bA}_{t})\big)^{2}\Big]\nonumber\\*
        &\quad\qquad+\frac{40V_{\max}^{2}}{(1-\lambda)n}\bigg[\kl(Q\| P_{0})+\log \frac{2N(q_{0}\pi_{0},q_{P^{*}}^{\pi_{0}})}{\delta^{2}}\bigg],\label{ieq:140}
    \end{align}
    where $0<C<e^{1/10}$ is an absolute constant.
    
    \textbf{Step 2: Control the fluctuation of both sides in inequality~\eqref{ieq:140} and calculate $\kl(Q\| P_{0})$}
    
    To control the fluctuation of both sides in inequality~\eqref{ieq:140} and calculate $\kl(Q\| P_{0})$, we take the same procedure in the steps 2, 3 and 4 in the proof of Theorem~\ref{thm:main_concen}. We derive the uniform convergence result that for all $f,\tilde{f}\in\calF_{\tf}(\bB)$ and all policies $\pi\in\Pi$, with probability at least $1-\delta$, we have 
    \begin{align}
        &\Big|\bbE_{q_{P^{*}}^{\pi_{0}}}\Big[\big(f(\bar{\bS},\bar{\bA})-\calT^{\pi}\tilde{f}(\bar{\bS},\bar{\bA})\big)^{2}\Big]-\calL^{\prime}(f,\tilde{f},\pi;\calD^{\prime})+\calL^{\prime}(\calT^{\pi}\tilde{f},\tilde{f},\pi;\calD^{\prime})\Big|\nonumber\\
        &\quad\leq \!\frac{C+(2-C)\lambda}{2}\bbE_{q_{P^{*}}^{\pi_{0}}}\!\Big[\big(f(\bar{\bS},\bar{\bA})\!-\!\calT^{\pi}\tilde{f}(\bar{\bS},\bar{\bA})\big)^{2}\Big]\!\nonumber\\
        &\quad\qquad+\!O\bigg(\frac{V_{\max}^{2}}{(1-\lambda)n}\biggl[mL^{2}d^{2}\log\frac{mdL\bar{B} n}{V_{\max}} +\log\frac{N(q_{0}\pi_{0},q_{P^{*}}^{\pi_{0}})\calN(\Pi,1/n,d_{\infty})}{\delta}\biggr]\bigg).\nonumber
    \end{align}
    
    Therefore, we conclude the proof of Proposition~\ref{thm:main_concen2}
\end{proof}

\subsubsection{Proof of Proposition~\ref{prop:pacbayes2}}\label{app:proofpacbayes2}
\begin{proof}[Proof of Proposition~\ref{prop:pacbayes2}]
    The proof consists of two main steps. First, we assume that the initial state is distributed as the stationary distribution $q$ and derive the results under this stationary setting. Second, we extend the result to the non-stationary Markov chain, i.e., the initial 
    state is not distributed as $q$ but $q_{0}$.

    
    \textbf{Step 1: Derive a  concentration bound when the initial state's distribution is the stationary distribution}
    
    Under the stationary setting, we make use of the following concentration results in \cite{jiang2018bernstein}.
    \begin{proposition}[Theorem 1 in~\cite{jiang2018bernstein}]\label{prop:stat_concen}
        Suppose $\{X_{i}\}_{i\geq 1}$ is a stationary Markov chain with invariant distribution $q$ and non-zero absolute spectral gap $1-\lambda>0$, and $f_{i}:x\rightarrow[-c,+c]$ is a sequence of functions with $q(f_{i})=0$. Let 
        $\sigma^{2}=1/n\sum_{i=1}^{n}q(f_{i}^{2})$. Then for any $0\leq t <(1-\lambda)/5c$, we have
        \begin{align*}
            \bbE_{q}\bigg[\exp\bigg(t\sum_{i=1}^{n}f_{i}(X_{i})\bigg)\bigg]\leq \exp\bigg(\frac{n\sigma^{2}}{c^{2}}(e^{tc}-1-tc)+\frac{n\sigma^{2}\lambda t^{2}}{1-\lambda-5ct}\bigg).
        \end{align*} 
    \end{proposition}
    
    Set $f_{i}(X_{i})=f(X_{i})-q(f)=g(X_{i})$. Proposition~\ref{prop:stat_concen} shows that for $0\leq t <(1-\lambda)n/(5c)$,
    \begin{align}
        \bbE_{q}\bigg[\exp \bigg(\frac{t}{n}\sum_{i=1}^{n}g(X_{i})\bigg)\bigg]\leq \exp \bigg[\frac{n\sigma^{2}}{c^{2}}\Big(e^{ct/n}-1-\frac{ct}{n}\Big)+\frac{\lambda\sigma^{2}t^{2}}{n(1-\lambda-5ct/n)}\bigg],\label{ieq:132}
    \end{align}
    where $\sigma^{2}=\sigma^{2}(f)$. We define
    \begin{align*}
        \varepsilon_{n}(f,X_{1}^{n})=\frac{t}{n}\sum_{i=1}^{n}g(X_{i})-\bigg[\frac{n\sigma^{2}}{c^{2}}\Big(e^{ct/n}-1-\frac{ct}{n}\Big)+\frac{\lambda\sigma^{2}t^{2}}{n(1-\lambda-5ct/n)}\bigg],
    \end{align*}
    By inequality~\eqref{ieq:132} and Markov's inequality, we have that for any distribution $P_{0}$ on the function class $\calF$, the random variable $\varepsilon_{n}(f,X_{1}^{n})$ induced by the Markov chain $\{X_{i}\}_{i=1}^{n}$ satisfies
    \begin{align}
        P_{q}\bigg(\bbE_{f\sim P_{0}}\Big[\exp\big( \varepsilon_{n}(f,X_{1}^{n})\big)\geq \frac{2}{\delta}\Big]\bigg)\leq \frac{\delta}{2},\label{ieq:133}
    \end{align}
    where the probability is taken with respect to the Markov chain with initial distribution $q$.
    
    Setting $g(f)=\varepsilon_{n}(f,X_{1}^{n})$ in Theorem~\ref{thm:DV}, we have
    \begin{align}
        \bbE_{Q}[\varepsilon_{n}(f,X_{1}^{n})]\leq \kl(Q\| P_{0})+\log \bbE_{P_{0}}\Big[\exp \big(\varepsilon_{n}(f,X_{1}^{n})\big)\Big].\label{ieq:134}
    \end{align}
    Substituting inequality~\eqref{ieq:133} into inequality~\eqref{ieq:134}, we have that with probability at least $1-\delta/2$
    \begin{align}
        \bbE_{Q}\bigg[\frac{t}{n}\sum_{i=1}^{n}g(X_{i})-\bigg[\frac{n\sigma^{2}}{c^{2}}\Big(e^{ct/n}-1-\frac{ct}{n}\Big)+\frac{\lambda\sigma^{2}t^{2}}{n(1-\lambda-5ct/n)}\bigg]\bigg]\leq \kl(Q\| P_{0})+\log \frac{2}{\delta}.\label{ieq:135}
    \end{align}
    Set $t/n=(1-\lambda)/(10c)$. Since $e^{x}-1-x\leq ax^{2}$ for all $x\in [0,\log 2a]$, the left-hand side of inequality~\eqref{ieq:135} can be upper bounded as
    \begin{align}
        &\bbE_{Q}\bigg[\frac{1}{n}\sum_{i=1}^{n}g(X_{i})\bigg]\nonumber\\
        &\quad\leq  \bigg[\frac{n}{c^{2}t}\Big(e^{ct/n}-1-\frac{ct}{n}\Big)+\frac{\lambda t}{n(1-\lambda-5ct/n)}\bigg]\bbE_{Q}[\sigma^{2}(f)]+\frac{1}{t}\kl(Q\| P_{0})+\frac{1}{t}\log \frac{2}{\delta}\nonumber\\
        &\quad \leq  \bigg[\frac{n}{c^{2}t}\cdot C \frac{t^{2}}{n^{2}}+\frac{\lambda t^{2}}{n(1-\lambda-5ct/n)}\bigg]\bbE_{Q}[\sigma^{2}(f)]+\frac{1}{t}\kl(Q\| P_{0})+\frac{1}{t}\log \frac{2}{\delta}\nonumber\\
        &\quad =\frac{C+(2-C)\lambda}{10c}\bbE_{Q}[\sigma^{2}(f)]+\frac{10c}{(1-\lambda)n}\bigg[\kl(Q\| P_{0})+\log \frac{2}{\delta}\bigg],\nonumber
    \end{align}
    where the $C$ in the second inequality is a constant that $C\leq e^{(1-\lambda)/10}< e^{1/10}$, the equality follows from substituting the value of $t$ into the second inequality, and the expectation in $\bbE_{Q}[\sigma^{2}(f)]$ is taken with respect to the distribution $Q$ on the set of function class $\calF$. From symmetry, we can show that the  with probability (taken with respect to the  Markov  chain initialized with the stationary distribution) at least $1-\delta$
    \begin{align}
        \bigg|\bbE_{Q}\bigg[\frac{1}{n}\sum_{i=1}^{n}g(X_{i})\bigg]\bigg|\leq \frac{C+(2-C)\lambda}{10c}\bbE_{Q}[\sigma^{2}(f)]+\frac{10c}{(1-\lambda)n}\bigg[\kl(Q\| P_{0})+\log \frac{2}{\delta}\bigg],\label{ieq:136}
    \end{align}
    where $0<C<e^{1/10}$ is an absolute constant.
    
    
    \textbf{Step 2: Extend inequality~\eqref{ieq:136} to an arbitrarily initialized Markov chain.}
    
    To extend the results to an arbitrarily initialized Markov chain, we make use of the following result in~\cite{paulin2015concentration}.
    \begin{proposition}[Proposition 3.15 in~\cite{paulin2015concentration}]\label{prop:nonstat}
        Let $\{X_{i}\}_{i=1}^{\infty}$ be a time homogeneous Markov chain with state space $\Omega$, and stationary distribution $q$. Suppose that $g : \Omega^n \to\bbR$ is a real-valued measurable function. Then
        \begin{align*}
            P_{q_{0}}\big(g(X_{1},\cdots,X_{n})\geq t\big)\leq N(q_{0},q)^{1/2} \cdot \Big[P_{q}\big(g(X_{1},\cdots,X_{n})\geq t\big)\Big]^{1/2},
        \end{align*}
        where $q_{0}$ is any distribution on $\Omega$, and $P_{q_{0}}$ and $P_{q}$ are the probability measures with respect to the Markov chains with initial state $X_{1}\sim q_{0}$ and $X_{1}\sim q$ respectively.
    \end{proposition}
    Combining Proposition~\ref{prop:nonstat} and inequality~\eqref{ieq:136}, we have that with probability (taken with respect to the arbitrarily initialized Markov chain) at least $1-\delta$
    \begin{align}
        \bigg|\bbE_{Q}\bigg[\frac{1}{n}\sum_{i=1}^{n}g(X_{i})\bigg]\bigg|\leq \frac{C+(2-C)\lambda}{10c}\bbE_{Q}[\sigma^{2}(f)]+\frac{10c}{(1-\lambda)n}\bigg[\kl(Q\| P_{0})+\log \frac{2N(q_{0},q)}{\delta^{2}}\bigg].
    \end{align}
    This concludes the proof of Proposition~\ref{prop:pacbayes2}.
\end{proof}

\section{Experiments}\label{app:exp}
Although the main aim of this paper is primarily theoretical, we provide some experiments of the model-free algorithms to illustrate the superiority of the transformer in homogeneous \ac{marl}.

\subsection{Simulation Environment}
In the experiments, we evaluate the performance of the algorithms on the \ac{mpe}~\citep{lowe2017multi,mordatch2018emergence}. We focus on the \emph{cooperative navigation} task, where $N$ agents move cooperatively to cover $L$ landmarks in the environment. Given $N$ agent positions $\bx_{i}\in\bbR^{2}$ for $i\in[N]$ and $L$ landmark positions $\by_{j}\in\bbR^{2}$ for $j\in[L]$, the agents receive the reward $$r=-\sum_{j=1}^{L}\min_{i\in[N]}\|\by_{j}-\bx_{i}\|_{2}.$$ This reward encourages the agents to move closer to the landmarks. We set the number of agents as $N=3,6,15,30$ and the number of landmarks as $L=N$. To collect an offline dataset, we learn a policy in the online setting, and the dataset is collected from the induced stationary distribution of such policy. 

In the training process, we use the Titan RTX and Intel(R) Core(TM) i7-6900K CPU @ 3.20GHz to train the neural networks. The size of the offline dataset is $60000\times 25$, where we simulate $60000$ episodes and implement $25$ steps in each episode. The learning rate is set to $10^{-3}$. The batch size is $1024$. The discount factor is $\gamma=0.95$.

\begin{figure}[t]
    \centering
    \subfigure[$N=3$]{
    \begin{minipage}[t]{0.5\linewidth}
    \centering
    \includegraphics[width=1\columnwidth]{n3.eps}
    \label{fig:n3}
    \end{minipage}%
    }%
    \subfigure[$N=6$]{
    \begin{minipage}[t]{0.5\linewidth}
    \centering
    \includegraphics[width=1\columnwidth]{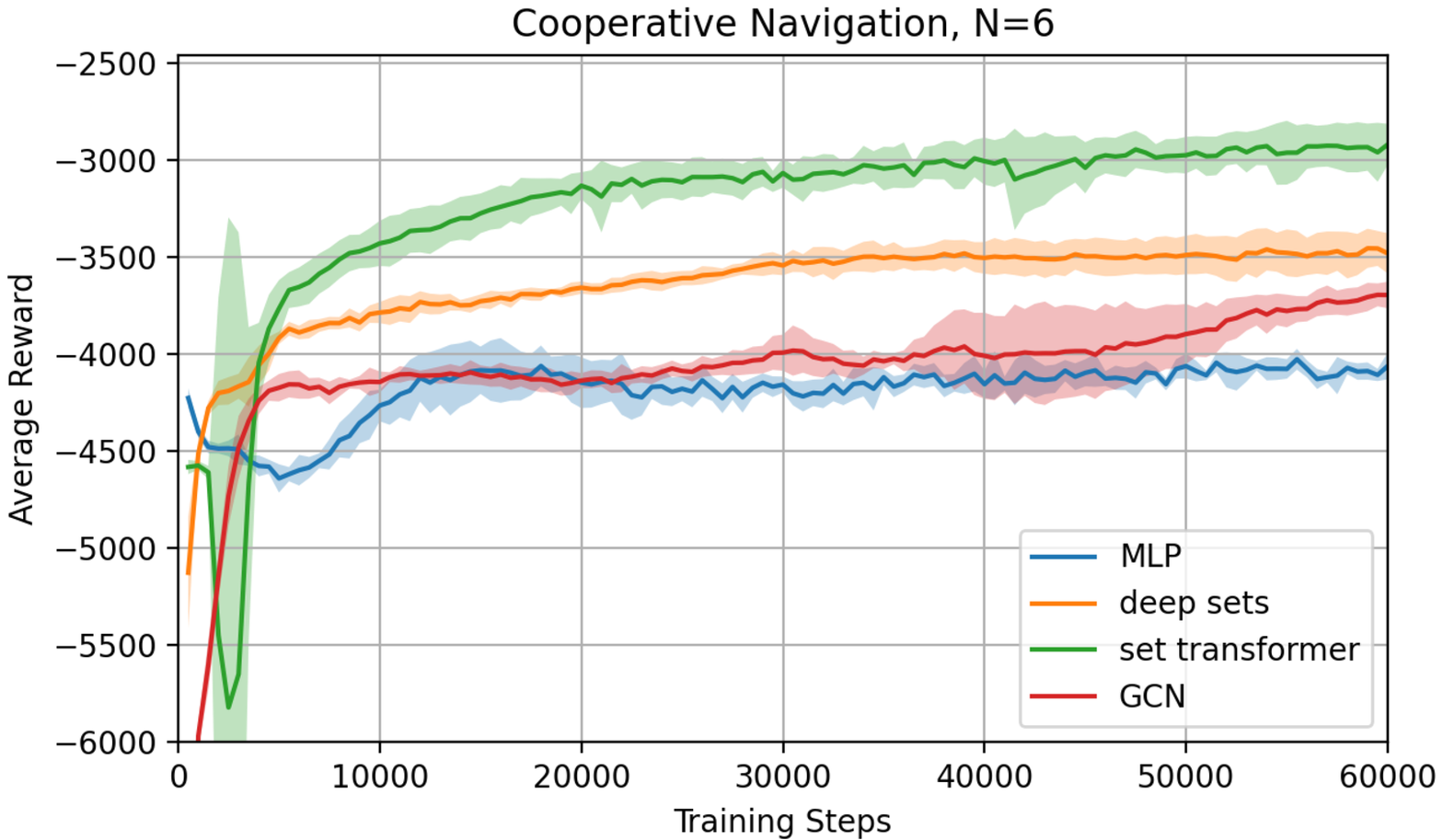}
    \label{fig:n6}
    \end{minipage}%
    }%
    
    \subfigure[$N=15$]{
    \begin{minipage}[t]{0.5\linewidth}
    \centering
    \includegraphics[width=1\columnwidth]{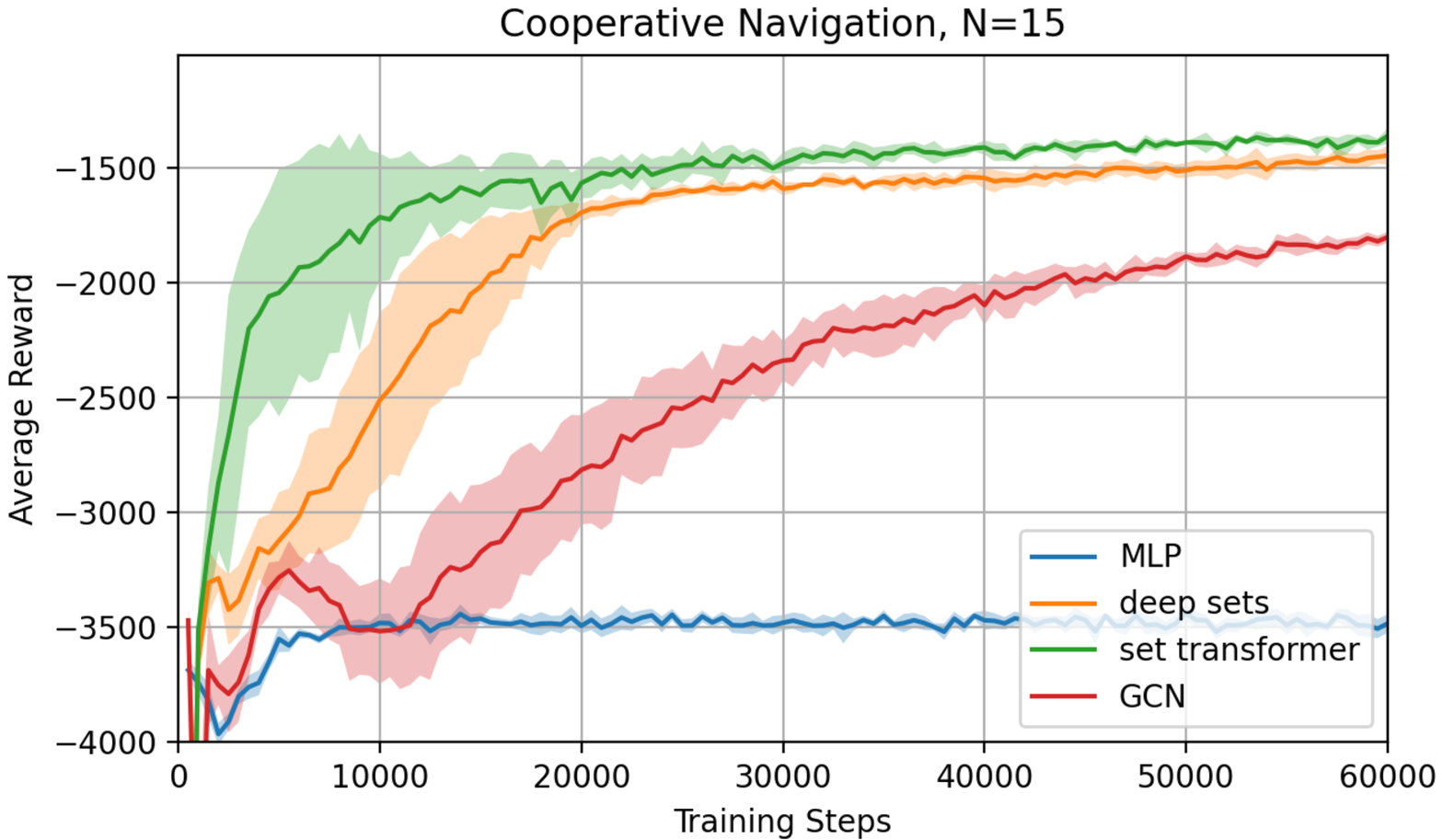}
    \label{fig:n15}
    \end{minipage}
    }%
    \subfigure[$N=30$]{
    \begin{minipage}[t]{0.5\linewidth}
    \centering
    \includegraphics[width=1\columnwidth]{n30.eps}
    \label{fig:n30}
    \end{minipage}
    }%
    \centering
    \caption{The average rewards of the model-free \ac{rl} algorithms with their standard deviations.}
    \label{fig:mfree}
\end{figure}
\subsection{Simulation Results}

We respectively adopt the \ac{mlp}, deep sets, \ac{gcn}~\citep{liu2020pic} and set transformer to estimate the value function. We note that the deep sets, \ac{gcn}, and set transformer are   permutation invariant functions. We use the code in~\cite{zaheer2017deep} for the  implementation of the deep sets and set transformer. To implement the model-free algorithm specified in Eqn.~\eqref{algo:mfree}, we optimize the policy and the action-value function in an alternating fashion. In addition, instead of imposing the hard constraint on the Bellman error $\calE(f,\pi;\calD)$, we added a Lagrangian multiplier to account for this inequality constraint.

In  Figure~\ref{fig:mfree}, we plot the performances of the model-free \ac{rl} algorithms that adopt different neural networks to estimate the action-value function. When the number of agents are small, as shown in Figure~\ref{fig:n3}, the performances of different neural networks are similar. As shown in Theorem~\ref{thm:approx}, relational reasoning abilities of the deep sets and the \ac{mlp} are worse than that of the set transformer. As a consequence, when the number of agents increases, as shown in Figures~\ref{fig:n6} to ~\ref{fig:n30}, the superiority of the algorithm that adopts the set transformer to estimate the action-value function becomes obvious. This strongly corroborates our theoretical results in Theorems~\ref{thm:approx} and \ref{thm:main_mfree}.

\end{document}